\documentclass{article}
\usepackage{mathrsfs}
\usepackage{amsthm}
\usepackage{amsmath}
\usepackage{amsfonts}
\usepackage{amssymb}
\usepackage{commath}
\usepackage[ansinew]{inputenc}
\usepackage{graphicx,color}
\usepackage{url}
\usepackage{latexsym,enumerate}
\usepackage{booktabs}
\usepackage{framed}
\usepackage{subfigure}
\usepackage{multirow} 
\usepackage{color}
\usepackage{colortbl}

\usepackage{algorithm}
\usepackage{algorithmic}

\definecolor{gray1}{rgb}{0.8,0.8,0.8}
\definecolor{gray2}{rgb}{0.95,0.95,0.95}

\newcommand{\RE}{\,{\rm Re}}

\newcommand{\argmin}{\mathop{\rm argmin}}

\newcommand{\Shrink}{\mathop{\rm Shrink}}

\newcommand{\CST}{\mathop{\rm CST}}

\newcommand{\ST}{\mathop{\rm ST}}

\newcommand{\cY}{{\mathcal Y}}
\newcommand{\cX}{{\mathcal X}}
\newcommand{\cP}{{\mathcal P}}
\newcommand{\cE}{{\mathcal E}}
\newcommand{\cA}{{\mathcal A}}
\newcommand{\cC}{{\mathcal C}}
\newcommand{\cD}{{\mathcal D}}

\newcommand{\cF}{{\mathcal F}}
\newcommand{\cL}{{\mathcal L}}
\newcommand{\cB}{{\mathcal B}}

\newcommand{\cI}{{\mathcal I}}
\newcommand{\cM}{{\mathcal M}}
\newcommand{\cN}{{\mathcal N}}

\newcommand{\BA}{{\mathbf A}}

\newcommand{\Bd}{{\mathbf d}}
\newcommand{\Bh}{{\mathbf h}}
\newcommand{\By}{{\mathbf y}}

\newcommand{\Bf}{{\mathbf f}}
\newcommand{\Bu}{{\mathbf u}}

\newcommand{\Br}{{\mathbf r}}
\newcommand{\Bv}{{\mathbf v}}
\newcommand{\Be}{{\mathbf e}}

\newcommand{\Bw}{{\mathbf w}}

\newcommand{\Bg}{{\mathbf g}}

\newcommand{\Bt}{{\mathbf t}}

\newcommand{\Brho}{{\boldsymbol \rho}}

\newcommand{\Beps}{{\boldsymbol \epsilon}}

\newcommand{\Bome}{{\boldsymbol \omega}}

\newcommand{\BDo}{{\mathbf{D_1}}}
\newcommand{\BDoT}{{\mathbf{D}_{\mathbf 1}^{\text T}}}
\newcommand{\BDt}{{\mathbf{D_2}}}
\newcommand{\BDtT}{{\mathbf{D}_{\mathbf 2}^{\text T}}}

\newcommand{\Bz}{{\boldsymbol z}}

\newcommand{\Bk}{{\boldsymbol k}}

\newcommand{\Bm}{{\boldsymbol m}}
\newcommand{\Bn}{{\boldsymbol n}}
\newcommand{\Bx}{{\boldsymbol x}}

\newtheorem{thm}{Theorem}[section]
\newtheorem{prop}[thm]{Proposition}
\newtheorem{rem}[thm]{Remark}
\newtheorem{lem}[thm]{Lemma}

\begin{document}

\title{Directional Mean Curvature for Textured Image Demixing}

%

\author{Duy Hoang Thai \thanks{Department of Statistical Science, Duke University, Durham, USA.} \thanks{Corresponding Author:  Email: dht11@stat.duke.edu}
\ and 
David Banks
\thanks{Email: banks@stat.duke.edu}
}

\maketitle

\begin{abstract}

Approximation theory plays an important role in image processing,
especially image deconvolution and decomposition.
For piecewise smooth images, there are many methods that have been
developed over the past thirty years.
The goal of this study is to devise similar and practical methodology 
for handling textured images.
This problem is motivated by forensic imaging, since fingerprints,
shoeprints and bullet ballistic evidence are textured images.
In particular, it is known that texture information is almost 
destroyed by a blur operator, such as a blurred ballistic image 
captured from a low-cost microscope.
The contribution of this work is twofold: first, we propose a 
mathematical model for textured image deconvolution and decomposition 
into four meaningful components, using a high-order partial
differential equation approach 
based on the directional mean curvature.
Second, we uncover a link between functional analysis and multiscale 
sampling theory, e.g., harmonic analysis and filter banks. 
Both theoretical results and examples with natural images are
provided to illustrate the performance of the proposed model. 

\end{abstract}

\section*{Keywords}
Image Deblurring, Demixing Problem, Variational Calculus, Feature
Extraction, Harmonic Analsysis, Sampling Theory, Multiscale Analysis,
Mean Curvature, High-order PDE

\section{Introduction} \label{sec:Introduction}

Processing textured images is difficult, but such images are common in
many applications. 
For example, as part of the 2015-2016 forensic statistics program at the 
Statistical and Applied Mathematical Sciences Institute, much of the
research dealt with images of fingerprints, balistic striations on bullets,
and shoeprints, and all of these are typically textured.

There is an extensive literature on methods for processing piecewise
smooth images; e.g., image restoration by functional analysis
\cite{MumfordShah1989, Potts1952, RudinOsherFatemi1992,
  AubertVese1997, ChanMarquinaMullet2000, ZhuChan2012,
  PapafitsorosSchonlieb2014, GoldsteinOsher2009}, 
image representation by harmonic analysis \cite{CandesDonoho2004,
  CandesDemanetDonohoYing2006, StarckDonohoCandes2003, MaPlonka2010,
  KutyniokLabate2012, DoVetterli2005, UnserVandeville2010,
  GillesTranOsher2014} 
and image segmentation \cite{ChanVese2001,
  BressonEsedogluVandergheynstThiranOsher2007,
  ChanEsedogluNikolova2012, LieLysakerTai2006,
  KassWitkinTerzopoulos1988, ZhuTaiChan2013, BrownChanBresson2010,
  BrownChanBresson2012, GuWangTai2012, BrownChanBresson2009,
  ChanSandbergVese2000}.
Much less has been done with textured images.

Instead of being piecewise smooth, textured images have fine scale
discontinuities, so that neighboring pixels in the image may take
very different values in gray scale or color space.
Examples are a shoeprint in sand, where the grainy structure
of the sand provides the texture, or a fingerprint on leather, where
the leather's texture interacts with the signal, see
\cite{MaltoniMaioJainPrabhakar2009}. 

To demix (i.e., to simultaneously deblur and decompose) textured
images, we substantially 
extend the work in Thai and Gottschlich \cite{ThaiGottschlich2016DG3PD}.
That paper provided a directional decomposition of gray-scale
images $\Bf$ into three parts, consisting of piecewise
smooth structure $\Bu$, texture $\Bv$, and
fine scale residual structure $\Brho$. 
The solution required minimization of the directional total variation norm,
which entails the solution to a second-order partial differential
equation.

The new approach is based upon \cite{LingStrohmer2015},
and also focuses upon gray-scale images. 
But now it decomposes the image $\Bf$ into four parts, adding 
fine-scale noise structure $\Beps$ to the 
previous three distinctions.
It also models a blur operator $\Bh$ (for simplicity, $\Bh$
is assumed to be known, which is reasonable since in principle
it can be measured for any camera in a specific application).
Thus the demixing problem is formulated as
\begin{align*}
\Bf = \Bh \ast \left( \Bu + \Bv + \Brho \right) + \Beps \,,
\end{align*}
where $\ast$ is a convolution operator.

The solution minimizes a function of several well-chosen norms
and uses directional mean curvature rather than a directional decomposition.
This entails solution of a fourth-order partial differential equation
(cf.\ Zhu and Chan \cite{ZhuChan2012, ZhuTaiChan2013}) 
which addresses the ``staircase effect'' in the TV-L2 model 
\cite{RudinOsherFatemi1992}. 
Other high-order partial differential equation (PDE) approaches 
for image reconstruction are
given in \cite{LysakerLundervoldTai2003, RahmanTaiOsher2007,
  HahnWuTai2012, TaiHahnChung2011, CalatroniDuringSchonlieb2014}.
Importantly, this work finds a new connection between the minimization
problem in demixing and multiscale harmonic analysis.
This connection enables a more general theory and a new solution strategy.

Our model realistically reflects the mechanism of image capture. 
The true image is a combination of both smooth regions, texture, and
fine scale structure.
Blurring is inevitable, and the additive noise term accounts for 
other distortions (e.g., threshold variation in the CMOS or CCD light
sensors, edge effects between pixels). 
The reconstruction error (the pixel-wise difference between the
true image and the estimated image) is especially useful since it permits
a quantitative basis for comparing the quality of demixing 
algorithms.
When two algorithms make comparable assumptions about smoothness
and the blur operator,
then any structure that persists after demixing appears in the
reconstruction error.  
Its mean squared error, or the eigenvalues of the estimated covariance
matrix, enable one to determine which algorithm has successfully 
extracted more signal.

Despite the ability of the Euler-Lagrange equations
associated with the variational model to achieve advanced
performance in image analysis, numerical solution of this high order 
PDE is difficult.
Following \cite{ZhuTaiChan2013, TaiHahnChung2011}, a numerical
augmented Lagrangian method (ALM) is applied to split the
directional mean curvature norm into several $\ell_2$- and
$\ell_1$-norms, which are solvable by iterative
shrinkage/thresholding (IST) algorithms
\cite{DaubechiesDefriseMol2004, BeckTeboulle2009,
  DiasFigueiredo2007}. 
Wu et al. \cite{WuTai2010} proved the equivalence between
ALM, dual methods (e.g., Chambolle's
projection \cite{Chambolle2004}), and the splitting Bregman method
\cite{GoldsteinOsher2009} for solving the convex optimization. 
(Note that ALM or the splitting Bregman method can be employed in the
Douglas-Rachford splitting scheme \cite{Setzer2011}.)  

We focus on understanding the advantages of the fourth-order PDE
approach to reconstruct a smooth image with sharp edges, and
make two contributions:
\begin{enumerate}
\item we provide a general solution to a mathematical model for 
textured image deconvolution and decomposition with
four meaningful components, and
\item we find a link between functional analysis and multiscale
  sampling theory in harmonic analysis and filter banks.  
\end{enumerate} 
This employs a novel Directional Mean Curvature Demixing (DMCD) method 
for textured images corrupted by i.i.d.\ noise with a given blur
kernel. 

To develop these ideas, section 2 introduces the mathematical 
framework, and section 3 describes the use of directional mean
curvature for demixing and
the algorithm needed to fit the model.
Section 4 makes a quantitative comparison between the demixing 
results from the proposed algorithm
and results from a standard alternative algorithm.
Section 5 describes the new model's correspondence with 
multiscale harmonic analysis.  
Section 6 summarizes our conclusions.
For readability, mathematical proofs are provided in the Appendix and we also refer the reader to \cite{ThaiMentch2016, ThaiGottschlich2016DG3PD, ThaiGottschlich2016inpainting, ThaiGottschlich2016G3PD} for more detailed notation and literature review.

\section{Notation and Mathematical Preliminaries}

In this section we define the image and present the directional
forward/backward difference operators.  
We also specify the directional gradient and divergence operators,
the directional Laplacian operator, and the discrete directional $\text{G}_S$-norm. 
Using the curvelet transform and pointwise shrinkage operators,
these enable the mathematical derivation
of the new demixing algorithm.

Given a discrete grayscale image $f[\Bk] :
\Omega \rightarrow \mathbb R_+$ of size $d_1 \times d_2$, with
the lattice
$$
\Omega = \Big\{ \Bk = [k_1, k_2] \in [0 \,, d_1-1] \times [0 \,, d_2-1] \subset \mathbb Z^2 \Big\} \,,
$$
let $X$ be the Euclidean space whose dimension is given by the size of
the lattice $\Omega$; i.e., $X = \mathbb R^{\abs{\Omega}}$. 
The 2D discrete Fourier transform $\mathcal F$ acting on $f[\Bk]$ is
\begin{equation*}
 f[\Bk] ~\stackrel{\mathcal F}{\longleftrightarrow}~ F(e^{j \Bome}) 
 = \sum_{\Bk \in \Omega} f[\Bk] \cdot e^{-j \langle \Bk \,, \Bome \rangle_{\ell_2}},
\end{equation*}
where the Fourier coordinates $\Bome \in [-\pi \,, \pi]^2$ are defined 
on the lattice as
\begin{equation*}
 \cI = \Big\{ \Bome = [\omega_1 \,, \omega_2] = \left( \frac{2 \pi d_2'}{d_2} \,, \frac{2 \pi d_1'}{d_1} \right) 
 \mid (d_2' \,, d_1') \in \left[ - \frac{d_2}{2} \,, \frac{d_2}{2} \right) \times \left[ - \frac{d_1}{2} \,, \frac{d_1}{2} \right) \subset \mathbb Z^2
 \Big\} \,.
\end{equation*}
We let $\Bz = [z_1, z_2] = \big[ e^{j\omega_1} \,, e^{j\omega_2}
\big]$ denote the discrete coordinates of the Fourier transform. 

Given the matrix
\begin{equation*}
 \BDo = \begin{pmatrix}
         -1            & 1             &0            &\hdots           &0              \\
          0            &-1             &1            &\hdots           &0              \\
          \vdots       &\vdots         &\vdots       &\ddots           &\vdots         \\
          0            &0              &0            &\hdots           &1              \\
          1            &0              &0            &\hdots            &-1            \\             
        \end{pmatrix}
  \in \mathbb R^{d_1 \times d_1},
\end{equation*}
the directional forward and backward difference operators (with periodic boundary condition
and direction $l = 0, \ldots, L-1$)
in a matrix form, i.e. $\partial_l^\pm \Bf = \Big[ \partial_l^\pm f[\Bk] \Big]_{\Bk \in \Omega} \in X$, 
are
\begin{align*}
 &\partial_l^+ \Bf = \cos(\frac{\pi l}{L}) \Bf \BDtT + \sin(\frac{\pi l}{L}) \BDo \Bf
 ~~~~~~~\stackrel{\cF}{\longleftrightarrow}~ 
 \left[ \cos\left(\frac{\pi l}{L}\right) (z_2 - 1) + \sin\left(\frac{\pi l}{L}\right) (z_1 - 1) \right] F(\Bz) \,,
 \\
 &\partial_l^- \Bf = - \left[ \cos\left(\frac{\pi l}{L}\right) \Bf \BDt + \sin\left(\frac{\pi l}{L}\right) \BDoT \Bf \right]
 ~~\stackrel{\cF}{\longleftrightarrow}~   
 -\left[ \cos\left(\frac{\pi l}{L}\right) (z_2^{-1} - 1) + \sin\left(\frac{\pi l}{L}\right) (z_1^{-1}-1) \right] F(\Bz) \,.
\end{align*}
The transposed matrices of $\BDt$ and $\BDo$ are $\BDtT$ and $\BDoT$, respectively.
Their adjoint operators are $\big( \partial_l^\pm \big)^* = -\partial_l^\mp$.

The directional gradient and divergence operators are, respectively, 
\begin{align*}
 \nabla_L^\pm \Bf &= \Big[ \partial_l^\pm \Bf \Big]_{l=0}^{L-1} \in X^L 
 ~~\text{and}~~
 \text{div}_L^\pm \vec{\Bg} = \sum_{l=0}^{L-1} \partial_l^\pm \Bg_l \in X.
\end{align*}

Note that the adjoint operator of $\nabla_L^\pm$ is $\big( \nabla_L^\pm \big)^* = - \text{div}_L^\mp$, i.e., 
\begin{equation*}
 \Big\langle \nabla_L^\pm \Bf \,, \vec{\Bg} \Big\rangle_{\ell_2} = 
 \Big\langle \Bf \,, \underbrace{ - \text{div}_L^\mp }_{=
   (\nabla_L^\pm)^* } \vec{\Bg} \Big\rangle_{\ell_2} \,.  
\end{equation*}

The directional Laplacian operator is 
\begin{align*}
 \Delta_{\text{d}L} \Bf = \text{div}^-_L \nabla^+_L \Bf =
 \sum_{l=0}^{L-1} \partial_l^- \partial_l^+ \Bf 
 ~\stackrel{\cF}{\longleftrightarrow}~
 - \sum_{l=0}^{L-1} \abs{ \cos\left(\frac{\pi l}{L}\right)(z_2 - 1) +
   \sin\left(\frac{\pi l}{L}\right)(z_1 - 1) }^2 F(\Bz). 
\end{align*}

Remind that given $\Bx \in \mathbb R^2$,
the impulse responses of directional derivative and directional Laplacian operators in a continuous setting are
\begin{align*}
 \partial_l \delta(\Bx) &= \left[ \cos\left( \frac{\pi l}{L} \right) \partial_x + \sin\left( \frac{\pi l}{L} \right) \partial_y \right] \delta(\Bx)
 ~\stackrel{\cF}{\longleftrightarrow}~
 \cos\left( \frac{\pi l}{L} \right) j \omega_x + \sin\left( \frac{\pi l}{L} \right) j \omega_y \,,
 \\
 \Delta_L \delta(\Bx) &= \sum_{l=0}^{L-1} \partial_l^2 \delta(\Bx) 
 ~\stackrel{\cF}{\longleftrightarrow}~
 -\sum_{l=0}^{L-1} \left[ \cos \left( \frac{\pi l}{L} \right) \omega_x + \sin \left( \frac{\pi l}{L} \right) \omega_y \right]^2 \,,
\end{align*}
with $(\omega_x \,, \omega_y)$ are continuous version of $(\omega_2 \,, \omega_1)$ which is numerically used instead.

Extending \cite{Meyer2001, AujolChambolle2005, VeseOsher2003}, and
due to the discrete nature of images,  
Thai and Gottschlich \cite{ThaiGottschlich2016DG3PD} defines 
the discrete directional $\text{G}_S$-norm with $S \in \mathbb N_+$ in the
anisotropic version as 
\begin{align} \label{eq:directionalGnorm}
 \norm{\Bv}_{\text{G}_S} = \inf \Big\{ \norm{\vec{\Bg}}_{\ell_1} = \sum_{s=0}^{S-1} \norm{ \Bg_s }_{\ell_1} \,,~ \Bv = \text{div}^-_S \vec{\Bg} \,,~ \vec{\Bg} = \big[ \Bg_s \big]_{s=0}^{S-1} \in X^S \Big\} \,.
\end{align}
As stated in \cite{Meyer2001} and \cite[p. 87]{AujolChambolle2005}, 
the space of bounded variation $BV$ is suitable for the piecewise smooth
image component $\Bu$, the  
$\text{G}$-space is suitable for the texture component $\Bv$, and the dual 
Besov space $\dot{\text{B}}^{-1}_{\infty, \infty}$ is suitable for the noise 
component $\Beps$, where 
\begin{align*}
 \dot{\text{B}}^{1}_{1,1} \subset \dot{\text{BV}} \subset \text{L}_2 \subset \text{G} \subset
 \dot{\text{B}}^{-1}_{\infty, \infty} \,.  
\end{align*}

Since natural images are better described in terms of 
multi-scale and multi-direction,
we apply the curvelet transform \cite{CandesDemanetDonohoYing2006,
  CandesDonoho2004, MaPlonka2010} instead of the 
wavelet transform in the dual Besov space, see
\cite{ThaiGottschlich2016G3PD, ThaiGottschlich2016DG3PD}. 
Motivated by the Dantzig selector \cite{CandesTao2007}, Aujol and Chambolle
\cite{AujolChambolle2005} 
and Thai and Gottschlich \cite{ThaiGottschlich2016DG3PD}, find that
the residual $\Beps$ is better captured by $\norm{\cC \{ \cdot
  \}}_{\ell_\infty}$ (bounded by a constant $\nu$) in the curvelet domain,
which can be represented by an indicator function on a feasible convex set 
$\mathscr G(\nu) = \big\{ \Beps \in X ~:~ \norm{ \cC\{ \Beps \}
}_{\ell_\infty} \leq \nu \big\}$ as 
\begin{align}  \label{eq:residual:curvelet}
 \mathscr G^* \left(\frac{\Beps}{\nu}\right) = 
 \begin{cases} 
  0 \,,~ & \Beps \in \mathscr G(\nu) \\ +\infty \,,~ & \text{else.}
 \end{cases} \,
\end{align}
This measure controls the maximum curvelet coefficient of the residual $\Beps$. 
Due to the curvlet transform, no assumption on the kind of noise is needed
(e.g. Gaussian, Laplacian or weakly correlated noise), see 
\cite{ThaiGottschlich2016inpainting, ThaiMentch2016}. 
Following \cite{HaltmeierMunk2014}, Thai and Gottschlich 
\cite{ThaiGottschlich2016G3PD} proposes a threshold based on extreme 
value theory for (\ref{eq:residual:curvelet}). 
Note that if $\Beps$ is normally distributed, the random variable 
$\norm{\cC \{ \Beps \}}_{\ell_\infty}$ has the Gumbel distribution 
(since the curvelet of a Gaussian process is weakly correlated), 
see \cite{HaltmeierMunk2014}. 
In general, for correlated noise, its distribution is a max-stable process \cite{Smith1990, Schlather2002}.

It is known that the solution of $\ell_1$ minimization is a shrinkage operator, 
see \cite{DaubechiesDefriseMol2004, BeckTeboulle2009, Thai2015PhD}. 
It can be defined in a matrix form as
\begin{align*}
 \Shrink(\Bf \,, \alpha) := \frac{\Bf}{\abs{\Bf}} \cdot^\times \max \big( \abs{\Bf} - \alpha \,, 0 \big) \,,
\end{align*}
with the point-wise operators. 
For example, a multiply pointwise operator of functions $\Bf \,, \Bd \in X$ is
$\Bf \cdot^\times \Bd = \big[ f[\Bk] d[\Bk] \big]_{\Bk \in \Omega}$.
The discrete convolution of two functions $\Bf \,, \Bg \in X$ is defined in a matrix form as
\begin{align*}
 \Bf \ast \Bd = \left[ (f \ast d)[\Bk] \right]_{\Bk \in \Omega} 
 ~~\text{and}~~
 (f \ast d)[\Bk] = \sum_{\Bn \in \Omega} f[\Bn] d[\Bk - \Bn] \,.
\end{align*}
For simplicity, given $\vec{\Bt} \,, \vec{\Br} \in X^{L+1}$, we denote 
\begin{align*}
 \left\langle \vec{\Bt} \,, \vec{\Br} \right\rangle_X = \sum_{l=0}^{L} \Bt_l \cdot^\times \Br_l \in X
 \quad \text{and} \quad 
 \Bf \cdot^\times \vec{\Br} = \left[ \Bf \cdot^\times \Br_l \right]_{l=0}^L \in X^{L+1} \,.
\end{align*}

Given a curvelet transform $\cC$ and its inverse version $\cC^*$ \cite{CandesDemanetDonohoYing2006} and a function $\Bf$,
the curvelet shrinkage thresholding operator $\text{CST}(\cdot \,, \cdot)$ is defined as 
\begin{align*}
 \text{CST}(\Bf \,, \nu) = \cC^* \Big\{ \Shrink \big( \cC \{ \Bf \} \,, \alpha \big) \Big\} \,.
\end{align*}
Note that one can easily apply other kinds of harmonic analysis than
the curvelet, e.g., the  shearlet \cite{YiLabateEasleyKrim2009},
steerable wavelet \cite{UnserVandeville2010, UnserChenouardVandeville2011, UnserSageVandeville2009},
contourlet \cite{DoVetterli2005, CunhaZhouDo2006} and dual-tree complex wavelet \cite{SelesnickBaraniukKingsbury2005}.
Given a function $\Bf \in X$, its time reversed function $\check{\Bf}
= \big[ \check{f}[\Bk] \big]_{\Bk \in \Omega} \in X$ is defined as  
\begin{align*}
 \check{f}[\Bk] = f[-\Bk] ~\stackrel{\cF}{\longleftrightarrow}~ F(\Bz^{-1}) \,.
\end{align*} 
For more mathematical background and notation, we refer the reader to
\cite{Thai2015PhD, ThaiGottschlich2016G3PD,
  ThaiGottschlich2016DG3PD}. 


\section{Directional mean curvature for image demixing (DMCD)}

We assume that the original image $\Bf$, which consists of piecewise smooth regions $\Bu$, texture $\Bv$ 
and fine scale structure $\Brho$,
is blurred by the operator $h$ and corrupted by noise $\Beps$ as
\begin{align*}
 f[\Bk] = \big[ h \ast (u + v + \rho) \big] [\Bk] + \epsilon[\Bk] \,,~ \Bk \in \Omega \,.
\end{align*}
Instead of the total variation norm \cite{RudinOsherFatemi1992}, and
following the higher order approaches \cite{ChanMarquinaMulet2000, LysakerLundervoldTai2003, ZhuChan2012, ZhuTaiChan2013, ZhuTaiChan2013ALM},
we propose a discrete directional mean curvature (DMC) norm to reconstruct the piecewise-smooth image $\Bu$ as
\begin{align*}
 \norm{\kappa_L^\text{d} \big\{ \Bu \big\}}_{\ell_1} = \sum_{\Bk \in \Omega} \abs{ \kappa_L^\text{d} \big\{ \Bu \big\} [\Bk] } 
 ~~\text{and}~~
 \kappa^\text{d}_L \{\Bu\} = \text{div}^-_L \Bigg\{ \frac{\nabla^+_L \Bu}{ \sqrt{\boldsymbol 1 + \abs{\nabla^+_L \Bu}^{\cdot 2}} } \Bigg\}
\end{align*}
with the matrix of ones $\boldsymbol 1$ (size $d_1 \times d_2$).
According to \cite{ZhuTaiChan2013ALM}, a discrete DMC is rewritten as
\begin{align*}
 \kappa^\text{d}_L(\Bu) = \text{div}^-_L \Bigg\{ \frac{\big[ \nabla^+_L \Bu \,, \boldsymbol 1 \big]}{ \abs{ \big[ \nabla^+_L \Bu \,, \boldsymbol 1 \big] } } \Bigg\} \,.
\end{align*}

As in the DG3PD model \cite{ThaiGottschlich2016DG3PD}, and following the image generation 
mechanism shown in Fig.\ 1, we define 
the discrete DMCD model for image deconvolution and decomposition problem as
\begin{align} \label{eq:minimization:SDMCDD:1}
 \min_{(\Bu, \Bv, \Brho, \Beps) \in X^4} \bigg\{
 &\norm{\kappa_L^\text{d} \big\{ \Bu \big\}}_{\ell_1} + \mu_1 \norm{\Bv}_{\text{G}_S} + \mu_2 \norm{\Bv}_{\ell_1}   \notag
 \\
 &\text{s.t.}~~ \Bf = \Bh \ast \left( \Bu + \Bv + \Brho \right) + \Beps
 \,,~ \norm{\cC\{ \Brho \}}_{\ell_\infty} \leq \nu_\rho
 \,,~ \norm{\cC\{ \Beps \}}_{\ell_\infty} \leq \nu_\epsilon
 \bigg\} \,.
\end{align}
Following \cite{ThaiGottschlich2016G3PD, VeseOsher2003,
  AujolChambolle2005, ThaiGottschlich2016DG3PD, Gilles2012, Meyer2001,
  CandesDonoho2004},  
the discrete directional $\text{G}_S$-norm measures texture in
several directions,  
and the dual of a generalized Besov space in the curvelet domain $\cC$
captures the residual structure $\Brho$ and noise $\Beps$. 
Note that the $\ell_\infty$-norm of the curvelet transform is a good
measure for fine scale oscillating patterns (i.e., residual structure and noise),
which can be either independent or ``weakly'' correlated and need not
follow a Gaussian distribution. The $\text{G}_S$-norm for texture is
handled by the approach of Vese and Osher \cite{VeseOsher2003}. 
According to Meyer \cite{Meyer2001} (or \cite{GarnettJonesLeVese2011}), 
the oscillating components do not have small $L_2$ or $L_1$-norm.

From the definition of the directional $\text{G}_S$-norm (\ref{eq:directionalGnorm}) 
and the curvelet space for noise measurement (\ref{eq:residual:curvelet}), 
the constrained minimization (\ref{eq:minimization:SDMCDD:1}) is rewritten as
\begin{align} \label{eq:minimization:SDMCDD:2}
 \min_{(\Bu, \Bv, \Brho, \Beps, \vec{\Bg}) \in X^{4+S}} \bigg\{
 &\norm{\kappa_L^\text{d} \big\{ \Bu \big\}}_{\ell_1} + \mu_1 \sum_{s=0}^{S-1} \norm{ \Bg_s }_{\ell_1}
 + \mu_2 \norm{\Bv}_{\ell_1} + \mathscr G^* \left(\frac{\Brho}{\nu_\rho}\right) + \mathscr G^*\left(\frac{\Beps}{\nu_\epsilon}\right)      \notag
 \\
 &\text{s.t.}~~ \Bf = \Bh \ast \Bu + \Bh \ast \Bv + \Bh \ast \Brho + \Beps \,,~ \Bv = \text{div}^-_S \vec{\Bg}
 \bigg\} \,.
\end{align}
\begin{figure}
\begin{center} 
 \includegraphics[width=0.75\textwidth]{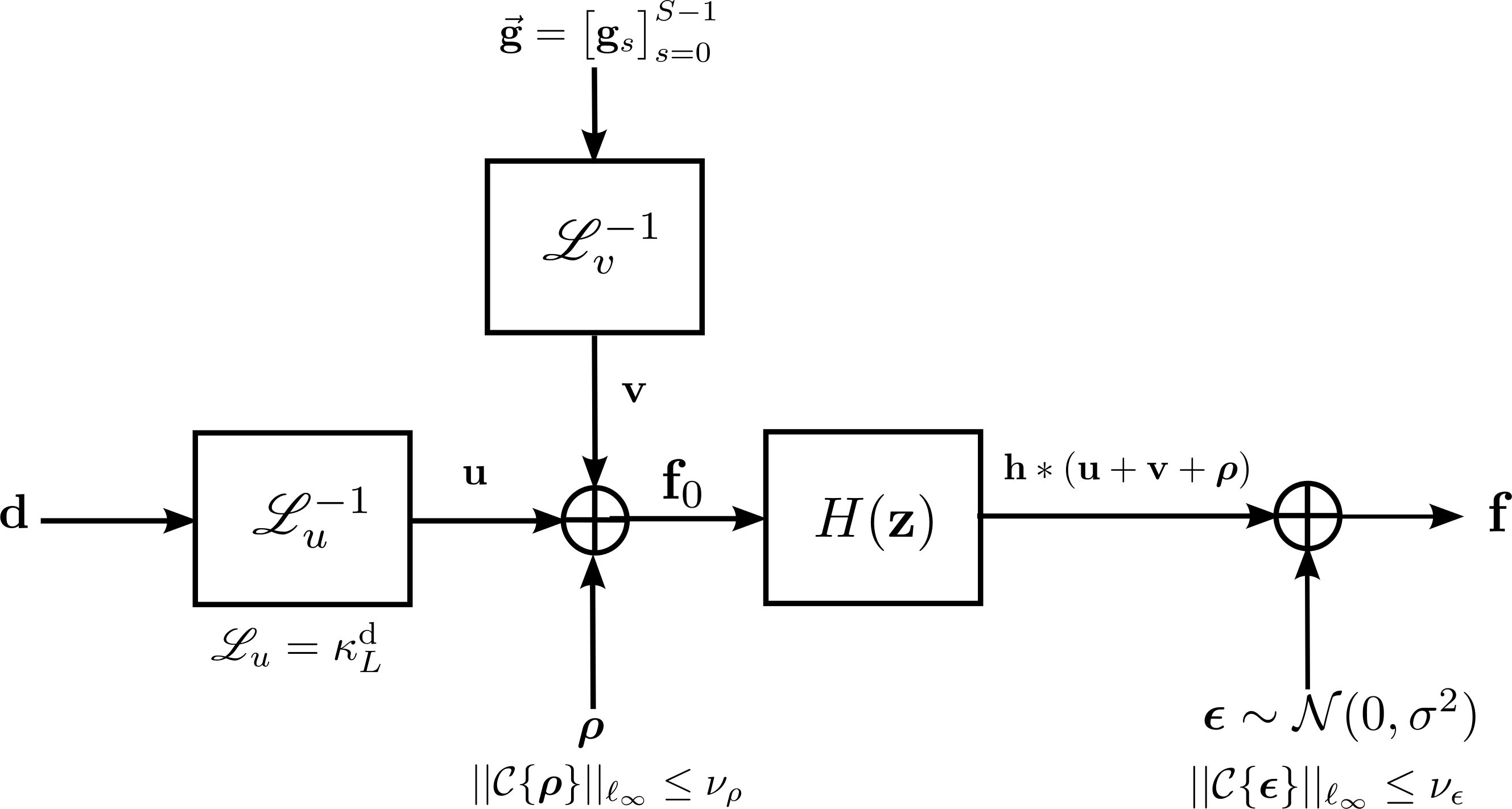}
 
 \caption{\label{fig:InnovationModel} Visualization of the innovation model for DMCD. The inverse operator $\mathscr L^{-1}_u$ does not exist in general, but the reconstructed image is obtained by the proposed minimization (\ref{eq:minimization:SDMCDD:2}).}
\end{center} 
\end{figure}



Similar to \cite{ThaiGottschlich2016DG3PD, ZhuTaiChan2013, ZhuTaiChan2013ALM},
in order to solve (\ref{eq:minimization:SDMCDD:2}) we introduce five new variables 
$\Bd \,, \vec{\Br} = \big[ \Br_l \big]_{l=0}^{L} \,, \vec{\Bt} = \big[ \Bt_l \big]_{l=0}^{L} \,, 
\vec{\By} = \big[ \By_l \big]_{l=0}^L \,, \vec{\Bw} = \big[ \Bw_s \big]_{s=0}^{S-1}$ where
\begin{align*}
\begin{cases}
 \abs{\vec{\Br}} - \big\langle \vec{\By} \,, \vec{\Br} \big\rangle_X = 0 \,, 
 \\
 \vec{\Br} = \big[ \nabla^+_L \Bu \,, 1 \big]       & \in X^{L+1} \,, 
 \\ 
 \Bd = \text{div}^-_L \vec{\Bt}      & \in X \,,  
 \\
 \vec{\Bt} = \vec{\By}      & \in X^{L+1} \,, 
 \\
 \vec{\By} = \big[ \By_l \big]_{l=0}^L    & \in \mathscr R \,,
 \\  
 \vec{\Bw} = \vec{\Bg}
 &\in X^S \,,
\end{cases} 
\end{align*}
with the indicator function on its feasible convex set
\begin{align*}
 \mathscr R^* (\vec{\By}) &= \begin{cases} 0 \,, & \vec{\By} \in \mathscr R \\ + \infty \,, & \text{else} \end{cases}   
 \,,~ \quad
 \mathscr R = \Big\{ \vec{\By} \in X^{L+1} ~:~ \abs{y_l[\Bk]} \leq 1 \,, l = 0, \ldots, L \,, \Bk \in \Omega \Big\} \,.
\end{align*}
The augmented Lagrangian method (ALM) is applied to (\ref{eq:minimization:SDMCDD:2}) by introducing Lagrange multipliers 
$\big( \boldsymbol{\lambda}_{\boldsymbol 1} \in X \,, \vec{\boldsymbol{\lambda}}_{\boldsymbol 2} = \big[ \boldsymbol{\lambda}_{\boldsymbol 2 l} \big]_{l=0}^L \in X^{L+1} \,,
\boldsymbol{\lambda}_{\boldsymbol 3} \in X \,, 
\vec{\boldsymbol{\lambda}}_{\boldsymbol 4} = \big[ \boldsymbol{\lambda}_{\boldsymbol 4 l} \big]_{l=0}^L \in X^{L+1} \,,
\boldsymbol{\lambda}_{\boldsymbol 5} \in X \,,
\vec{\boldsymbol{\lambda}}_{\boldsymbol 6} = \big[ \boldsymbol{\lambda}_{\boldsymbol 6 s} \big]_{s=0}^{S-1} \in X^S \,,
\boldsymbol{\lambda}_{\boldsymbol 7} \in X
\big)$
and positive parameters $\big[ \beta_i \big]_{i=1}^7 > 0$ as
\begin{align}  \label{eq:SimSegDeb:LagrangeFunc}
 &\min_{(\Bu, \Bv, \Brho, \Beps, \Bd, \vec{\Br}, \vec{\Bt}, \vec{\By}, \vec{\Bw}, \vec{\Bg}) \in X^{5+3(L+1)+2S}} 
 \cL (\Bu, \Bv, \Brho, \Beps, \Bd, \vec{\Br}, \vec{\Bt}, \vec{\By}, \vec{\Bw}, \vec{\Bg} \,;~ 
      \boldsymbol{\lambda}_{\boldsymbol 1}, \vec{\boldsymbol{\lambda}}_{\boldsymbol 2}, \boldsymbol{\lambda}_{\boldsymbol 3},
      \vec{\boldsymbol{\lambda}}_{\boldsymbol 4}, \boldsymbol{\lambda}_{\boldsymbol 5},
      \vec{\boldsymbol{\lambda}}_{\boldsymbol 6}, \boldsymbol{\lambda}_{\boldsymbol 7}) 
\end{align}
and the Lagrange function is
\begin{align*}
 &\cL (\cdot \,; \cdot) = 
 \norm{\Bd}_{\ell_1} + \mu_1 \sum_{s=0}^{S-1} \norm{ \Bw_s }_{\ell_1} + \mu_2 \norm{\Bv}_{\ell_1} 
 + \mathscr G^*\left(\frac{\Brho}{\nu_\rho}\right) + \mathscr G^*\left(\frac{\Beps}{\nu_\epsilon}\right) + \mathscr R^* (\vec{\By}) 
 + \Big\langle \boldsymbol{\lambda}_{\boldsymbol 1} + \beta_1 \,, \abs{\vec{\Br}} - \langle \vec{\By} \,, \vec{\Br} \rangle_X \Big\rangle_{\ell_2}
 \\&
 + \frac{\beta_2}{2} \sum_{l=0}^{L-1} \norm{ \Br_l - \partial^+_l \Bu + \frac{ \boldsymbol{\lambda}_{\boldsymbol 2 l} }{\beta_2} }^2_{\ell_2}
 + \frac{\beta_2}{2} \norm{ \Br_L - \boldsymbol 1 + \frac{\boldsymbol{\lambda}_{\boldsymbol 2 L}}{\beta_2} }^2_{\ell_2}
 + \frac{\beta_3}{2} \norm{ \Bd - \text{div}^-_L \vec{\Bt} + \frac{\boldsymbol{\lambda}_{\boldsymbol 3}}{\beta_3} }^2_{\ell_2}
 + \frac{\beta_4}{2} \norm{ \vec{\Bt} - \vec{\By} + \frac{\vec{\boldsymbol{\lambda}}_{\boldsymbol 4}}{\beta_4} }^2_{\ell_2}
 \\&
 + \frac{\beta_5}{2} \norm{ \Bf - \Bh \ast \Bu - \Bh \ast \Bv - \Bh \ast \Brho - \Beps + \frac{\boldsymbol{\lambda}_{\boldsymbol 5}}{\beta_5} }^2_{\ell_2}
 + \frac{\beta_6}{2} \sum_{s=0}^{S-1} \norm{\Bw_s - \Bg_s + \frac{\boldsymbol{\lambda}_{\boldsymbol 6 s}}{\beta_6}}^2_{\ell_2}
 + \frac{\beta_7}{2} \norm{ \Bv - \text{div}^-_S \vec{\Bg} + \frac{\boldsymbol{\lambda}_{\boldsymbol 7}}{\beta_7} }^2_{\ell_2} \,.
\end{align*}
Due to multi-variable minimization, the numerical solution of (\ref{eq:SimSegDeb:LagrangeFunc}) is obtained by 
applying the alternating directional method of multipliers through iteration $\tau = 1, 2, \ldots$ to find
\begin{align*} 
 & \Big( \Bu^{(\tau)}, \Bv^{(\tau)}, \Brho^{(\tau)}, \Beps^{(\tau)}, \Bd^{(\tau)}, \vec{\Br}^{(\tau)}, \vec{\Bt}^{(\tau)}, \vec{\By}^{(\tau)}, \vec{\Bw}^{(\tau)}, \vec{\Bg}^{(\tau)} \Big) =
 \\
 &\argmin ~
 \cL \Big(\Bu, \Bv, \Brho, \Beps, \Bd, \vec{\Br}, \vec{\Bt}, \vec{\By}, \vec{\Bw}, \vec{\Bg} \,;~ 
          \boldsymbol{\lambda}_{\boldsymbol 1}^{(\tau-1)}, \vec{\boldsymbol{\lambda}}_{\boldsymbol 2}^{(\tau-1)}, \boldsymbol{\lambda}_{\boldsymbol 3}^{(\tau-1)},
          \vec{\boldsymbol{\lambda}}_{\boldsymbol 4}^{(\tau-1)}, \boldsymbol{\lambda}_{\boldsymbol 5}^{(\tau-1)},
          \vec{\boldsymbol{\lambda}}_{\boldsymbol 6}^{(\tau-1)}, \boldsymbol{\lambda}_{\boldsymbol 7}^{(\tau-1)} \Big) \,.
\end{align*}
Given the initialization as 
$\Bu^{(0)} = \Bf \,, 
 \Bv^{(0)} = \Brho^{(0)} = \Beps^{(0)} = \Bd^{(0)} = \vec{\Br}^{(0)} = \vec{\Bt}^{(0)} = \vec{\By}^{(0)} = \vec{\Bw}^{(0)} = \vec{\Bg}^{(0)} 
 = \boldsymbol{\lambda}_{\boldsymbol 1}^{(0)} = \vec{\boldsymbol{\lambda}}_{\boldsymbol 2}^{(0)} 
 = \boldsymbol{\lambda}_{\boldsymbol 3}^{(0)} = \vec{\boldsymbol{\lambda}}_{\boldsymbol 4}^{(0)}
 = \boldsymbol{\lambda}_{\boldsymbol 5}^{(0)} = \vec{\boldsymbol{\lambda}}_{\boldsymbol 6}^{(0)}
 = \boldsymbol{\lambda}_{\boldsymbol 7}^{(0)} = \mathbf 0$,
we solve the following ten subproblems and then update the seven Lagrange multipliers in each iteration.


\noindent
{\bfseries The ``$\Bu$-problem'':} Fix 
$\Bv, \Brho, \Beps, \Bd, \vec{\Br}, \vec{\Bt}, \vec{\By}, \vec{\Bw}, \vec{\Bg}$ and then solve
\begin{align}  \label{eq:problem:u}
 \min_{\Bu \in X} \bigg\{  
 \frac{\beta_2}{2} \sum_{l=0}^{L-1} \norm{ \Br_l - \partial^+_l \Bu + \frac{ \boldsymbol{\lambda}_{\boldsymbol 2 l} }{\beta_2} }^2_{\ell_2}
 + \frac{\beta_5}{2} \norm{ \Bf - \Bh \ast \Bu - \Bh \ast \Bv - \Bh \ast \Brho - \Beps + \frac{\boldsymbol{\lambda}_{\boldsymbol 5}}{\beta_5} }^2_{\ell_2}
 \bigg\}.
\end{align}
Given $\Bk \in \Omega$, we denote the discrete Fourier transforms of 
$h[\Bk] \,, r_l[\Bk] \,, \lambda_{2l}[\Bk] \,, f[\Bk] \,, v[\Bk] \,, 
\rho[\Bk] \,, \epsilon[\Bk]$ and $\lambda_5[\Bk]$ by
$H(\Bz) \,, R_l(\Bz) \,, \Lambda_{2l}(\Bz) \,, F(\Bz) \,, V(\Bz) \,, P(\Bz) \,, \cE(\Bz)$
and $\Lambda_5(\Bz)$, respectively.
The minimizer of (\ref{eq:problem:u}) is solved in the Fourier domain as
\begin{equation*}
 \Bu^* = \RE \Big[ \cF^{-1} \Big\{ \frac{\cY(\Bz)}{\cX(\Bz)} \Big\} \Big] \,,
\end{equation*}
with
\begin{align*}
 \cY(\Bz) &= \beta_2 \sum_{l=0}^{L-1} \big[ \cos\left(\frac{\pi l}{L}\right)(z_2^{-1} - 1) + \sin\left(\frac{\pi l}{L}\right)(z_1^{-1} - 1) \big]
 \Big[ R_l(\Bz) + \frac{ \Lambda_{2 l}(\Bz) }{\beta_2} \Big]
 \\&
 + \beta_5 H(\Bz^{-1}) \Big[ F(\Bz) - H(\Bz) V(\Bz) - H(\Bz) P(\Bz) - \cE(\Bz) + \frac{\Lambda_5(\Bz)}{\beta_5} \Big] \,,
\end{align*}
and
\begin{align*}
 \cX(\Bz) &= \beta_2 \sum_{l=0}^{L-1} \abs{ \cos\left(\frac{\pi l}{L}\right)(z_2 - 1) + \sin\left(\frac{\pi l}{L}\right)(z_1 - 1) }^2 + \beta_5 \abs{H(\Bz)}^2 \,.
\end{align*}

\noindent
{\bfseries The ``$\Bv$-problem'':} Fix $\Bu, \Brho, \Beps, \Bd, \vec{\Br}, \vec{\Bt}, \vec{\By}, \vec{\Bw}, \vec{\Bg}$ and then solve
\begin{align}  \label{eq:problem:v}
 \min_{\Bv \in X} \bigg\{  
 \mu_2 \norm{\Bv}_{\ell_1} 
 + \frac{\beta_5}{2} \norm{ \Bh \ast \Bv - \Big( \Bf - \Bh \ast \Bu - \Bh \ast \Brho - \Beps + \frac{\boldsymbol{\lambda}_{\boldsymbol 5}}{\beta_5} \Big) }^2_{\ell_2}
 + \frac{\beta_7}{2} \norm{ \Bv - \Big( \text{div}^-_S \vec{\Bg} - \frac{\boldsymbol{\lambda}_{\boldsymbol 7}}{\beta_7} \Big) }^2_{\ell_2}
 \bigg\}.
\end{align}
Let $\boldmath{\delta}$ denote the matrix-valued Dirac delta function 
evaluated at $\boldmath{k} \in \Omega$. 
Then a solution of (\ref{eq:problem:v}) is
\begin{align*}  
 \Bv^{(\tau)} &= \Shrink \Big( \Bt_\Bv^{(\tau)} \,,~ \frac{\mu_2 \alpha^{(\tau)}}{\beta_5 + \alpha^{(\tau)} \beta_7} \Big) \,,~~ \tau = 1, \ldots 
\end{align*}
with
\begin{align*}
 \Bt_\Bv^{(\tau)} &= \frac{ \beta_5 }{ \beta_5 + \alpha^{(\tau)} \beta_7 }
 \Big( \Big[ \boldmath{\delta} - \alpha^{(\tau)} \check{\Bh} \ast \Bh \Big] \ast \Bv^{(\tau-1)}  
 + \alpha^{(\tau)} \check{\Bh} \ast \big( \Bf - \Bh \ast \Bu - \Bh \ast \Brho - \Beps + \frac{\boldsymbol{\lambda}_{\boldsymbol 5}}{\beta_5} \big) \Big) 
 \\
 &+ \frac{ \beta_7 \alpha^{(\tau)} }{ \beta_5 + \alpha^{(\tau)} \beta_7 }
 \Big( \underbrace{ -\sum_{s=0}^{S-1} \big[ \cos\left(\frac{\pi s}{S}\right) \Bg_s \BDt + \sin\left(\frac{\pi s}{S}\right) \BDoT \Bg_s \big] }_{= \text{div}^-_S \vec{\Bg} }
     - \frac{\boldsymbol{\lambda}_{\boldsymbol 7}}{\beta_7} \Big) \,.
\end{align*}
Using the inverse cumulative distribution function to get the quantile corresponding to probability  $\alpha$,
one can choose an adaptive $\mu_2$ at each iteration $\tau$ as
\begin{align*}
 \mu_2 = \frac{ \alpha_{\mu_2} \norm{\Bt_\Bv^{(\tau)}}_{\ell_\infty} (\beta_5 + \alpha^{(\tau)} \beta_7) }{ \alpha^{(\tau)} } \,.
\end{align*}

\noindent
{\bfseries The ``$\Bd$-problem'':} Fix $\Bu, \Bv, \Brho, \Beps, \vec{\Br}, \vec{\Bt}, \vec{\By}, \vec{\Bw}, \vec{\Bg}$ and then solve
\begin{align}  \label{eq:problem:d}
 \min_{\Bd \in X} \Big\{
 \norm{\Bd}_{\ell_1} 
 + \frac{\beta_3}{2} \norm{ \Bd - \Big( \text{div}^-_L \vec{\Bt} - \frac{\boldsymbol{\lambda}_{\boldsymbol 3}}{\beta_3} \Big) }^2_{\ell_2}
 \Big\}.
\end{align}
The solution of the $\ell_1$ minimization (\ref{eq:problem:d}) is obtained by the shrinkage operator as
\begin{align}  \label{eq:problem:d:solution}
 \Bd^* &= \Shrink \Big( \underbrace{ - \sum_{l=0}^L \big[ \cos\left(\frac{\pi l}{L}\right) \Bt_l \BDt + \sin\left(\frac{\pi l}{L}\right) \BDoT \Bt_l \big] }_{= \text{div}^-_L \vec{\Bt} }
 - \frac{\boldsymbol{\lambda}_{\boldsymbol 3}}{\beta_3} \,, \frac{1}{\beta_3} \Big) \,.
\end{align}

\noindent
{\bfseries The ``$\vec{\Br}$-problem'':} Fix $\Bu, \Bv, \Brho, \Beps, \Bd, \vec{\Bt}, \vec{\By}, \vec{\Bw}, \vec{\Bg}$ and solve
\begin{align}  \label{eq:problem:r}
 &\min_{\vec{\Br} \in X^{L+1}} \bigg\{  
 \Big\langle \boldsymbol{\lambda}_{\boldsymbol 1} + \beta_1 \,, \abs{\vec{\Br}} - \langle \vec{\By} \,, \vec{\Br} \rangle_X \Big\rangle_{\ell_2}
 + \frac{\beta_2}{2} \sum_{l=0}^{L-1} \norm{ \Br_l - \partial^+_l \Bu + \frac{ \boldsymbol{\lambda}_{\boldsymbol 2 l} }{\beta_2} }^2_{\ell_2}
 + \frac{\beta_2}{2} \norm{ \Br_L - \boldsymbol 1 + \frac{\boldsymbol{\lambda}_{\boldsymbol 2 L}}{\beta_2} }^2_{\ell_2}
 \bigg\}.
\end{align}
The minimizer of problem (\ref{eq:problem:r}) is obtained 
by the shrinkage operator as
\begin{align*}
 \Br_l^* &=
 \begin{cases}
  \displaystyle \Shrink \bigg( \Big[ \underbrace{ \cos\left(\frac{\pi l}{L}\right) \Bu \BDtT + \sin\left(\frac{\pi l}{L}\right) \BDo \Bu }_{= \partial_l^+ \Bu } \Big] 
        - \frac{ \boldsymbol{\lambda}_{\boldsymbol 2 l} }{ \beta_2 } + \frac{ \boldsymbol{\lambda_1} + \beta_1 }{ \beta_2 } \cdot^\times \By_l
  \,,~ \frac{ \boldsymbol{\lambda_1} + \beta_1}{\beta_2} \bigg)
  \,,~ & l = 0, \ldots, L-1
  \\
  \displaystyle \Shrink \left( \mathbf{1} - \frac{ \boldsymbol{\lambda}_{\boldsymbol 2 L} }{ \beta_2 } + \frac{ \boldsymbol{\lambda_1} + \beta_1 }{ \beta_2 } \cdot^\times \By_L
  \,,~ \frac{ \boldsymbol{\lambda_1} + \beta_1}{\beta_2}
  \right)
  \,,~ & l = L \,.
 \end{cases}
\end{align*}

\noindent
{\bfseries The ``$\vec{\Bt}$-problem'':} Fix $\Bu, \Bv, \Brho, \Beps, \Bd, \vec{\Br}, \vec{\By}, \vec{\Bw}, \vec{\Bg}$ and solve
\begin{align}  \label{eq:problem:t}
 \min_{\vec{\Bt} \in X^{L+1}} \bigg\{  
 \frac{\beta_3}{2} \norm{ \Bd - \text{div}^-_L \vec{\Bt} + \frac{\boldsymbol{\lambda}_{\boldsymbol 3}}{\beta_3} }^2_{\ell_2}
 + \frac{\beta_4}{2} \norm{ \vec{\Bt} - \vec{\By} + \frac{\vec{\boldsymbol{\lambda}}_{\boldsymbol 4}}{\beta_4} }^2_{\ell_2}
 \bigg\}.
\end{align}
Similar to the ``$\Bu$-problem'', we denote
$Y_l(\Bz) \,, \Lambda_{4l}(\Bz) \,, D(\Bz) \,, T_l(\Bz)$
and $\Lambda_3(\Bz)$ as the discrete Fourier transforms of 
$y_l[\Bk] \,, \lambda_{4l}[\Bk] \,, d[\Bk] \,, t_l[\Bk]$
and $\Lambda_3[\Bk]$ with $\Bk \in \Omega$, respectively.
The solution of this ``$\vec{\Bt}$-problem'' (\ref{eq:problem:t}) for each separable problem
$l = 0, \ldots, L$ is
\begin{align*}
 \begin{cases}
  \displaystyle
  \Bt_l = \RE \Big[ \cF^{-1} \Big\{ \frac{\cM_l(\Bz)}{\cN_l(\Bz)} \Big\} \Big]
  \,,~ l = 0, \ldots, L-1 \,,
  \\
  \displaystyle
  \Bt_L = \By_L - \frac{ \boldsymbol{\lambda}_{\boldsymbol 4L} }{ \beta_4 } \,,
 \end{cases}
\end{align*}
with
\begin{align*}
 \cN_l(\Bz) &= \beta_4 + \beta_3 \abs{ \cos\left(\frac{\pi l}{L}\right) (z_2 - 1) + \sin\left(\frac{\pi l}{L}\right) (z_1 - 1) }^2 \,,
 \\ 
 \cM_l(\Bz) &= \beta_4 \left[ Y_l(\Bz) - \frac{ \Lambda_{4l}(\Bz) }{ \beta_4 } \right]
 -\beta_3 \left[ \cos\left(\frac{\pi l}{L}\right) (z_2 - 1) + \sin\left(\frac{\pi l}{L}\right) (z_1 - 1) \right] \times
 \\ & \qquad ~
 \left[ D(\Bz) + \sum_{l'=[0, L-1] \backslash \{l\}} \left[ \cos\left(\frac{\pi l'}{L}\right) (z_2^{-1} - 1) + \sin\left(\frac{\pi l'}{L}\right) (z_1^{-1} - 1) \right] T_{l'}(\Bz) + \frac{ \Lambda_3(\Bz) }{ \beta_3 } \right] \,.
\end{align*}
Note that $\cD(\Bz)$ is similar to the auto-correlation function in the Riesz basis \cite{Unser2000}. 

\noindent
{\bfseries The ``$\vec{\By}$-problem'':} Fix $\Bu, \Bv, \Brho, \Beps, \Bd, \vec{\Br}, \vec{\Bt}, \vec{\Bw}, \vec{\Bg}$ and solve
\begin{align}  \label{eq:problem:y}
 \min_{\vec{\By} \in X^{L+1}} \bigg\{  
 \mathscr R^* (\vec{\By}) 
 + \Big\langle \boldsymbol{\lambda}_{\boldsymbol 1} + \beta_1 \,, \abs{\vec{\Br}} - \langle \vec{\By} \,, \vec{\Br} \rangle_X \Big\rangle_{\ell_2}
 + \frac{\beta_4}{2} \norm{ \vec{\Bt} - \vec{\By} + \frac{\vec{\boldsymbol{\lambda}}_{\boldsymbol 4}}{\beta_4} }^2_{\ell_2}
 \bigg\}.
\end{align}
Due to its separability, we consider the problem at $l = 0, \ldots, L-1$ and
the solution of (\ref{eq:problem:y}) is
\begin{align*}
 \By_l^* &= \begin{cases}
             \By'_l \,, & \abs{\vec{\By'}} \leq 1  \\
             \frac{ \By'_l }{ \abs{\vec{\By'}} } \,, & \abs{\vec{\By'}} > 1
            \end{cases}
            \,,~~ l = 0, \ldots, L
 \,, \quad \text{and} \quad
  \begin{cases}
   \By'_l &= \Bt_l + \frac{ \boldsymbol{\lambda}_{\boldsymbol 4 l} }{\beta_4} + \Br_l \cdot^\times \frac{\boldsymbol{\lambda_1} + \beta_1}{\beta_4} \,,
   \\
   \\
   \abs{\vec{\By'}} &= \sqrt{ \sum_{l=0}^L \Big[ \Bt_l + \frac{ \boldsymbol{\lambda}_{\boldsymbol 4 l} }{\beta_4} + \Br_l \cdot^\times \frac{\boldsymbol{\lambda_1} + \beta_1}{\beta_4} \Big]^{\cdot 2} } \,.
  \end{cases}  
\end{align*}

\noindent
{\bfseries The ``$\vec{\Bw}$-problem'':} Fix $\Bu, \Bv, \Brho, \Beps, \Bd, \vec{\Br}, \vec{\Bt}, \vec{\By}, \vec{\Bg}$ and solve
\begin{align}  \label{eq:problem:w}
 \min_{\vec{\Bw} \in X^S} \bigg\{ \mu_1 \sum_{s=0}^{S-1} \norm{ \Bw_s }_{\ell_1} 
 + \frac{\beta_6}{2} \sum_{s=0}^{S-1} \norm{\Bw_s - \Bg_s + \frac{\boldsymbol{\lambda}_{\boldsymbol 6 s}}{\beta_6}}^2_{\ell_2}
 \bigg\}.
\end{align}
Due to the $\ell_1$-minimization, a solution of (\ref{eq:problem:w}) for each separable problem $s = 0, \ldots, S-1$ is
\begin{align*}
 \Bw_s^* = \Shrink \Big( \Bg_s - \frac{\boldsymbol{\lambda}_{\boldsymbol 6 s}}{\beta_6} \,,~ \frac{\mu_1}{\beta_6} \Big) \,.
\end{align*}
Note that $\mu_1$ can be adaptively chosen as
$\mu_1 = \beta_6 \alpha_{\mu_1} \norm{ \Bg_s - \frac{\boldsymbol{\lambda}_{\boldsymbol 6 s}}{\beta_6} }_{\ell_\infty}$.

\noindent
{\bfseries The ``$\vec{\Bg}$-problem'':} Fix $\Bu, \Bv, \Brho, \Beps, \Bd, \vec{\Br}, \vec{\Bt}, \vec{\By}, \vec{\Bw}$ and solve
\begin{align}  \label{eq:problem:g}
 \min_{\vec{\Bg} \in X^S} \bigg\{ 
 \frac{\beta_6}{2} \sum_{s=0}^{S-1} \norm{\Bw_s - \Bg_s + \frac{\boldsymbol{\lambda}_{\boldsymbol 6 s}}{\beta_6}}^2_{\ell_2}
 + \frac{\beta_7}{2} \norm{ \Bv - \text{div}^-_S \vec{\Bg} + \frac{\boldsymbol{\lambda}_{\boldsymbol 7}}{\beta_7} }^2_{\ell_2} 
 \bigg\}.
\end{align}
Due to the higher order partial differential equation (in a discrete setting) of the Euler-Lagrange equation for (\ref{eq:problem:g}),
the solution is obtained in the Fourier domain, which is similar to the ``$\Bu$-problem'' and the ``$\vec{\Bt}$-problem''.
We denote
$W_s(\Bz) \,, \Lambda_{6s}(\Bz) \,, G_s(\Bz)$
and $\Lambda_7(\Bz)$ as the discrete Fourier transform of 
$w_s[\Bk] \,, \lambda_{6s}[\Bk] \,, g_s[\Bk]$
and $\lambda_7[\Bk]$ with $\Bk \in \Omega$, respectively.
The solution of this ``$\vec{\Bg}$-problem'' (\ref{eq:problem:g}) for each separable problem
$s = 0, \ldots, S-1$ is
\begin{align*}
G_s(\Bz) = \frac{ \cB_s(\Bz) }{ \cA_s(\Bz) } \,,
\end{align*}
with
\begin{align*}
 &\cA_s(\Bz) = \beta_6 + \beta_7 \abs{ \cos\left(\frac{\pi s}{S}\right)(z_2 - 1) + \sin\left(\frac{\pi s}{S}\right)(z_1 - 1) }^2 \,,
 \\
 &\cB_s(\Bz) = \beta_6 \left[ W_s(\Bz) + \frac{\Lambda_{6 s}(\Bz)}{\beta_6} \right]
 -\beta_7 \left[ \cos\left(\frac{\pi s}{S}\right)(z_2 - 1) + \sin\left(\frac{\pi s}{S}\right) (z_1 - 1) \right] \times
 \\&
 \left[ V(\Bz) + \sum_{s'=[0, S-1] \backslash \{s\}} \left[ \cos\left(\frac{\pi s'}{S}\right)(z_2^{-1} - 1) + \sin\left(\frac{\pi s'}{S}\right) (z_1^{-1} - 1) \right] G_{s'}(\Bz) + \frac{\Lambda_7(\Bz)}{\beta_7} \right] \,.
\end{align*}

\noindent
{\bfseries The ``$\Brho$-problem'':} Fix $\Bu, \Bv, \Beps, \Bd, \vec{\Br}, \vec{\Bt}, \vec{\By}, \vec{\Bw}, \vec{\Bg}$ and solve
\begin{align}  \label{eq:problem:rho}
 \min_{\Brho \in X} \Big\{
 \mathscr G^*\left(\frac{\Brho}{\nu_\rho}\right) + 
 \frac{\beta_5}{2} \norm{ \Bh \ast \Brho - \Big[ \Bf - \Bh \ast \Bu - \Bh \ast \Bv - \Beps + \frac{\boldsymbol{\lambda}_{\boldsymbol 5}}{\beta_5} \Big] }^2_{\ell_2} 
 \Big\}.
\end{align}
Minimization of (\ref{eq:problem:rho}) is approximated by the first-order Taylor expansion as
\begin{align*}
 \min_{\Brho \in X} \Big\{
 \mathscr G^*(\frac{\Brho}{\nu_\rho}) + 
 \frac{\beta_5}{2 \alpha^{(\tau)}} \Big|\Big| \Brho - \Big[ \underbrace{ \big( \delta - \alpha^{(\tau)} \check \Bh \ast \Bh \big) \ast \Brho^{(\tau-1)} + \alpha^{(\tau)} \check \Bh \ast \big( \Bf - \Bh \ast \Bu - \Bh \ast \Bv - \Beps + \frac{\boldsymbol{\lambda}_{\boldsymbol 5}}{\beta_5} \big) }_{\displaystyle = \tilde{\Brho}} \Big] \Big|\Big|^2_{\ell_2} 
 \Big\} \,.
\end{align*}
Its solution is 
\begin{align*}
 \Brho &= \tilde{\Brho} - \CST \Big[ \tilde{\Brho} \,,~ \nu_\rho \Big] \,,
\end{align*}
where $\nu_\rho$ can be selected by the $\alpha$-quantile as
$\nu_\rho = \alpha_\rho \norm{ \cC \{ \tilde{\Brho} \} }_{\ell_\infty}$.

Note that without component $\Brho$ in minimization (\ref{eq:minimization:SDMCDD:2}),
it is difficult to separate texture $\Bv$ and some very fine scale structures in an original image because of a blurring kernel. 

\noindent
{\bfseries The ``$\Beps$-problem'':} Fix $\Bu, \Bv, \Brho, \Bd, \vec{\Br}, \vec{\Bt}, \vec{\By}, \vec{\Bw}, \vec{\Bg}$ and solve
\begin{align}  \label{eq:problem:epsilon}
 \min_{\Beps \in X} \Big\{
 \mathscr G^*(\frac{\Beps}{\nu_\epsilon}) + 
 \frac{\beta_5}{2} \big|\big| \Beps - \Big[ \underbrace{ \Bf - \Bh \ast \Bu - \Bh \ast \Bv - \Bh \ast \Brho + \frac{\boldsymbol{\lambda}_{\boldsymbol 5}}{\beta_5} }_{\displaystyle = \tilde{\Beps}} \Big] \big|\big|^2_{\ell_2} 
 \Big\}.
\end{align}
Similar to the ``$\Brho$-problem'' (\ref{eq:problem:rho}) and given a possible choice $\nu_\epsilon = \alpha_\epsilon \norm{ \cC \{ \tilde{\Beps} \} }_{\ell_\infty}$,
the solution of (\ref{eq:problem:epsilon}) is
\begin{align*}
 \Beps^* = \tilde{\Beps} - \CST \Big[ \tilde{\Beps} \,,~ \nu_\epsilon \Big] \,. 
\end{align*}


The final solution is found by iteratively updating the Lagrange multipliers
 $\big( \boldsymbol{\lambda}_{\boldsymbol 1} \,, \vec{\boldsymbol{\lambda}}_{\boldsymbol 2} \,, \boldsymbol{\lambda}_{\boldsymbol 3}
  \,, \vec{\boldsymbol{\lambda}}_{\boldsymbol 4} \,, \boldsymbol{\lambda}_{\boldsymbol 5} \,, \vec{\boldsymbol{\lambda}}_{\boldsymbol 6} \,, \boldsymbol{\lambda}_{\boldsymbol 7} \big)
  \in X^{4+2(L+1)+S
  }
 $
 
\begin{align*}
 \boldsymbol{\lambda}_{\boldsymbol 1}^{(\tau)} &= \boldsymbol{\lambda}_{\boldsymbol 1}^{(\tau-1)} 
 + \beta_1 \Big[ \abs{\vec{\Br}} - \langle \vec{\By} \,, \vec{\Br} \rangle_X \Big]  
 \,,~~ \abs{\vec{\Br}} = \sqrt{ \sum_{l=0}^L \Br_l^{\cdot 2} } \,,~~ 
 \langle \vec{\By} \,, \vec{\Br} \rangle_X = \sum_{l=0}^L \By_l \cdot^\times \Br_l 
 \\
 \boldsymbol{\lambda}_{\boldsymbol 2 l}^{(\tau)} &= 
 \begin{cases}  
  \boldsymbol{\lambda}_{\boldsymbol 2 l}^{(\tau-1)} + \beta_2 \big[ \Br_l - \cos\left(\frac{\pi l}{L}\right) \Bu \BDtT - \sin\left(\frac{\pi l}{L}\right) \BDo \Bu \big] \,, & l = 0, \ldots, L-1 \\
  \boldsymbol{\lambda}_{\boldsymbol 2 l}^{(\tau-1)} + \beta_2 \big[ \Br_l -  \boldsymbol 1 \big] \,, & l = L
 \end{cases}
 \\
 \boldsymbol{\lambda}_{\boldsymbol 3}^{(\tau)} &= \boldsymbol{\lambda}_{\boldsymbol 3}^{(\tau-1)}
 + \beta_3 \Big[ \Bd + \sum_{l=0}^{L-1} \big[ \cos\left(\frac{\pi l}{L}\right) \Bt_l \BDt + \sin\left(\frac{\pi l}{L}\right) \BDoT \Bt_l \big] \Big]
 \\
 \boldsymbol{\lambda}_{\boldsymbol 4 l}^{(\tau)} &= \boldsymbol{\lambda}_{\boldsymbol 4 l}^{(\tau-1)}
 + \beta_4 \Big[ \Bt_l - \By_l \Big] \,,~~ l = 0, \ldots, L
 \\
 \boldsymbol{\lambda}_{\boldsymbol 5}^{(\tau)} &= \boldsymbol{\lambda}_{\boldsymbol 5}^{(\tau-1)}
 + \beta_5 \Big[ \Bf - \Bh \ast \Bu - \Bh \ast \Bv - \Bh \ast \Brho - \Beps \Big]  
 \\
 \boldsymbol{\lambda}_{\boldsymbol 6 s}^{(\tau)} &= \boldsymbol{\lambda}_{\boldsymbol 6 s}^{(\tau-1)}
 + \beta_6 \Big[ \Bw_s - \Bg_s \Big]     
 \,,~ s = 0, \ldots, S-1 
 \\
 \boldsymbol{\lambda}_{\boldsymbol 7}^{(\tau)} &= \boldsymbol{\lambda}_{\boldsymbol 7}^{(\tau-1)}
 + \beta_7 \Big[ \Bv + \underbrace{ \sum_{s=0}^{S-1} \left[ \cos\left(\frac{\pi s}{S}\right) \Bg_s \BDt + \sin\left(\frac{\pi s}{S}\right) \BDoT \Bg_s \right] }_{= -\text{div}^-_S \vec{\Bg}} \Big] \,.     
\end{align*}
As in \cite{ThaiGottschlich2016DG3PD, ThaiGottschlich2016G3PD}, 
relative error on the log scale shows the convergence of the algorithm as
\begin{align} \label{eq:relativeError:v}
 \text{Err}_{\Bv}(\tau) = \log \frac{\norm{\Bv^{(\tau)} - \Bv^{(\tau-1)}}_{\ell_2}}{\norm{\Bv^{(\tau-1)}}_{\ell_2}} \,,~ \tau = 1, \ldots \,.
\end{align}
Note that the convergence can be performed by other criteria, see \cite{ZhuTaiChan2013, ZhuTaiChan2013ALM, TaiHahnChung2011}.
We use this criterion because the problem is convex in $\boldmath{v}$,
and because our method emphasizes texture recovery.  
Figure \ref{fig:deblurdecomp:barbara:nonoise} illustrates a performance of our demixing model in terms of simultaneously decomposing and deblurring 
through directional mean curvature.
This result will be clearly explained in later sections after we set up 
the link between the DMCD model and filter banks in harmonic analysis. 
The Algorithms 1-4 in Appendix C summarize a numerical solution of the DMCD model 
(\ref{eq:minimization:SDMCDD:1}).


\section{Variational Analysis and Filter Banks}   \label{sec:VariationalAnalysisandFilterBanks}

In this section, we establish a deep connection between the proposed demixing 
model and multiscale harmonic analysis.
Specifically, we analyze filter banks generated by the DMCD model at 
iteration $\tau$ in the ordering for the 
$(\Bu \,, \Bv \,, \Brho \,, \Beps)$-problems according to the Algorithms 1-4
in Appendix C.
These filter banks are similar to a wavelet-like operator \cite{KhalidovUnser2006, UnserVandeville2010, UnserSageVandeville2009}.
We then generalize the concept of filter banks in the $\Bu$-problem and the 
$\vec \Bg$-problem to continuous and discrete multiscale sampling versions.
The mathematical proofs are described in proposition 
\ref{prop:VCPyramid:problem:u:spatial}-\ref{prop:DMCDDMultiscaleSamplingTheory:g-problem}
in Appendix A.
To summarize the filter banks of these solutions, we refer the reader to 
Algorithms 5-7 in Appendix C 
and Figure \ref{fig:MDCD_let:FilterBanks:c0_1}, 
\ref{fig:MDCD_let:FilterBanks:c10}, \ref{fig:MDCD_let:FilterBanks:3D} (for a version of sampling theory) and Figure \ref{fig:MDCD_let:u-g-problems} (for multiscale version).
Although the concept of filter banks and scaling/wavelet functions are 
different in harmonic analysis, we shall combine these two concepts in this section.

\subsection{The ``$\Bu$-problem''}  \label{subsection:uproblem:multiscale}

Consider the sampling theory form for the $\Bu$-problem (see proposition \ref{prop:VCPyramid:problem:u:spatial} in Appendix A).
Given $\Bk \in \Omega$ and taking $\beta_2 = c_{25} \beta_5 \,, \beta_3 = c_{34} 
\beta_4$, a solution of the $\Bu$-problem (\ref{eq:problem:u}) at iteration 
$\tau$ can be rewritten in a form of sampling theory as
\begin{align} \label{eq:VCPyramid:problem:u:spatial:1}
 &u^{(\tau)}[\Bk] = \left( \phi^{L, c_{25}} \ast \check{h} \ast \left( f - h \ast v^{(\tau-1)} - h \ast \rho^{(\tau-1)} - \epsilon^{(\tau-1)} + \frac{\lambda_5^{(\tau-1)}}{\beta_5} \right) \right)[\Bk]   \notag
 \\&
 + \sum_{l=0}^{L-1} \left( \check{\tilde{\psi}}^{L, c_{25}}_l \ast 
 \left[ \Shrink \left( \psi^L_l \ast u^{(\tau-1)} - \frac{ \lambda_{2 l}^{(\tau-1)} }{ \beta_2 } + \frac{ \lambda_1^{(\tau-1)} + \beta_1 }{ \beta_2 } y_l^{(\tau-1)} 
                     \,,~ \frac{ \lambda_1^{(\tau-1)} + \beta_1}{\beta_2} \right)       
       + \frac{ \lambda_{2 l}^{(\tau-1)} }{\beta_2} \right]
 \right) [\Bk]
\end{align}
with three definitions.

\noindent
{\bfseries Definition the of $\boldsymbol{y_l^{(\tau-1)} [\Bk]}$-component.} 
At a direction $l = 0, \ldots, L-1$, we have 
$y_l^{(\tau-1)} [\Bk] = \text{Proj}_{[-1,1]} \left[ \vec{y}\prime^{(\tau-1)}[\Bk] \right]$
with $\vec{y}^{\prime^(\tau-1)}[\Bk] = \left[ y_l\prime^{(\tau-1)}[\Bk] \right]_{l=0}^{L-1}$ and
\begin{align} \label{eq:VCPyramid:problem:y:1}
 y_l^{\prime^(\tau-1)}[\Bk] 
 &= \left( \xi^{L, c_{34}}_l \ast \left[ y_l^{(\tau-2)} - \frac{ \lambda_{4l}^{(\tau-2)} }{ \beta_4 } \right] \right) [\Bk]
 + \frac{ \lambda_{4l}^{(\tau-2)}[\Bk] }{\beta_4}              \notag              
 \\&
 + \left( \theta^{L, c_{34}}_l \ast \left[  \text{Shrink} \left( \sum_{l'=0}^{L-1} \check{\tilde \theta}^L_{l'} \ast t_{l'}^{(\tau-2)} - \frac{\lambda_3^{(\tau-2)}}{\beta_3} \,,~ \frac{1}{\beta_3} \right) + \frac{ \lambda_3^{(\tau-2)} }{ \beta_3 }
         - \sum_{l'=[0, L-1] \backslash \{l\}} \check{\tilde \theta}^L_{l'} \ast  t_{l'}^{(\tau-2)} \right]
   \right) [\Bk]                                                               \notag                                                    
 \\& 
 + \frac{\lambda_1^{(\tau-2)}[\Bk] + \beta_1}{\beta_4} 
   \Shrink \left( \partial_l^+ u^{(\tau-2)} - \frac{ \lambda_{2 l}^{(\tau-2)} }{ \beta_2 } 
                  + \frac{ \lambda_1^{(\tau-2)} + \beta_1 }{ \beta_2 } y_l^{(\tau-2)}
                  \,, \frac{ \lambda_1^{(\tau-2)} + \beta_1}{\beta_2} \right) [\Bk].    
\end{align}
Note that due to the high order PDE behind directional mean curvature, 
$u^{(\tau)}$ in (\ref{eq:VCPyramid:problem:u:spatial:1}) is updated from 
$u^{(\tau-1)}$ and $u^{(\tau-2)}$ at every iteration $\tau$ and a thresholding value adaptively depends on 
Lagrange multiplier $\lambda_1^{(\tau-1)}$ and $\lambda_1^{(\tau-2)}$.

\noindent
{\bfseries Definition of frames in (\ref{eq:VCPyramid:problem:u:spatial:1}) at direction $\boldsymbol{l = 0, \ldots, L-1}$.}
To simply the notation, we use a parameter $c > 0$ instead of $c_{25}$ to define frames as
\begin{align}  \label{eq:VCPyramid:problem:frameElements:u:scalingfunc:1}
 \phi^{L, c}[\Bk] 
 &~\stackrel{\cF}{\longleftrightarrow}~
 \Phi^{L, c}(\Bz) = \frac{1}{\displaystyle \abs{H(\Bz)}^2 + c \sum_{l'=0}^{L-1} \abs{ \cos\left(\frac{\pi l'}{L}\right)(z_2 - 1) + \sin\left(\frac{\pi l'}{L}\right)(z_1 - 1) }^2 } \,,
 \\  \label{eq:VCPyramid:problem:frameElements:u:dualwavelet:1}
 \check{\tilde \psi}^{L, c}_l[\Bk] = - c \partial_l^- \phi^{L, c}[\Bk]
 &~\stackrel{\cF}{\longleftrightarrow}~ 
 \tilde{\Psi}^{L, c}_l(\Bz^{-1}) 
 = \frac{\displaystyle c \left[ \cos\left(\frac{\pi l}{L}\right)(z_2^{-1} - 1) + \sin\left(\frac{\pi l}{L}\right)(z_1^{-1} - 1) \right] }
         {\displaystyle c \sum_{l'=0}^{L-1} \abs{ \cos\left(\frac{\pi l'}{L}\right)(z_2 - 1) + \sin\left(\frac{\pi l'}{L}\right)(z_1 - 1) }^2 + \abs{H(\Bz)}^2 } \,,
 \\  \label{eq:VCPyramid:problem:frameElements:u:wavelet:1}
 \psi^L_l[\Bk] = \partial_l^+ \delta[\Bk]
 &~\stackrel{\cF}{\longleftrightarrow}~
 \Psi^L_l(\Bz) = \cos\left(\frac{\pi l}{L}\right)(z_2 - 1) + \sin\left(\frac{\pi l}{L}\right)(z_1 - 1) \,
\end{align}
where $\delta[\Bk]$ denotes the Dirac delta function evaluated at 
position $\Bk$. 
Note that given the spectrum of an impulse response for a blur operator $H(\Bz)$ with $\abs{H(e^{j \mathbf 0})}^2 > 0$ (usually $H(e^{j \mathbf 0} = 1$) due to its lowpass-like kernel),
these bounded frames satisfy the unity condition:
\begin{align*}
 & \abs{H(\Bz)}^2 \Phi^{L, c}(\Bz) + \sum_{l=0}^{L-1} \tilde \Psi^{L, c}_l(\Bz^{-1}) \Psi^L_l(\Bz) = 1 \,,
 \\
 &\Phi^{L, c}(e^{j \mathbf 0}) = \frac{1}{ \abs{H(e^{j \mathbf 0})}^2 }
 \quad \text{and} \quad
 \tilde{\Psi}^{L, c}_l(e^{j \mathbf 0}) = \Psi^L_l(e^{j \mathbf 0}) = 0 \,.
\end{align*}

\noindent
{\bfseries Definition of frames in (\ref{eq:VCPyramid:problem:y:1}) at 
direction $\boldsymbol{l = 0, \ldots, L-1}$.} We also simplify the notation by using $c>0$ instead of $c_{34}$, so
\begin{align} \label{eq:FilterBanks:XiThetaTheta_tilde:1:1}
 \xi^{L, c}_l[\Bk] = \Big[ 1 - c \partial_l^- \partial_l^+ \Big]^{-1} \delta[\Bk]
 &~\stackrel{\cF}{\longleftrightarrow}~ 
 \Xi^{L, c}_l(\Bz) = \frac{ 1 }{ 1 + c \abs{ \cos\left(\frac{\pi l}{L}\right) (z_2 - 1) + \sin\left(\frac{\pi l}{L}\right) (z_1 - 1) }^2 }  
 \,,
 \\  \label{eq:FilterBanks:XiThetaTheta_tilde:2:1}
 \theta^{L, c}_l[\Bk] = -c \partial_l^+ \xi^{L, c}_l[\Bk]
 &~\stackrel{\cF}{\longleftrightarrow}~   
 \Theta^{L, c}_l(\Bz) = 
 \frac{ -c \Big[ \cos\left(\frac{\pi l}{L}\right) (z_2 - 1) + \sin\left(\frac{\pi l}{L}\right) (z_1 - 1) \Big] }
        { 1 + c \abs{ \cos\left(\frac{\pi l}{L}\right) (z_2 - 1) + \sin\left(\frac{\pi l}{L}\right) (z_1 - 1) }^2 } \,,
 \\  \label{eq:FilterBanks:XiThetaTheta_tilde:3:1}
 \check{\tilde \theta}^L_l[\Bk] = \partial_l^- \delta[\Bk]
 &~\stackrel{\cF}{\longleftrightarrow}~  
 \tilde \Theta^L_l(\Bz^{-1}) = - \left[ \cos\left(\frac{\pi l}{L}\right) (z_2^{-1} - 1) + \sin\left(\frac{\pi l}{L}\right) (z_1^{-1} - 1) \right] \,.
\end{align}
Due to the splitting method for minimization in 
(\ref{eq:minimization:SDMCDD:2}), 
there is no effect from a blur operator $H(\Bz)$ on these frames.
Moreover, it is easy to see that these frames are bounded above from zero for $\Bome \in [-\pi \,, \pi]^2$ and also satisfy the unity condition in the 
Fourier domain as
\begin{align} \label{eq:FilterBanks:XiThetaTheta_tilde:unitycondition:1}
 &\Xi^{L, c}_l(\Bz) + \Theta^{L, c}_l(\Bz) \tilde \Theta^L_l(\Bz^{-1}) = 1 \,,
 \\
 & \Xi^{L, c}_l(e^{j\mathbf 0}) = 1 
 \text{ and }
 \Theta^{L, c}_l(e^{j\mathbf 0}) = \tilde \Theta^L_l(e^{j\mathbf 0}) = 0 \,.   \notag
\end{align}

Figures \ref{fig:MDCD_let:FilterBanks:c0_1} and \ref{fig:MDCD_let:FilterBanks:c10}
depict the spectrum of these filter banks with $c = 0.1$ and $c = 10$, respectively.
Figure \ref{fig:MDCD_let:FilterBanks:3D} illustrates a 3-dimensional version of $\Phi^{L,c}(\Bz)$ 
and $\Phi^{L,c}(\Bz)$ without the blurring effect; i.e., $H(\Bz) = 1$.


Now consider the multiscale sampling version of the $\Bu$-problem (see proposition \ref{prop:DMCDDMultiscaleSamplingTheory:u-problem}).
Since a solution of the $\Bu$-problem can be described in a form of 
the sampling theory in harmonic analysis, we generalize this form to its 
(continuous and discrete) multi-scale version (with some simplified notation 
as in the proof of proposition \ref{prop:DMCDDMultiscaleSamplingTheory:u-problem} in Appendix A).

For the discrete case, given a function $f \in \ell_2(\mathbb R^2)$, a constant $a>0$ 
and $\Bk \in \Omega$, the discrete multiscale sampling theory at scale $I$ 
and direction $L$ is 
 \begin{align} \label{label:multiScaleSamp:discreteprojection:prop}
 f[\Bk] = (f \ast \phi){[\Bk]}
 + \sum_{i=0}^{I-1} \sum_{l=0}^{L-1} ( f \ast \check{ \tilde{\psi} }_{il} \ast \psi_{il} ){[\Bk]}
 ~\stackrel{\cF}{\longleftrightarrow}~           
 F(\Bz) = F(\Bz) \Phi(\Bz)
 + \sum_{i=0}^{I-1} \sum_{l=0}^{L-1} F(\Bz) \tilde \Psi_{il}(\Bz^{-1}) \Psi_{il}(\Bz) \,.
\end{align}
Their frames are defined in the Fourier domain (see Figure \ref{fig:MDCD_let:u-g-problems}(b) for their spectra) as
\begin{align*}
 \phi[\Bk] 
 &~\stackrel{\cF}{\longleftrightarrow}~ 
 \Phi(\Bz) = I^{-1} \sum_{i=0}^{I-1} \Phi_\text{int}(\Bz^{a^i}) \,,
 \\
 \psi_{il}[\Bk] 
 &~\stackrel{\cF}{\longleftrightarrow}~ 
 \Psi_{il}(\Bz) = I^{-\frac{1}{2}} \Psi_l(\Bz^{a^i}) \,,
 \\
 \check{\tilde \psi}_{il}[\Bk] 
 &~\stackrel{\cF}{\longleftrightarrow}~ 
 \tilde \Psi_{il}(\Bz^{-1}) = I^{-\frac{1}{2}} \tilde \Psi_l(\Bz^{-a^i}) \,,
\end{align*}
with the discrete version of the interpolant $\phi_\text{int}(\cdot)$ and the directional mother dual/primal wavelet $\tilde \psi_l(\cdot)$ (with $l = 0, \ldots, L-1$) as
\begin{align*}
 \phi_\text{int}[\Bk] &= \left[ c (-\Delta_{\text{d}L}) + 1 \right]^{-1} \delta[\Bk]
 ~\stackrel{\cF}{\longleftrightarrow}~
 \Phi_\text{int}(\Bz) = \frac{1}{\displaystyle 1 + c \sum_{l'=0}^{L-1} \abs{ \sin\left(\frac{\pi l'}{L}\right)(z_1 - 1) + \cos\left(\frac{\pi l'}{L}\right)(z_2 - 1) }^2} \,,
 \\
 \check{\tilde \psi}_l[\Bk] &= \underbrace{ - c \left[ c (-\Delta_{\text{d}L}) + 1 \right]^{-1} \partial_l^- \delta[\Bk] }
                                 _{\displaystyle = -c \partial_l^- \phi_\text{int}[\Bk] }
 ~\stackrel{\cF}{\longleftrightarrow}~
 \tilde \Psi_l(\Bz^{-1}) = \frac{\displaystyle c \left[ \sin\left(\frac{\pi l}{L}\right) (z_1^{-1} - 1) + \cos\left(\frac{\pi l}{L}\right) (z_2^{-1} - 1) \right] }
                           {\displaystyle 1 + c \sum_{l'=0}^{L-1} \abs{ \sin\left(\frac{\pi l'}{L}\right)(z_1 - 1) + \cos\left(\frac{\pi l'}{L}\right)(z_2 - 1) }^2 } \,,
 \\
 \psi_l[\Bk] &= \partial^+_l \delta[\Bk] 
 ~\stackrel{\cF}{\longleftrightarrow}~
 \Psi_l(\Bz) = \sin\left(\frac{\pi l}{L}\right) (z_1 - 1) + \cos\left(\frac{\pi l}{L}\right) (z_2 - 1) \,.
\end{align*}
Note that these discrete frames are bounded and also satisfy the unity condition, since
\begin{align*}
 \Phi(\Bz) + \sum_{i=0}^{I-1} \sum_{l=0}^{L-1} \tilde \Psi_{il}(\Bz^{-1}) \Psi_{il}(\Bz) = 1 \,, \; 
 \Phi(e^{j 0}) = 1 \,, \Psi_{il}(e^{j 0}) = \tilde \Psi_{il}(e^{j 0}) = 0 \,.
\end{align*}

For the continuous case,
given a constant $a>0$, $\Bk \in \Omega$ and $u \in \ell_2(\Omega)$ whose discrete Fourier transform is $F(e^{j\Bome})$ or $F(\Bz)$,
the multiscale sampling theory at scale $I$ and direction $L$ is  
  \begin{align}  \label{label:multiScaleSamp:continuousprojection:prop}
  f[\Bk] &= (f \ast \phi)[\Bk] + \sum_{i=0}^{I-1} \sum_{l=0}^{L-1} ( f \ast \check{ \tilde{\psi} }_{il} \ast \psi_{il} )[\Bk] 
  ~\stackrel{\cF}{\longleftrightarrow}~
  F(e^{j\Bome}) = F(e^{j\Bome}) \widehat \phi(\Bome) + \sum_{i=0}^{I-1} \sum_{l=0}^{L-1} F(e^{j\Bome}) \widehat{\tilde{\psi}^*_{il}}(\Bome) \widehat{\psi_{il}} (\Bome) \,
 \end{align}
and their frames are defined in the Fourier domain with $\Bx \in \mathbb R^2$ (see Figure \ref{fig:MDCD_let:u-g-problems}(a) for their spectra) as 
 \begin{align*}
  \phi(\Bx) = I^{-1} \sum_{i=0}^{I-1} a^{-i} \phi_\text{int} (a^{-1} \Bx)
  &~\stackrel{\cF}{\longleftrightarrow}~ 
  \widehat{\phi}(\Bome) = I^{-1} \sum_{i=0}^{I-1} \widehat{\phi_\text{int}}(a^i \Bome) \,,
  \\
  \psi_{il}(\Bx) = I^{-\frac{1}{2}} a^{-i} \psi_l(a^{-i} \Bx)
  &~\stackrel{\cF}{\longleftrightarrow}~ 
  \widehat{\psi_{il}}(\Bome) = I^{-\frac{1}{2}} \widehat{\psi_l}(a^i\Bome) \,,
  \\
  \check{\tilde \psi}_{il}(\Bx) = I^{-\frac{1}{2}} a^{-i} \check{\tilde\psi}_l(a^{-i} \Bx) 
  &~\stackrel{\cF}{\longleftrightarrow}~ 
  \widehat{\tilde \psi^*_{il}}(\Bome) = I^{-\frac{1}{2}} \widehat{\tilde \psi_l^*}(a^i\Bome) \,,
 \end{align*} 
with the interpolant $\phi_\text{int}(\cdot)$ in the continuous setting and, 
for $l = 0, \ldots, L-1$, 
the directional dual/primal wavelet $\tilde \psi_l(\cdot)$ and  $\psi_l(\cdot)$, so 
 \begin{align*}
  \phi_\text{int}(\Bx) = \left[ c (-\Delta_{L}) + 1 \right]^{-1} \delta(\Bx)
  &~\stackrel{\cF}{\longleftrightarrow}~
  \widehat{\phi_\text{int}}(\Bome) = \frac{1}{\displaystyle 1 + c \sum_{l'=0}^{L-1} \left[ \cos(\frac{\pi l'}{L}) \omega_2 + \sin(\frac{\pi l'}{L}) \omega_1 \right]^2} \,,
  \\
  \check{\tilde \psi}_l(\Bx) = -c \partial_l \phi_\text{int}(\Bx)
  &~\stackrel{\cF}{\longleftrightarrow}~
  \widehat{\tilde \psi^*_l}(\Bome) = \frac{\displaystyle -c \left[ \cos\left(\frac{\pi l}{L}\right) j \omega_2 + \sin\left(\frac{\pi l}{L}\right) j \omega_1 \right] }
                            {\displaystyle 1 + c \sum_{l'=0}^{L-1} \left[ \cos\left(\frac{\pi l'}{L}\right) \omega_2 + \sin\left(\frac{\pi l'}{L}\right) \omega_1 \right]^2} \,,
  \\
  \psi_l(\Bx) = \partial_l \delta(\Bx)
  &~\stackrel{\cF}{\longleftrightarrow}~
  \widehat{\psi_l}(\Bome) = \cos\left(\frac{\pi l}{L}\right) j\omega_2 + \sin\left(\frac{\pi l}{L}\right) j\omega_1 \,.
 \end{align*} 
 Similar to its discrete version, these bounded frames in a continuous setting also satisfy the unity condition as
 \begin{align*}
  \widehat{\phi}(\Bome) + \sum_{i=0}^{I-1} \sum_{l=0}^{L-1} \widehat{\tilde \psi^*_{il}}(\Bome) \widehat{\psi_{il}}(\Bome) = 1 \,,
  \widehat \phi(0) = 1 \,, \widehat{\psi_{il}}(0) = \widehat{\tilde \psi_{il}}(0) = 0 \,.
 \end{align*} 
Also note that, unlike the continuous filter banks in Figure \ref{fig:MDCD_let:u-g-problems}(a), there are aliasing effects in the discrete version in Figure \ref{fig:MDCD_let:u-g-problems}(b) because of the exponential operator in the complex domain; i.e., $\Bz = e^{j \Bome}$.


\subsection{The ``$\Bv$-problem''}   \label{subsection:vproblem:multiscale}

The solution of the $\Bv$-problem (\ref{eq:problem:v}) is rewritten in a 
sampling theory form with two shrinkage operators (due to $\norm{\Bv}_{\ell_1}$ and $\norm{ \Bg_s }_{\ell_1}$ in the DMCD model (\ref{eq:minimization:SDMCDD:2})), the
Lagrange multipliers $(\boldsymbol{\lambda_5} \,, \boldsymbol{\lambda_6} \,, \boldsymbol{\lambda_7})$
and the blur kernel $\Bh$ as
\begin{align}  \label{eq:VCPyramid:problem:v:spatial:1}
 \Bv^{(\tau)} &= \Shrink \left( \Bt_\Bv^{(\tau)} \,,~ \frac{\mu_2 \alpha^{(\tau)}}{\beta_5 + \alpha^{(\tau)} \beta_7} \right) 
\end{align}
with $\Bt_\Bv^{(\tau)} = \left[ t_v^{(\tau)} [\Bk] \right]_{\Bk \in \Omega}$ and
\begin{align*}
 t_v^{(\tau)}[\Bk] &= \frac{ \beta_5 }{ \beta_5 + \alpha^{(\tau)} \beta_7 }
 \left( \left( \delta - \alpha^{(\tau)} \check{h} \ast h \right) \ast v^{(\tau-1)}  
 + \alpha^{(\tau)} \check{h} \ast \left( f - h \ast u^{(\tau)} - h \ast \rho^{(\tau-1)} - \epsilon^{(\tau-1)} + \frac{\lambda_{5}^{(\tau-1)}}{\beta_5} \right) \right) [\Bk]
 \\
 &+ \frac{ \beta_7 \alpha^{(\tau)} }{ \beta_5 + \alpha^{(\tau)} \beta_7 }
 \left( \sum_{s=0}^{S-1} \check{\tilde \theta}^{S}_s \ast g_s^{(\tau)} - \frac{\lambda_7^{(\tau-1)}}{\beta_7} \right) [\Bk].
\end{align*}
By choosing $\beta_7 = c_{67} \beta_6$, the solution of the $\Bg$-problem (\ref{eq:problem:g}) at iteration $\tau$ is 
\begin{align} \label{eq:FilterBanks:problem:g:spatial:1}
 g_s^{(\tau)}[\Bk] &= \left( \xi^{S, c_{67}}_s \ast \left[ \Shrink \left( g_s^{(\tau-1)} - \frac{\lambda_{6 s}^{(\tau-1)}}{\beta_6} \,,~ \frac{\mu_1}{\beta_6} \right) + \frac{\lambda_{6 s}^{(\tau-1)}}{\beta_6} \right] \right) [\Bk]
 + \left( \theta^{S, c_{67}}_s \ast \left[ \check{\tilde \theta}^{S}_s \ast g_s^{(\tau-1)} + \frac{\lambda_7^{(\tau-1)}}{\beta_7} \right] \right) [\Bk] \,.
\end{align}
Note that the sampling theory form for (\ref{eq:FilterBanks:problem:g:spatial:1}) is more obvious if we simplify the equation by removing the shrinkage operator and the Lagrange multipliers, so
\begin{align*} 
 g_s^{(\tau)}[\Bk] &= \left( \xi^{S, c_{67}}_s \ast g_s^{(\tau-1)} \right) [\Bk]
 + \Big( \theta^{S, c_{67}}_s \ast \check{\tilde \theta}^{S}_s \ast g_s^{(\tau-1)} \Big)[\Bk] \,.         
\end{align*}
Frames  
$\xi^{S, c_{67}}_s(\cdot) \,, \theta^{S, c_{67}}_s(\cdot)$ and $\tilde \theta^S_s(\cdot)$ are well defined in
(\ref{eq:FilterBanks:XiThetaTheta_tilde:1:1})-(\ref{eq:FilterBanks:XiThetaTheta_tilde:unitycondition:1}), see proposition 
\ref{prop:VCPyramid:problem:v:spatial:1} in Appendix A.

The multiscale sampling version of the $\vec{\Bg}$-problem is similar to 
the $\Bu$-problem; see proposition \ref{prop:DMCDDMultiscaleSamplingTheory:g-problem} in Appendix A. 
Given a discrete function (data) $\Bf \in X$, constant $a > 0$ and $\Bk \in \Omega$, 
the discrete multiscale sampling theory at scale $I$ and direction $L$ is
\begin{align}  \label{label:multiScaleSamp:discreteprojection:gproblem:prop}
 f[\Bk] = (f \ast \xi)[\Bk]
 + \sum_{i=0}^{I-1} \sum_{s=0}^{S-1} ( f \ast \check{\tilde \theta}_{si} \ast \theta_{si} )[\Bk]
 ~\stackrel{\cF}{\longleftrightarrow}~
 F(\Bz) = 
 F(\Bz) \Xi(\Bz) +  
 \sum_{i=0}^{I-1} \sum_{s=0}^{S-1} F(\Bz) \tilde\Theta_{si}(\Bz^{-1}) \Theta_{si}(\Bz)
\end{align}
and their frames are defined in the Fourier domain as 
\begin{align*}
 \xi[\Bk] \stackrel{\cF}{\longleftrightarrow}
 \Xi(\Bz) = \frac{1}{SI} \sum_{i=0}^{I-1} \sum_{s=0}^{S-1} \Xi_s(\Bz^{a^i}) \,,~
 \theta_{si}[\Bk] \stackrel{\cF}{\longleftrightarrow} 
 \Theta_{si}(\Bz) = \frac{1}{\sqrt{SI}} \Theta_s(\Bz^{a^i})    \text{ and }
 \tilde\theta_{si}[\Bk] \stackrel{\cF}{\longleftrightarrow}
 \tilde\Theta_{si}(\Bz) = \frac{1}{\sqrt{SI}} \tilde\Theta_s(\Bz^{a^i}) \,
\end{align*}
(see Figure \ref{fig:MDCD_let:u-g-problems}(d) for their spectra).
The $\Xi_s(\Bz) = \Xi_s^{S, c_{67}}(\Bz) \,, \Theta_s(\Bz) = \Theta_s^{S, c_{67}}(\Bz)$ and $\tilde\Theta_s(\Bz) = \tilde \Theta_s^S(\Bz)$ are defined in (\ref{eq:FilterBanks:XiThetaTheta_tilde:1:1})-(\ref{eq:FilterBanks:XiThetaTheta_tilde:3:1}).
These multiscale frames also satisfy the unity condition in the Fourier domain as
\begin{align*}
 \Xi(\Bz) + \sum_{i=0}^{I-1} \sum_{s=0}^{S-1} \tilde\Theta_{si}(\Bz^{-1}) \Theta_{si}(\Bz) = 1 \,,~
 \Xi(e^{j\mathbf 0}) = 1 
 ~~\text{and}~~
 \Theta_{si}(e^{j\mathbf 0}) = \tilde\Theta_{si}(e^{j\mathbf 0}) = 0 \,.
\end{align*}
Given a constant $a > 0 \,, \Bk \in \Omega$ and a discrete function $f[\Bk] \stackrel{\cF}{\longleftrightarrow} F(e^{j\Bome})$, the continuous multiscale sampling theory form at scale $I$ and direction $L$ is
\begin{align}  \label{label:multiScaleSamp:continuousprojection:gproblem:prop}
 f[\Bk] = (f \ast \xi)[\Bk] + \sum_{i=0}^{I-1} \sum_{s=0}^{S-1} (f \ast \check{\tilde \theta}_{si} \ast \theta_{si})[\Bk] 
 ~\stackrel{\cF}{\longleftrightarrow}~
 F(e^{j\Bome}) = F(e^{j\Bome}) \widehat \xi(\Bome) + \sum_{i=0}^{I-1} \sum_{s=0}^{S-1} F(e^{j\Bome}) \widehat{\tilde \theta_{si}^*} (\Bome) \widehat{\theta_{si}} (\Bome) 
\end{align}
with 
\begin{align*}
 \xi(\Bx) = \frac{1}{SI} \sum_{i=0}^{I-1} \sum_{s=0}^{S-1} a^{-i} \xi_s(a^{-i} \Bx)
 &~\stackrel{\cF}{\longleftrightarrow}~
 \widehat \xi(\Bome) = \frac{1}{SI} \sum_{i=0}^{I-1} \sum_{s=0}^{S-1} \widehat \xi_s(a^i \Bome)
 \\
 \theta_{si}(\Bx) = \frac{1}{\sqrt{SI}} a^{-i} \theta_s(a^{-i} \Bx)
 &~\stackrel{\cF}{\longleftrightarrow}~
 \widehat \theta_{si}(\Bome) = \frac{1}{\sqrt{SI}} \widehat \theta_s(a^i \Bome)
 \\
 \tilde \theta_{si}(\Bx) = \frac{1}{\sqrt{SI}} a^{-i} \tilde \theta_s(a^{-i} \Bx)
 &~\stackrel{\cF}{\longleftrightarrow}~
 \widehat{\tilde \theta}_{si}(\Bome) = \frac{1}{\sqrt{SI}} \widehat{\tilde \theta}_s(a^i \Bome) \,
\end{align*}
and frames 
\begin{align*} 
 \xi_s(\Bx) = \Big[ 1 - c \partial_s^2 \Big]^{-1} \delta(\Bx) 
 &~\stackrel{\cF}{\longleftrightarrow}~  
 \widehat{\xi_s}(\Bome) = \frac{ 1 }{ 1 + c \left[ \cos \left( \frac{\pi s}{S} \right) \omega_2 + \sin \left( \frac{\pi s}{S} \right) \omega_1 \right]^2 }  \,,
 \\ 
 \theta_s(\Bx) = -c \partial_s \xi_s(\Bx) 
 &~\stackrel{\cF}{\longleftrightarrow}~  
 \widehat{\theta_s}(\Bome) = 
 \frac{ -c \Big[ \cos(\frac{\pi s}{S}) j \omega_2 + \sin(\frac{\pi s}{S}) j \omega_1 \Big] }
        { 1 + c \left[ \cos \left( \frac{\pi s}{S} \right) \omega_2 + \sin \left( \frac{\pi s}{S} \right) \omega_1 \right]^2 } \,,
 \\ 
 \check{\tilde \theta}_s(\Bx) = \partial_s \delta(\Bx)
 &~\stackrel{\cF}{\longleftrightarrow}~
 \widehat{\tilde \theta_s^*}(\Bome) = \cos \left( \frac{\pi s}{S} \right) j \omega_2 + \sin \left( \frac{\pi s}{S} \right) j \omega_1 \,
\end{align*}
(see Figure \ref{fig:MDCD_let:u-g-problems}(d) for their spectra).
These bounded frames satisfy the unity condition in the Fourier 
domain since 
\begin{align*}
 \widehat \xi(\Bome) + \sum_{i=0}^{I-1} \sum_{s=0}^{S-1} \widehat{\tilde \theta_{si}^*} (\Bome) \widehat{\theta_{si}} (\Bome) = 1 \,, 
 \widehat \xi(0) = 1 \,, \; \widehat{\tilde \theta_{si}}(0) = \widehat{\theta_{si}}(0) = 0 \,.
\end{align*}

Figure \ref{fig:MDCD_let:u-g-problems}(c, d) illustrates these multiscale filter banks for 
continuous and discrete settings.


\subsection{The ``$\Brho$ and $\Beps$-problems''}

Multi-scale and multi-direction analysis in the curvelet domain is 
especially useful for demixing the noise and residuals.
We minimize (\ref{eq:minimization:SDMCDD:2}) in terms of the supremum 
norms of the curvlet coefficients corresponding to the indicator
functions for the residual structure and noise terms, $\mathscr G^*(\frac{\Brho}{\nu_\rho})$ 
and $\mathscr G^*(\frac{\Beps}{\nu_\epsilon})$, respectively.

There are two terms in every solution of the $\Brho$-problem 
(\ref{eq:problem:rho}) or the $\Beps$-problems (\ref{eq:problem:epsilon}), 
namely an updated remainder and its curvelet smoothing term 
$\CST \left[ \cdot \,, \cdot \right]$ determined by the soft-thresholding 
operator. 
We describe the case of the $\Beps$-problem, but the explanation is the 
same for the $\Brho$-problem.
A solution of (\ref{eq:problem:epsilon}) at iteration $\tau$ is 
\begin{align}  \label{eq:epsilonproblem:filterbanks}
 \Beps^{(\tau)} &= \tilde{\Beps}^{(\tau)} - \CST \Big[ \tilde{\Beps}^{(\tau)} \,,~ \nu_\epsilon \Big] 
\end{align}
where $\tilde{\Beps}^{(\tau)}$, the remainder at iteration $\tau$, can
be approximated as
\begin{align*}
 \tilde{\Beps}^{(\tau)} &= \Bf - \Bh \ast \Bu^{(\tau)} - \Bh \ast \Bv^{(\tau)} - \Bh \ast \Brho^{(\tau)} + \frac{\boldsymbol{\lambda}_{\boldsymbol 5}^{(\tau-1)}}{\beta_5} 
 \approx \Beps^{(\tau)} \,.
\end{align*}
In (\ref{eq:epsilonproblem:filterbanks}), we call 
$\tilde{\Beps}^{(\tau)} \approx \Beps^{(\tau)}$ the updated term and 
$\CST \Big[ \tilde{\Beps}^{(\tau)} \,,~ \nu_\epsilon \Big]$ the smoothing 
term found by curvelet soft-thresholding $\CST \left[ \cdot \,, \cdot \right]$ 
(up to a level $\nu_\epsilon$).
If, at iteration $\tau$, the updated term $\tilde{\Beps}^{(\tau)}$ 
still contains both signal and noise, then the noise in 
$\tilde{\Beps}^{(\tau)}$ is removed by the $\CST$-operator. 
Finally, by subtraction, $\Beps^{(\tau)}$ contains almost pure noise 
(up to a degree $\nu_\epsilon$). 
This ``smoothing and subtraction'' procedure results from the constraint 
$\norm{\cC \{ \Beps \}}_{\ell_\infty} \leq \nu_\epsilon$, which takes  
advantage of the sparsity assumption in the curvelet transform; i.e.,
the  multi-scale and multi-directional partition of the Fourier domain.


\section{Experimental Results and Comparison}

To quantify the improvement obtained by the proposed algorithm,
we perform our demixing algorithm upon two images:  
the ``Barbara'' image Figure \ref{fig:deblurdecomp:barbara:nonoise}(a),
and a fingerprint image used in the SAMSI program on forensic statistics.
We use Barbara image because it contains many key apects of 
image analysis, i.e., homogeneous areas and texture at different scales 
(on the scarf, trouser and tablecloth).
We compare our proposed algorithm to the TV-$\ell_2$ deblurring 
algorithm \cite{Getreuer2012} and to the blind deblurring by the 
Matlab function deconvblind.m.
The criterion for comparing peformance is the mean squared error (MSE) in
the pixel-wise difference between the true image $\Bf_0$ and a combination of
the signal components obtained from the demixing algorithms $\Bf_\text{re} = \Bu + \Bv + \Brho$:
\begin{align*}
 \text{MSE} = \frac{ \norm{ \Bf_0 - \Bf_\text{re} }^2_{\ell_2} }{ d_1 d_2 } \,.
\end{align*}
In order to evaluate the performance of the algorithms, besides 
simple visual comparison and the calculation of the
mean square error, 
we also use the eigenvalues of the estimated covariance matrix under the 
assumption that any structure which persists after demixing appears in 
the reconstruction error and the eigenvalues of the estimated 
covariance matrix enable one to determine which algorithm 
successfully extracted more signal.
In the same vein, we denote 
$\left\{ \Be_l \right\}_{l \geq 1}$ as a vectorized sample of $10 \times 10$ non-overlapping blocks of the error ($\Bf_0 - \Bf_\text{re}$), i.e. $\Be_l \in \mathbb R^{100}$.
The maximum eigenvalue of the estimated covariance matrix (MEC) for sample $\left\{ \Be_l \right\}_{l \geq 1}$
is defined as
\begin{align*}
 \text{MEC} = \max_{s = 1, \ldots, 100} \left( \text{eig} \left\{ \Sigma \right\} \right)_{[s]}
\end{align*}
with sample mean and sample covariance, respectively, as
\begin{align*}
 \bar{\Be} = \frac{1}{\# \left( \left\{ \Be_l \right\}_{l \geq 1} \right)} \sum_{l \geq 1} \Be_l 
 \quad \text{and} \quad
 \Sigma = \frac{1}{\# \left( \left\{ \Be_l \right\}_{l \geq 1} \right)} \sum_{l \geq 1} (\Be_l - \bar{\Be}) (\Be_l - \bar{\Be})^\text{T} \,.
\end{align*}
Figure \ref{fig:deblurdecomp:barbara:nonoise} illustrates the result of 
the proposed demixing DMCD model on the Barbara image. 
A blurred image $\Bf$ (b) (a convolution of original image $\Bf_0$ (a) and a known smoothing kernel $h[\Bk] \stackrel{\cF}{\longleftrightarrow} H(\Bz)$ in (c))
is simultaneously reconstructed and decomposed into the piecewise smooth image $\Bu$ (d), the sparse texture $\Bv$ (e) and its binarized version (j) and 
the residual structure $\Brho$ (f).
The convergence of the algorithm is measured by the relative error of 
texture $\Bv$ in log scale (k).
It shows that the proposed DMCD model can reconstruct a blurred image and simultaneously decompose it into different components, including sparse texture $\Bv$ and piecewise smooth $\Bu$ with sharp edges, while preserving contrast. 
So, the reconstructed image with $\Bf_\text{re} = \Bu + \Bv + \Brho$ 
in Figure \ref{fig:deblurdecomp:barbara:nonoise}(h) can preserve contrast and texture with small mean squared error in comparison with the original image Figure \ref{fig:deblurdecomp:barbara:nonoise}(a).

Figure \ref{fig:Barbara:compare:2-multidir} illustrates the benefit of directional mean curvature (DMC) over mean curvature (MC) $(L = 2)$.
Although a reconstructed image (a) with mean curvature $(L = 2)$ is good, texture still remains in the piecewise smooth component $\Bu$ (b), see Figure \ref{fig:deblurdecomp:barbara:nonoise}(d-g) for a comparison with DMC $(L = 10)$. 
This artifact is due to the large bandwidth of a lowpass $\Phi(\Bz)$ which covers texture information, see (m).
For the homogeneous areas, a comparison of a reconstructed image between $L = 2$ and $L = 10$ directions is depicted in (e-l).
We see that the reconstructed components $\Bu$ by DMC (h,l) are smoother 
in approximating the function $\Bf$ (f, j),
and MC produces the ``stair-case'' effect, see (u, k).
An explanation for this benefit of DMC is that increasing $L$ makes 
the bandwidth of the lowpass $\Phi(\Bz)$ smaller (see (m) and (o) for $L=2$ and $L=10$, respectively) while small wavelet coefficients are eliminated in 
different directions, see equation (\ref{eq:VCPyramid:problem:u:spatial:1}).  
This effect makes a cartoon $\Bu$ smoother and removes oscillating patterns, 
such as texture and noise.
The highpass $\Psi(\Bz)$ are depicted in (n) and (p) for $L=2$ and $L=10$, respectively.
The stair case effect is due to the assumption of sparse signal under 
the gradient operator, and the directional version of the total variation norm \cite{ThaiGottschlich2016DG3PD, ThaiGottschlich2016inpainting, 
ThaiMentch2016} is known to handle this limitation.
The proposed directional mean curvature norm benefits from the
advantages of both the high-order PDE approach and the ability of 
directional methods to enhance sparse signal while preserve sharp edges in 
the restored image. 

Figure \ref{fig:deblurdecomp:barbara:compare} compares our demixing method
to TV-$\ell_2$ deblurring and the blind deblurring by the Matlab function 
deconvblind.m.
We see that TV-$\ell_2$ (c, f) can recover very sharp edges, but it also eliminates texture. 
The blind deblurring can recover texture, but it also produces 
``ringing'' effects (i.e., the larger a kernel size is, the more artifacts there
are in the reconstructed image), see (b, e). 
We observed that a kernel of size 7 is the best choice in terms of 
minimizing ringing.
The Matlab function can directly estimate an unknown blur kernel, which
our method does not, but it cannot decompose an image into different 
components while deblurring, and its performance on deblurring still
has problematic ringing.

Besides the Barbara image, we also demix a fingerprint image which 
contains small scale objects (noise) together with fingerprint 
patterns (texture), see Figure \ref{fig:deblurdecomp:fingerprint:nonoise}.
DMC removes the texture component in the piecewise smooth component
$\Bu$ (d) while preserving sharp edges, and the texture component 
$\Bv$ (e, j) is sparse. 
Also, small scale structure is separated in $\Brho$. 
Finally, the reconstructed by DMCD (h) achieves good performance in 
terms of mean squared error and visualization.

Figure \ref{fig:deblurdecomp:fingerprint:noise} illustrates the 
performance of our method when signal is corrupted by noise. 
We add an i.i.d.\ Gaussian noise $\mathcal N(0 \,, \sigma^2)$ with 
$\sigma = 10$ to a blurred signal (b).
By choosing a threshold $\nu_\epsilon = 6.5$, the noise component can 
be separated by the DMCD model---see its QQ-plot (d). 
And a reconstructed image (i) still preserves texture.
Note that mathematically selecting an optimal threshold for this 
Gaussian noise is beyond this paper. 
We use the QQ-plot to evaluate this threshold instead.
The DMCD model for other textured images are depicted in Figure \ref{fig:deblurdecomp:ballistic}-\ref{fig:deblurdecomp:tiger} in Appendix B.

Figure \ref{fig:MDCD_let:fingerprint:u-g-problems} illustrates a multiscale 
decomposition of the fingerprint image by filter banks in an harmonic analysis 
obtained from the $\Bu$-problem and the $\vec{\Bg}$-problem 
(or the $\Bv$-problem) in subsection \ref{subsection:uproblem:multiscale} 
and \ref{subsection:vproblem:multiscale}, respectively.
The corresponding filter banks of 
Figure \ref{fig:MDCD_let:fingerprint:u-g-problems} are depicted in 
Figure \ref{fig:MDCD_let:u-g-problems}. 

\section{Conclusion}
\label{sec:conclusion}

We provide the DMCD method to demix a blurred image $\Bf$ (with a known blurred kernel $\Bh$) into four meaningful components: piecewise smooth $\Bu$, 
texture $\Bv$, fine scale residual structure $\Brho$, and noise $\Beps$, so 
$\Bf = \Bh \ast \left( \Bu + \Bv + \Brho \right) + \Beps$. 

A cornerstone of the DMCD analysis is the assumption that signal is sparse 
under some transformed domains.
Using novel norms as key ingredients, we address some transformed domains 
to enforce on these components:
\begin{itemize}
 \item The directional mean curvature (DMC) norm eliminates texture from a piecewise smooth $\Bu$ while keeping edges and preserving its contrast. This property is due to the multi-directional and high-order approach which enhances 
sparsity of the objects under DMC.
 
 \item The directional $\text{G}_S$-norm is applied to capture texture 
$\Bv$ and the $\ell_1$-norm $\norm{\Bv}_{\ell_1}$ obtains sparse coefficients 
which are mainly due to repeated pattern. 
 
 \item The fine scale residual structure $\Brho$ and noise $\Beps$ are 
measured in the $\ell_\infty$-norm of curvelet coefficients 
$\norm{\cC \{ \Beps \}}_{\ell_\infty}$. 
Since this $\ell_\infty$-norm takes the advantage of the multi-scale and 
multi-directional curvelet transform, oscillating components can be 
independent (e.g., white noise $\Beps$) or weakly correlated 
(fine scale residual structure $\Brho$).
\end{itemize}
We also apply our DMCD model to real images to find superior results 
compared to other state-of-the-art methods, as measured by mean squared error 
and the maximum eigenvalue of the estimated covariance matrix. 
Moreover, DMCD simultaneously solves the decomposition and deblurring 
problems. 
Finally, we uncover a link between functional analysis and multiscale 
sampling theory, e.g., between harmonic analysis and filter banks. 

Due to high-order PDE problem, following \cite{ZhuTaiChan2013ALM}, 
an augmented Lagrangian method is applied to split DMC into several 
$\ell_1$- and $\ell_2$-norms.  
The advantage of this splitting method is to approximate 
complicated norms, but it also introduces new parameters which are 
chosen beforehand to speed of the convergence of the algorithm.

\section*{Acknowledgements}
This material was based upon work partially supported by the National Science Foundation under Grant DMS-1127914 to the Statistical and Applied Mathematical Sciences Institute and department of Statistical Science at Duke university. 
Any opinions, findings, and conclusions or recommendations expressed in this material are those of the authors and do not necessarily reflect the views of the National Science Foundation.


\bibliographystyle{unsrt}


\setcounter{subfigure}{0}
\begin{figure}
\begin{center}

 \includegraphics[width=1\textwidth]{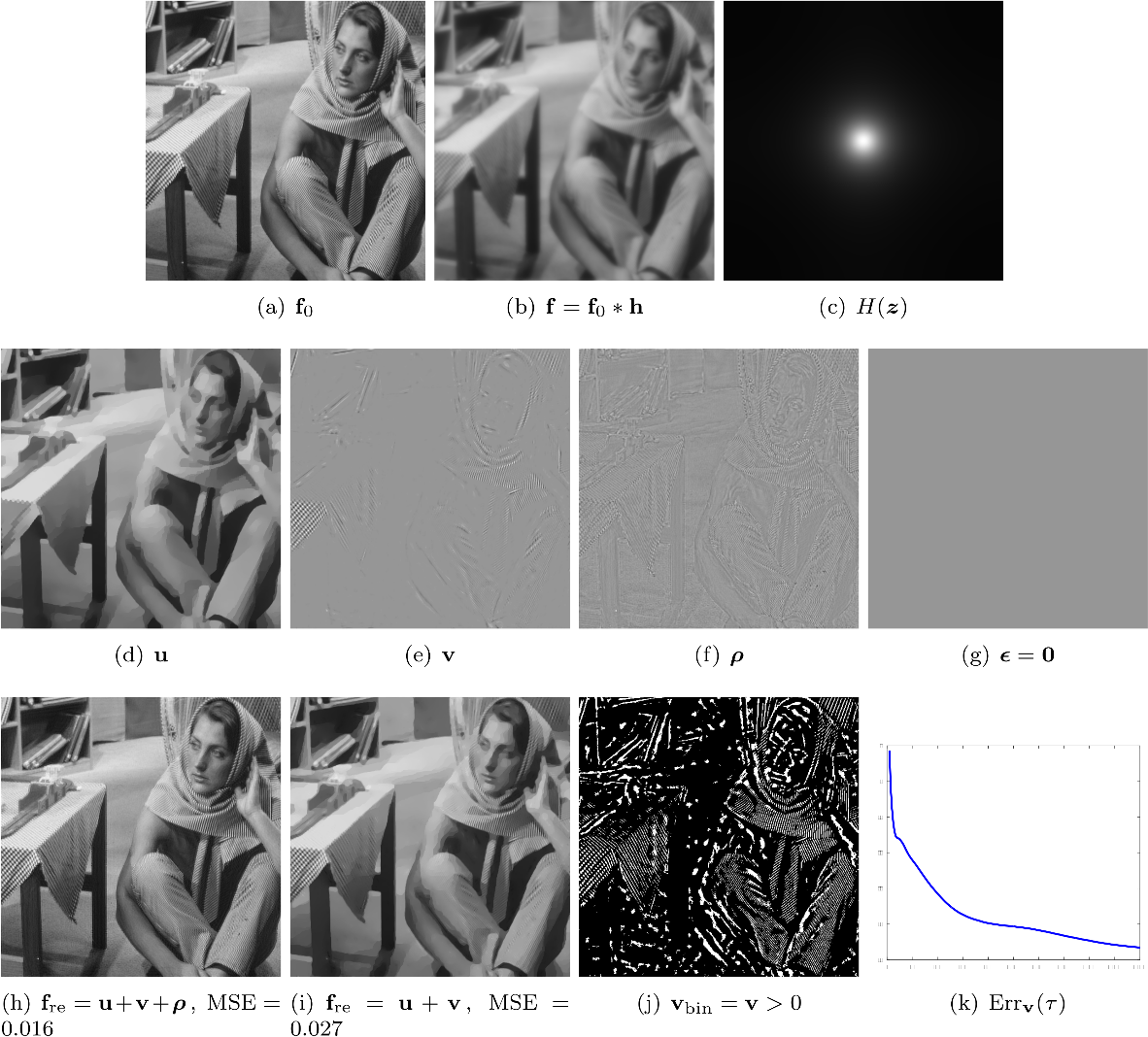}
 
 \caption{The original image $\Bf_0$ (a) is convolved by an operator $h[\Bk] \stackrel{\cF}{\longleftrightarrow} H(\Bz)$ in (c) to obtain a blurred version $\Bf$ (b).
          A reconstructed image (h) or (i) is obtained by applying the DMCD model to different reconstructed components of $\Bf$, 
          i.e. piecewise smooth image $\Bu$ (d), texture $\Bv$ (e, j) with $\text{sparsity} = \frac{ \# \{ v[\Bk] \neq 0 \,, \Bk \in \Omega \} }{ d_1 d_2 } 100 \% = 29.59\%$
          and residual $\Brho$ (f).
          By choosing parameters as $L_\text{blur} = 20 \,, \nu_\rho = 20 \,, \nu_\epsilon = 0 \,, \mu_2 = 4 \times 10^{10} \,, L = S = 10$,
          $\big[ \beta_i \big]_{i=1}^7 = \mu_1 = 10^{10} \text{ and } \alpha = 0.1$,
          the convergence of DMCD is illustrated by a relative error of $\Bv$ in a log scale (k),        
          see Figure \ref{fig:Barbara:compare:2-multidir} for a comparison of our directional mean curvature based approach with the original mean curvature one ($L = 2$). 
          \label{fig:deblurdecomp:barbara:nonoise}
         }
\end{center}
\end{figure}

\begin{figure}
 
  \centering
  \includegraphics[width=1\textwidth]{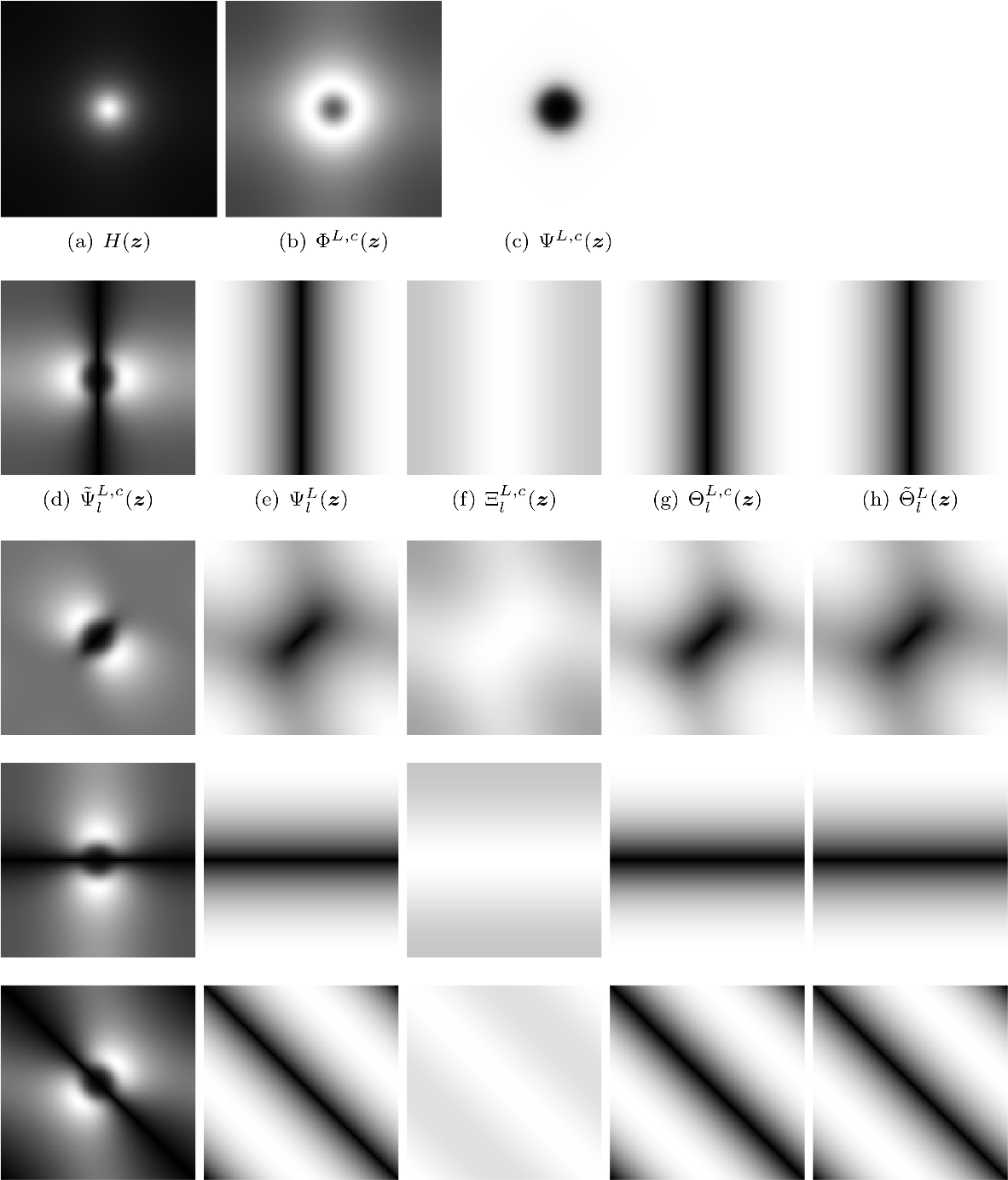}
 
 \caption{\label{fig:MDCD_let:FilterBanks:c0_1}
 This figure visualizes filter banks produced by the DMCD model (\ref{eq:VCPyramid:problem:frameElements:u:scalingfunc:1})-(\ref{eq:FilterBanks:XiThetaTheta_tilde:3:1}) with parameters $L_\text{blur} = 10 \,, L = S = 4 \,, c_{25} = c_{34} = c = 0.1$.
 A total wavelet function (c) is defined as $\displaystyle \Psi^{L,c}(\Bz) = \sum_{l=0}^{L-1} \tilde \Psi_l^{L,c}(\Bz^{-1}) \Psi_l^{L}(\Bz)$.
  }
\end{figure}

\setcounter{subfigure}{0}
\begin{figure}
\begin{center}

 \includegraphics[width=1\textwidth]{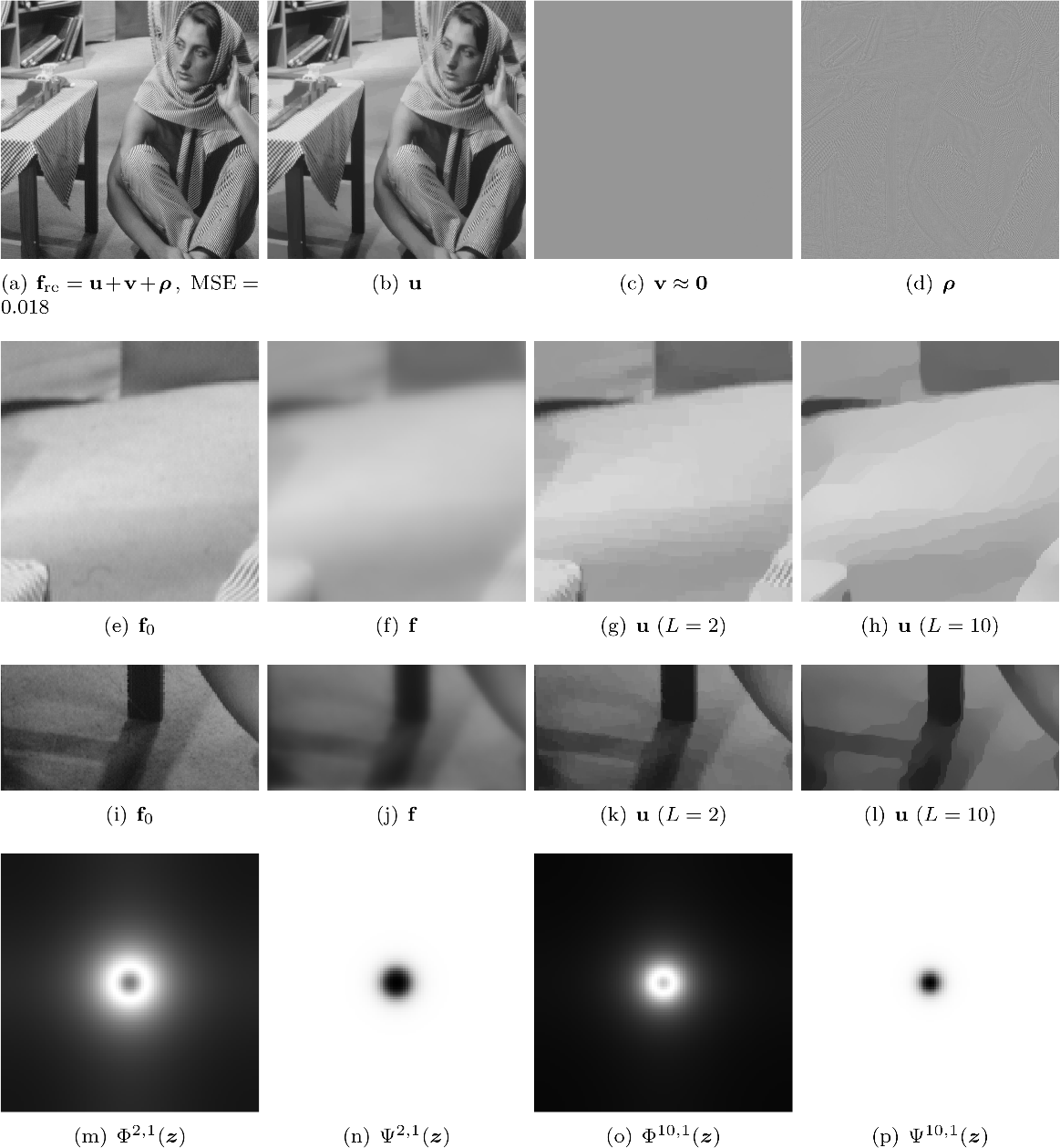} 
  
 \caption{\label{fig:Barbara:compare:2-multidir}
          This Figure illustrates a comparison of directional mean curvature and original mean curvature on 
          texture (a-d) and homogeneous areas (e-l).
          Visualization of the DMCD model with directions $L = S = 2$ and
          the other parameters are the same as in Figure \ref{fig:deblurdecomp:barbara:nonoise}.
          Although a reconstructed image $\Bf_\text{re}$ (a) has a good performance, 
          almost texture are still kept in $\Bu$. 
          This unsatisfied effect of a decomposition is due to a bandwidth of filters in the Fourier domain,
          see (m-p). 
          Plots (e-l) show the advantage of the multi-directional mean curvature
          for a reconstructed component $\Bu$ in terms of approximation and smoothness.         
          Increasing $L$ causes a shrinkage of bandwidth in a lowpass $\Phi(\Bz)$ (m) and (o)
          (for $L=2$ and $L=10$, respectively).
          This effect makes a cartoon $\Bu$ smoother and removes oscillating pattern, e.g. texture and noise.
          The highpass $\Psi(\Bz)$ are depicted in (n) and (p) for $L=2$ and $L=10$, respectively.
          The other parameters are the same as in Figure \ref{fig:deblurdecomp:barbara:nonoise}.
         }
\end{center}
\end{figure}

\setcounter{subfigure}{0}
\begin{figure}
\begin{center}
 
 \includegraphics[width=1\textwidth]{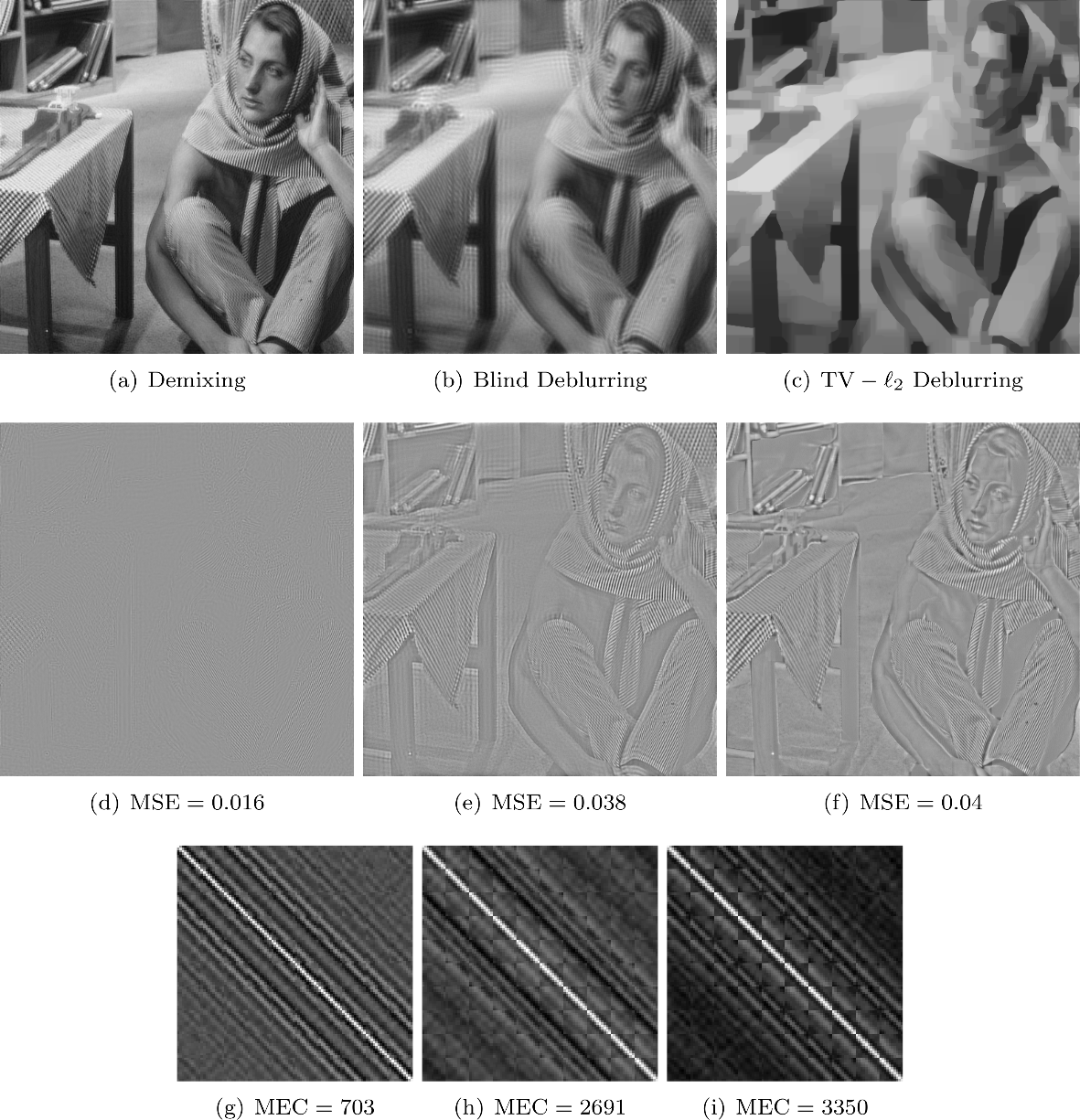}

 \caption{A comparison between the proposed DMCD (a, d) with Blind deblurring (b, e) and TV-$\ell_2$ (c, f) shows our better performance in terms of mean squared error (MSE)
 and the maximum eigenvalue of the estimated covariance matrix (MEC)
 and visualization.
 Reconstructed images are illustrated on the first row and their corresponding error images are on the second row.
 The value of all error images are added by 150 for visualization. 
 The third row shows sample covariance matrices of $10 \times 10$ non-overlapping blocks of the errors on the second row.
 We see that TV-$\ell_2$ can recover very sharp edges, but it also eliminates texture. 
 The blind deblurring (with Matlab function "deconvblind.m") can recover texture, but it also produces "ringing effects" (the larger a kernel size is, the more artifacts on its reconstructed image are). 
 We observed that a size of kernel 7 as its recommendation is the best choice in terms of "ringing effects" on its reconstructed image by visualization (although this Matlab function can estimate an unknown blur kernel itself which is more advantage than us, it cannot do demixing to decompose an image into different components while deblurring). 
          \label{fig:deblurdecomp:barbara:compare}
         }
\end{center}
\end{figure}

\setcounter{subfigure}{0}
\begin{figure}
\begin{center}
 
 \includegraphics[width=1\textwidth]{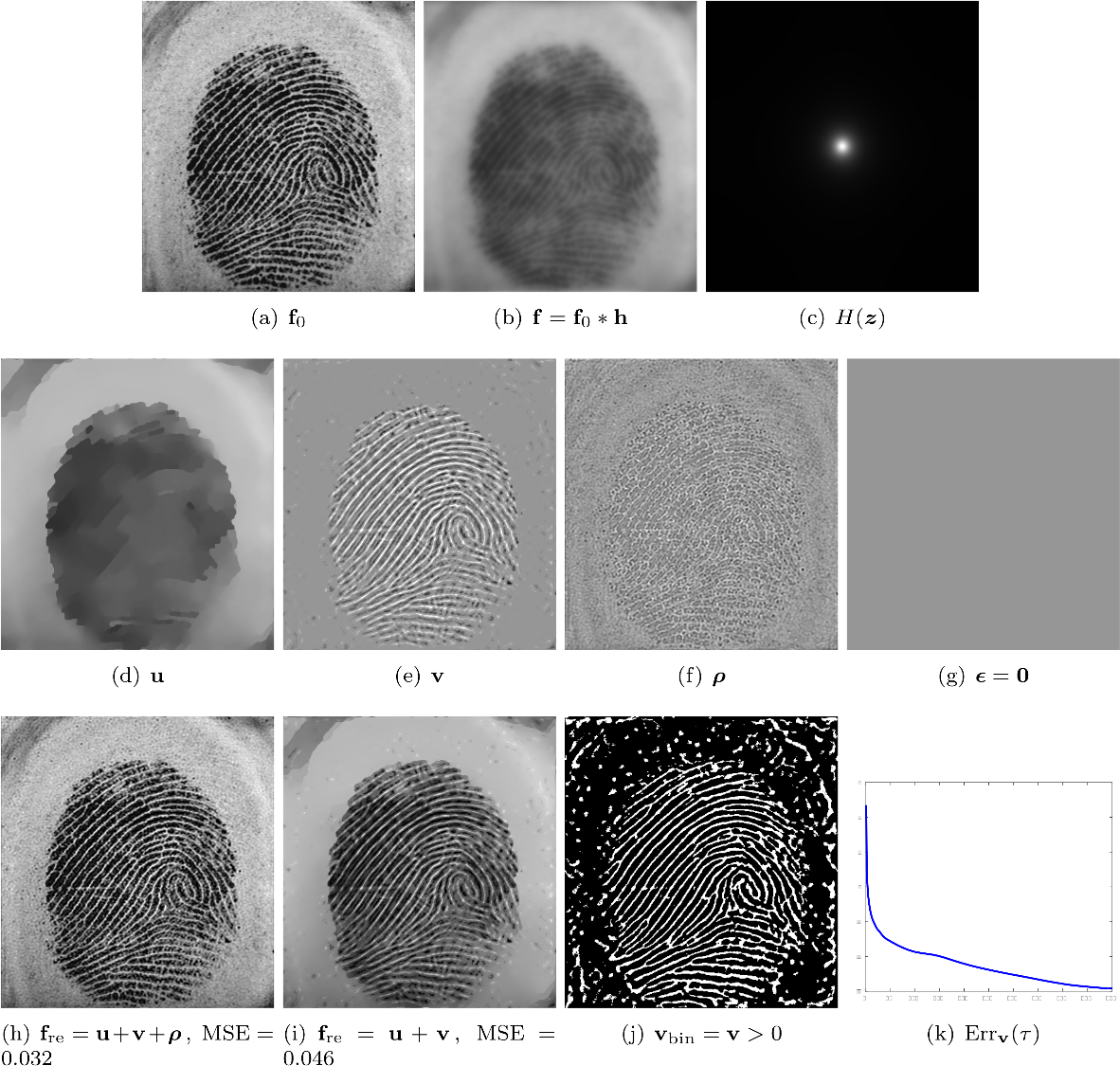} 
 
 \caption{Visualization of DMCD for fingerprint image with 
          $L_\text{blur} = 100 \,, \nu_\rho = 50 \,, \nu_\epsilon = 0 \,, \mu_2 = 3 \times 10^{10} \,, L = S = 10$  
          and the other parameters are the same as in Figure \ref{fig:deblurdecomp:barbara:nonoise}, 
          see Figure \ref{fig:deblurdecomp:fingerprint:noise} for its noisy version.
          \label{fig:deblurdecomp:fingerprint:nonoise}
         }
\end{center}
\end{figure}

\setcounter{subfigure}{0}
\begin{figure}
\begin{center}
 
 \includegraphics[width=1\textwidth]{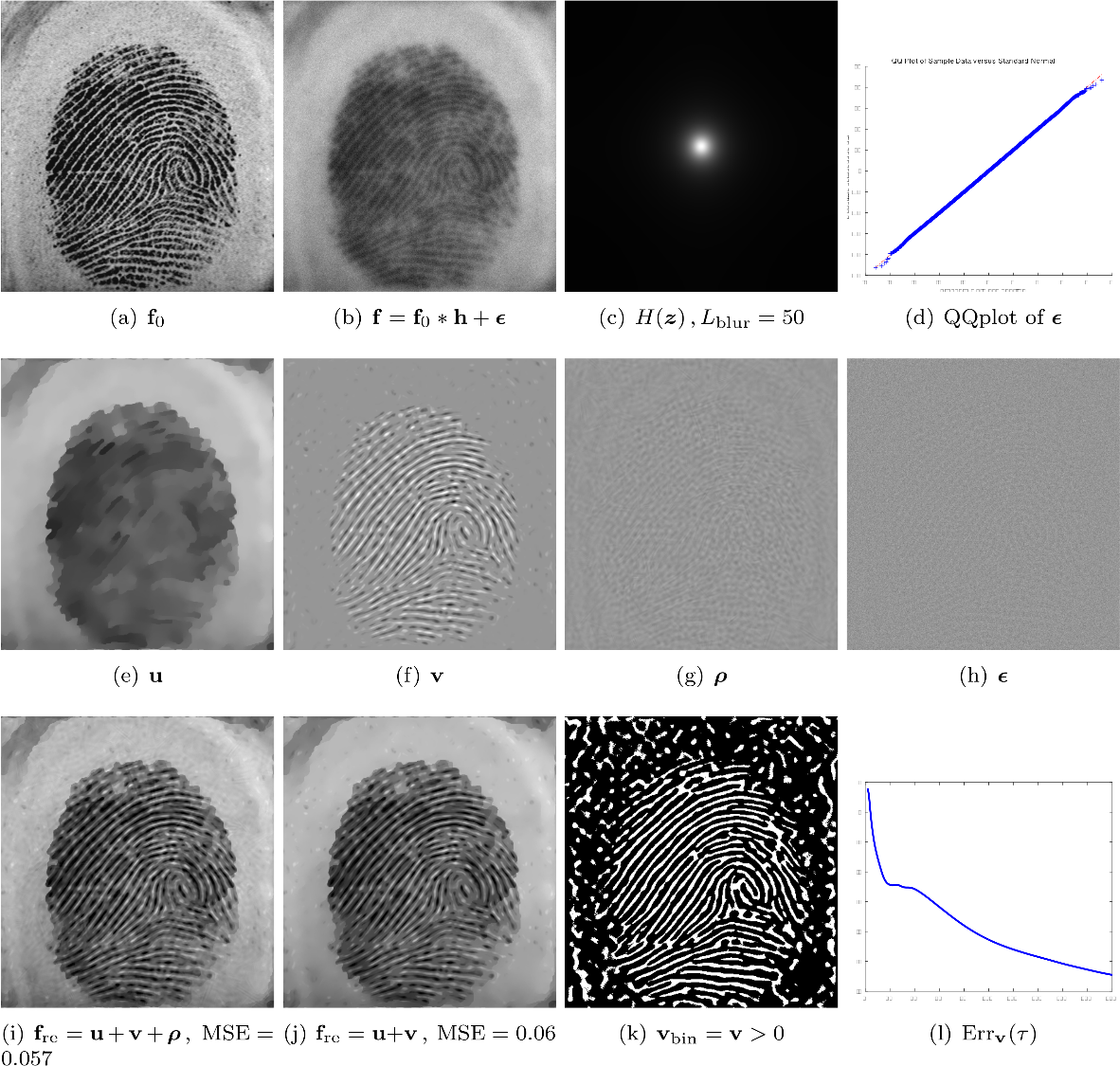} 
 
 \caption{Fingerprint image $\Bf_0$ (a) is corrupted by a blur operator (c) and an i.i.d. Gaussian noise $\cN(0, \sigma^2) \,, \sigma = 10$ to obtain (b).
          Choose parameters $L_\text{blur} = 50 \,, \nu_\rho = 15 \,, \nu_\epsilon = 6.5 \,, L = S = 10 \,, \mu_2 = 3 \times 10^{10}$ and
          the other parameters are as in Figure \ref{fig:deblurdecomp:barbara:nonoise}, 
          reconstructed images (i) and (j) have a good performance in terms of MSE and visualization.
          The QQplot of $\Beps$ (d) shows that an estimated $\Beps$ approaches the Gaussian assumption.
          The convergence of the algorithm (l) is computed in a log scale.
          \label{fig:deblurdecomp:fingerprint:noise}
         }
\end{center}
\end{figure}

\begin{figure*}
 \centering
  
 \includegraphics[width=1\textwidth]{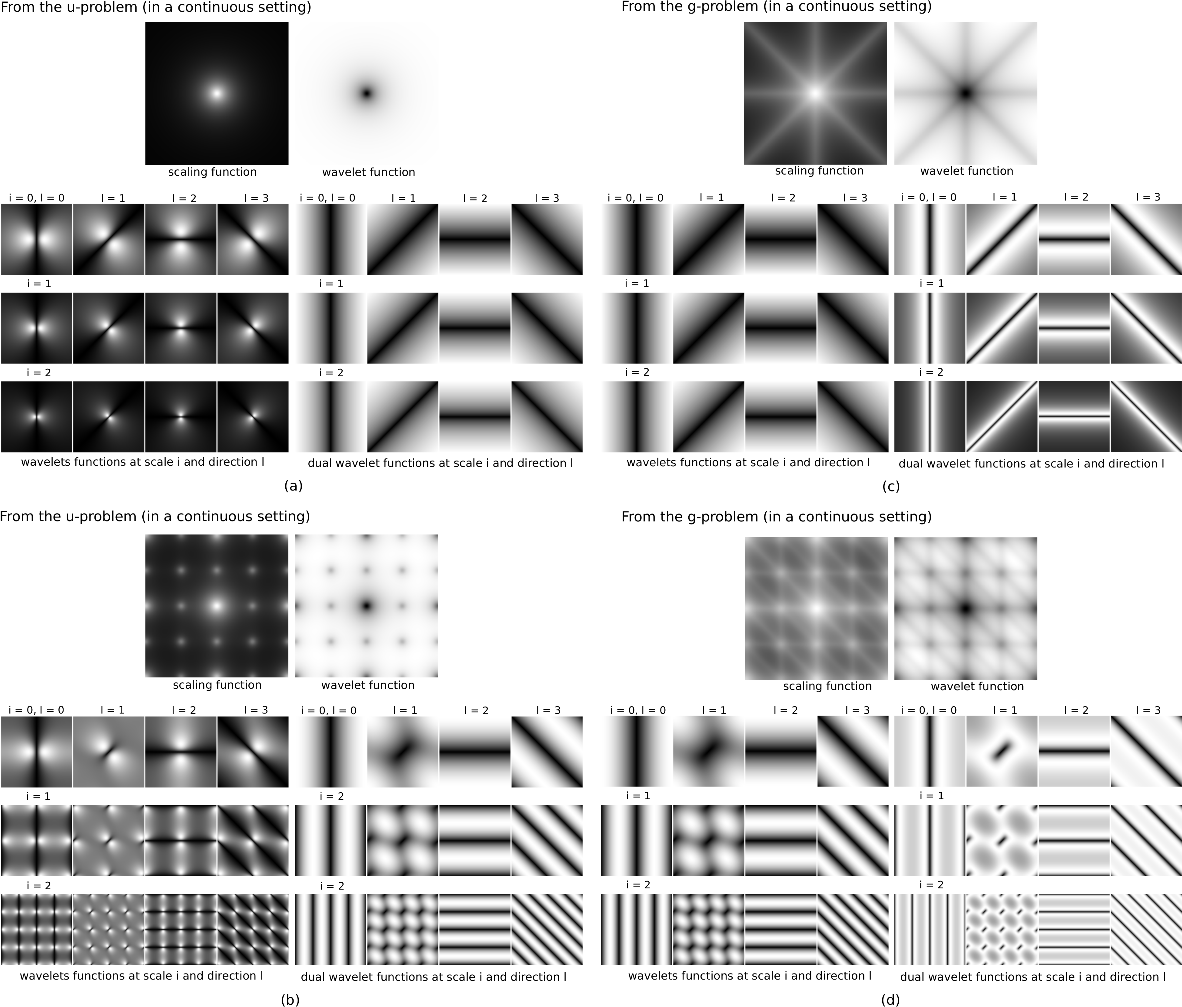}
 
 \caption{\label{fig:MDCD_let:u-g-problems} 
 Visualization of scaling and wavelet functions produced by the DMCD model at scale $I = 3$ and direction $L = 4$ in continuous and discrete settings from the "$\Bu$-problem" (a, b) and the "$\vec{\Bg}$-problem" (c, d) with $a = 2$ and $c = 1$. 
 The mathematical calculations for these frames are explained in proposition \ref{prop:DMCDDMultiscaleSamplingTheory:u-problem} 
 and \ref{prop:DMCDDMultiscaleSamplingTheory:g-problem} in Appendix A.
 Note that the aliasing effect in a discrete setting (b) and (d) is due to the exponential operator of the discrete Fourier transform: $\Bz = e^{j \Bome}$.
 Their scaling and wavelet coefficients of a fingerprint image are depicted in Figure \ref{fig:MDCD_let:fingerprint:u-g-problems}, respectively. 
 }
\end{figure*}

\begin{figure*}
 \centering
  
 \includegraphics[width=1\textwidth]{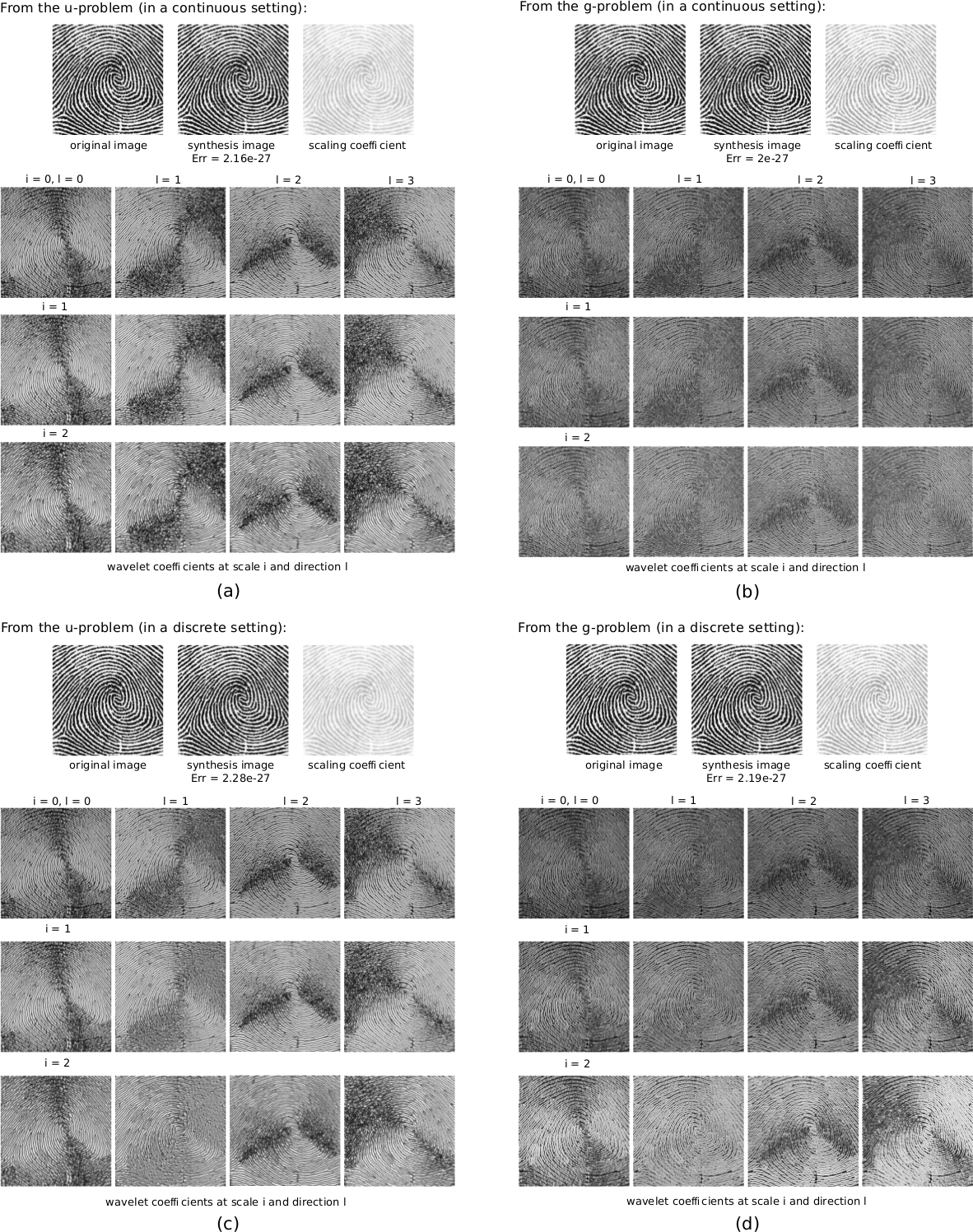}
 
 \caption{\label{fig:MDCD_let:fingerprint:u-g-problems} 
 Visualization of scaling, wavelet coefficients and its synthesis image by multiscale projection obtained from the "$\Bu$-problem" and the "$\vec{\Bg}$-problem" in continuous and discrete setting (a)-(d). For visualization, the scaling and wavelet coefficients are displayed in log scale.
 Their corresponding wavelet and scaling functions are illustrated in Figure \ref{fig:MDCD_let:u-g-problems} .
 }
\end{figure*}


\newpage 

\section*{Appendix A. Propositions and Proofs} \label{sec:appendixA:PropositionsProofs}


\begin{prop}  \label{prop:problem:t}
 The solution of the ``$\vec{\Bt}$-problem'' is
 \begin{align*}
  \vec{\Bt}^* &= \argmin_{ \vec{\Bt} \in X^{L+1}} \left\{ \mathscr F(\vec{\Bt}) = 
  \frac{\beta_3}{2} \norm{ \Bd - \text{div}^-_L \vec{\Bt} + \frac{\boldsymbol{\lambda_3}}{\beta_3} }^2_{\ell_2}
  + \frac{\beta_4}{2} \norm{ \vec{\Bt} - \vec{\By} + \frac{\vec{\boldsymbol{\lambda}}_{\boldsymbol 4}}{\beta_4} }^2_{\ell_2} 
  \right\}
  \\
  &=   
  \begin{cases}
   \displaystyle
   \RE \left[ \cF^{-1} \left\{ \frac{\cM_l(\Bz)}{\cN_l(\Bz)} \right\} \right] 
   \,,~ l = 0, \ldots, L-1 \,,
   \\
   \displaystyle
   \By_L - \frac{ \boldsymbol{\lambda}_{\boldsymbol 4L} }{ \beta_4 } \,.
  \end{cases}
 \end{align*} 
 with
 \begin{align*}
  \cN_l(\Bz) &= \beta_4 + \beta_3 \abs{ \cos\left(\frac{\pi l}{L}\right) (z_2 - 1) + \sin\left(\frac{\pi l}{L}\right) (z_1 - 1) }^2 \,,
  \\ 
  \cM_l(\Bz) &= \beta_4 \left[ Y_l(\Bz) - \frac{ \Lambda_{4l}(\Bz) }{ \beta_4 } \right]
  -\beta_3 \left[ \cos\left(\frac{\pi l}{L}\right) (z_2 - 1) + \sin\left(\frac{\pi l}{L}\right) (z_1 - 1) \right] \times
  \\ & \qquad ~
  \left[ D(\Bz) + \sum_{l'=[0, L-1] \backslash \{l\}} \left[ \cos\left(\frac{\pi l'}{L}\right) (z_2^{-1} - 1) + \sin\left(\frac{\pi l'}{L}\right) (z_1^{-1} - 1) \right] T_{l'}(\Bz) + \frac{ \Lambda_3(\Bz) }{ \beta_3 } \right]  \,.
 \end{align*}

\end{prop}

\begin{proof}
 
 The Euler-Lagrange equation is
 \begin{align*}
  0 &= \frac{\partial \mathscr F(\vec{\Bt})}{\partial \vec{\Bt}} = - \beta_3 \underbrace{ \frac{\partial \{ \text{div}^-_L \vec{\Bt} \}}{\partial \vec{\Bt}} }_{ ( \text{div}^-_L )^* \frac{\partial \vec{\Bt}}{\partial \vec{\Bt}} = -\nabla^+_L \delta }
  \left[ \Bd - \text{div}^-_L \vec{\Bt} + \frac{ \boldsymbol{\lambda_3} }{ \beta_3 } \right] 
  + \beta_4 \left[ \vec{\Bt} - \vec{\By} + \frac{ \vec{\boldsymbol{\lambda}_{\boldsymbol 4}} }{ \beta_4 } \right]
  \\
  &= \beta_3 \nabla^+_L \left[ \Bd - \text{div}^-_L \vec{\Bt} + \frac{ \boldsymbol{\lambda_3} }{ \beta_3 } \right] 
  + \beta_4 \left[ \vec{\Bt} - \vec{\By} + \frac{ \vec{\boldsymbol{\lambda}_{\boldsymbol 4}} }{ \beta_4 } \right]
 \end{align*}
 we have 
 $\nabla^+_L = \left[ \partial_l^+ \right]_{l=0}^{L-1} \,, \text{div}^-_L \vec{\Bt} = \sum_{l=0}^{L-1} \partial_l^- \Bt_l 
 \,, \vec{\Bt} = \left[ \Bt_l \right]_{l=0}^L \,, \vec{\By} = \left[ \vec{\By}_l \right]_{l=0}^L
 \,, \vec{\boldsymbol{\lambda_4}} = \left[ \boldsymbol{\lambda}_{\boldsymbol 4l} \right]_{l=0}^L
 $
 Thus,
 \begin{align*}
 &\begin{cases}
  \displaystyle
  \beta_3 \partial^+_l \left[ \Bd - \sum_{l'=0}^{L-1} \partial_{l'}^- \Bt_{l'} + \frac{ \boldsymbol{\lambda_3} }{ \beta_3 } \right] 
  + \beta_4 \left[ \Bt_l - \By_l + \frac{ \boldsymbol{\lambda}_{\boldsymbol 4a} }{ \beta_4 } \right] = 0
  \,,~ l = 0, \ldots, L-1.
  \\
  \displaystyle
  \beta_4 \left[ \Bt_L - \By_L + \frac{ \boldsymbol{\lambda}_{\boldsymbol 4L} }{ \beta_4 } \right] = 0
 \end{cases}
 \\ \Leftrightarrow~
 &\begin{cases}
  \displaystyle
  \left[ \beta_4 -\beta_3 \partial_l^+\partial_l^- \text{Id} \right] \Bt_l = 
  -\beta_3 \partial^+_l \left[ \Bd - \sum_{l'=[0, L-1] \backslash \{l\}} \partial_{l'}^- \Bt_{l'} + \frac{ \boldsymbol{\lambda_3} }{ \beta_3 } \right] 
  - \beta_4 \left[ - \By_l + \frac{ \boldsymbol{\lambda}_{\boldsymbol 4l} }{ \beta_4 } \right]
  \,,~ l = 0, \ldots, L-1.
  & (a)
  \\
  \displaystyle
  \Bt_L = \By_L - \frac{ \boldsymbol{\lambda}_{\boldsymbol 4L} }{ \beta_4 }
  & (b)
 \end{cases}
 \end{align*}
 The Fourier transform of equation (a) is
 \begin{align*}
  \Bt_l &= \left[ \beta_4 - \beta_3 \partial_l^+\partial_l^- \text{Id} \right]^{-1}
  \left[
  -\beta_3 \partial^+_l \left[ \Bd - \sum_{l'=[0, L-1] \backslash \{l\}} \partial_{l'}^- \Bt_{l'} + \frac{ \boldsymbol{\lambda_3} }{ \beta_3 } \right] 
  - \beta_4 \left[ - \By_l + \frac{ \boldsymbol{\lambda}_{\boldsymbol 4l} }{ \beta_4 } \right]
  \right]
  \\
  \stackrel{\cF}{\longleftrightarrow}~ 
  T_l(\Bz) &:= \frac{ \mathcal M_l(\Bz) }{ \mathcal N_l(\Bz) }
  = \left[ \beta_4 + \beta_3 \abs{ \cos\left(\frac{\pi l}{L}\right) (z_2 - 1) + \sin\left(\frac{\pi l}{L}\right) (z_1 - 1) }^2 \right]^{-1} \times
  \\ & \qquad \quad
  \Bigg[ - \beta_4 \left[ - Y_l(\Bz) + \frac{ \Lambda_{4l}(\Bz) }{ \beta_4 } \right]
  -\beta_3 \left[ \cos\left(\frac{\pi l}{L}\right) (z_2 - 1) + \sin\left(\frac{\pi l}{L}\right) (z_1 - 1) \right] \times
  \\ & \qquad \quad
  \left[ D(\Bz) + \sum_{l'=[0, L-1] \backslash \{l\}} \left[ \cos\left(\frac{\pi l'}{L}\right) (z_2^{-1} - 1) + \sin\left(\frac{\pi l'}{L}\right) (z_1^{-1} - 1) \right] T_{l'}(\Bz) + \frac{ \Lambda_3(\Bz) }{ \beta_3 } \right] 
  \Bigg] \,.
 \end{align*}
 
\end{proof}


\begin{prop}
 The solution of the ``$\vec{\Br} = \big[ \Br_l \big]_{l=0}^{L-1}$-problem'' is
 \begin{align*}
  \vec{\Br}^* &= \argmin_{\vec{\Br} \in X^{L+1}} \left\{ \mathscr F(\vec{\Br}) = 
  \left\langle \boldsymbol{\lambda_1} + \beta_1 \,, \abs{\vec{\Br}} - \langle \vec{\By} \,, \vec{\Br} \rangle_X \right\rangle_{\ell_2}
  + \frac{\beta_2}{2} \sum_{l=0}^{L-1} \norm{ \Br_l - \partial_l^+ \Bu + \frac{\boldsymbol{\lambda}_{\boldsymbol 2 l}}{\beta_2} }^2_{\ell_2}
  + \frac{\beta_2}{2} \norm{ \Br_L - \boldsymbol 1 + \frac{\boldsymbol{\lambda}_{\boldsymbol 2 L}}{\beta_2} }^2_{\ell_2}
  \right\}
  \\
  &= 
  \begin{cases}
   \displaystyle \Shrink \left( \partial_l^+ \Bu - \frac{ \boldsymbol{\lambda}_{\boldsymbol 2 l} }{ \beta_2 } + \frac{ \boldsymbol{\lambda_1} + \beta_1 }{ \beta_2 } \cdot^\times \By_l
   \,,~ \frac{ \boldsymbol{\lambda_1} + \beta_1}{\beta_2} \right)
   \,,~ & l = 0, \ldots, L-1 \,,
   \\
   \displaystyle \Shrink \left( \mathbf{1} - \frac{ \boldsymbol{\lambda}_{\boldsymbol 2 L} }{ \beta_2 } + \frac{ \boldsymbol{\lambda_1} + \beta_1 }{ \beta_2 } \cdot^\times \By_L
   \,,~ \frac{ \boldsymbol{\lambda_1} + \beta_1}{\beta_2}
   \right)
   \,,~ & l = L \,.
  \end{cases}  
 \end{align*} 
\end{prop}

\begin{proof}

 \begin{rem}
  Given $\Bf \,, \Bd \,, \Bu \in X$ and $\vec{\Bt} = \left[ \Bt_l \right]^L_{l=0} \in X^{L+1}$, we have  
 \begin{align*}
  \norm{\vec{\Bt}}^2_{\ell_2} = \sum_{l=0}^L \norm{\Bt_l}^2_{\ell_2} \,,~~
  \left \langle \Bf \,, \sum_{l=0}^L \Bt_l \right \rangle_{\ell_2} = \sum_{l=0}^L \left \langle \Bf \,, \Bt_l \right \rangle_{\ell_2}~~ \text{and} ~
  \left \langle \Bf \,, \Bd \cdot^\times \Bu \right \rangle_{\ell_2} = \left \langle \Bd \cdot^\times \Bf \,, \Bu \right \rangle_{\ell_2} \,.
 \end{align*}
   
  {\bfseries The proofs for this remark are}
  \begin{align*}
   \norm{ \vec{\Bt} }^2_{\ell_2} &= \sum_{l=0}^L \underbrace{ \sum_{\Bk \in \Omega} t^2_l[\Bk] }_{ = \norm{ \Bt_l }^2_{\ell_2} } \,,\qquad
  \left \langle \Bf \,, \sum_{l=0}^L \Bt_l \right \rangle_{\ell_2} 
  = \sum_{\Bk \in \Omega} f[\Bk] \left[ \sum_{l=0}^L \Bt_l[\Bk] \right]
  = \sum_{l=0}^L \underbrace{ \sum_{\Bk \in \Omega} f[\Bk] t_l[\Bk] }_{ = \left \langle \Bf \,, \Bt_l \right \rangle_{\ell_2} }   
  \quad \text{and}
  \\
  \left \langle \Bf \,, \Bd \cdot^\times \Bu \right \rangle_{\ell_2} &= \sum_{\Bk \in \Omega} f[\Bk] \left( \Bd \cdot^\times \Bu \right)[\Bk]
  = \sum_{\Bk \in \Omega} \bigg( \underbrace{ f[\Bk] \Bd[\Bk] }_{= (\Bf \cdot^\times \Bd)[\Bk]} \bigg) \Bu[\Bk] \,.
 \end{align*}
  
 \end{rem}


 Let $\vec{A} = \left[ \nabla_L^+ \Bu \,, 1 \right] \in X^L$, the objective function is rewritten as
 \begin{align*}
  \mathscr F(\vec{\Br}) &= 
  \left\langle \boldsymbol{\lambda_1} + \beta_1 \,, \abs{\vec{\Br}} \right\rangle_{\ell_2}
  + \frac{\beta_2}{2} \sum_{l=0}^{L} \norm{ \Br_l - A_l + \frac{\boldsymbol{\lambda}_{\boldsymbol 2 l}}{\beta_2} }^2_{\ell_2}
  - \left\langle \boldsymbol{\lambda_1} + \beta_1 \,, \sum_{l=0}^L \By_l \cdot^\times \Br_l \right\rangle_{\ell_2}  
  \\
  &= \left\langle \boldsymbol{\lambda_1} + \beta_1 \,, \abs{\vec{\Br}} \right\rangle_{\ell_2}
  + \frac{\beta_2}{2} \sum_{l=0}^L \left[ \norm{ \Br_l - A_l + \frac{\boldsymbol{\lambda}_{\boldsymbol 2 l}}{\beta_2} }^2_{\ell_2}
  - 2 \left\langle \frac{ \boldsymbol{\lambda_1} + \beta_1 }{\beta_2} \cdot^\times \By_l \,, \Br_l \right\rangle_{\ell_2} \right] 
  \\
  &= \left\langle \boldsymbol{\lambda_1} + \beta_1 \,, \abs{\vec{\Br}} \right\rangle_{\ell_2}
  + \frac{\beta_2}{2} \sum_{l=0}^L \bigg[ 
  \norm{ \Br_l }^2_{\ell_2} + 2 \left\langle \Br_l \,, - A_l + \frac{\boldsymbol{\lambda}_{\boldsymbol 2 l}}{\beta_2} - \frac{ \boldsymbol{\lambda_1} + \beta_1 }{\beta_2} \cdot^\times \By_l \right\rangle_{\ell_2}
  + \norm{ - A_l + \frac{\boldsymbol{\lambda}_{\boldsymbol 2 l}}{\beta_2} - \frac{ \boldsymbol{\lambda_1} + \beta_1 }{\beta_2} \cdot^\times \By_l }^2_{\ell_2}
  \\
  &+ \norm{ - A_l + \frac{\boldsymbol{\lambda}_{\boldsymbol 2 l}}{\beta_2} }^2_{\ell_2}
  - \norm{ - A_l + \frac{\boldsymbol{\lambda}_{\boldsymbol 2 l}}{\beta_2} - \frac{ \boldsymbol{\lambda_1} + \beta_1 }{\beta_2} \cdot^\times \By_l }^2_{\ell_2}
\bigg]   
\\
  &= \left\langle \boldsymbol{\lambda_1} + \beta_1 \,, \abs{\vec{\Br}} \right\rangle_{\ell_2}
  + 
  \frac{\beta_2}{2} \norm{ \vec \Br - \vec A + \frac{\vec{\boldsymbol{\lambda}}_{\boldsymbol 2}}{\beta_2} - \frac{ \boldsymbol{\lambda_1} + \beta_1 }{\beta_2} \cdot^\times \vec \By }^2_{\ell_2} 
  + 
  \underbrace{
  \frac{\beta_2}{2} \left[ 
  \norm{ - \vec A + \frac{\vec{\boldsymbol{\lambda}}_{\boldsymbol 2}}{\beta_2} }^2_{\ell_2}
  - \norm{ - \vec A + \frac{\vec{\boldsymbol{\lambda}}_{\boldsymbol 2 l}}{\beta_2} -  \frac{ \boldsymbol{\lambda_1} + \beta_1 }{\beta_2} \cdot^\times \vec\By }^2_{\ell_2}
\right]   
}_{ =~ \text{constant}}
 \end{align*}
 Under a minimization, a constant term can be removed and the "$\vec{\Br}$-problem" is rewritten as
 \begin{align*}
  \min_{\vec{\Br} \in X^{L+1}} \left\{ \mathscr F(\vec{\Br}) =
  \sum_{\Bk \in \Omega} \left[
  (\lambda_1[\Bk] + \beta_1) \abs{\vec{r}[\Bk]}
  + \frac{\beta_2}{2} \abs{ \vec{r}[\Bk] - \vec{A}[\Bk] + \frac{ \vec{\lambda_2}[\Bk] }{ \beta_2 } - \frac{ \lambda_1[\Bk] + \beta_1 }{ \beta_2 } \vec{y}[\Bk] }^2
  \right]
  \right\}
 \end{align*}
 Due to separability, consider at $\Bk = \Bm \in \Omega$, we have
 \begin{align*}
  \vec{r}^*[\Bm] 
  &= \argmin_{\vec{r}[\Bm] \in \mathbb R^{L+1}} \left\{ \mathscr F(\vec{r}[\Bm]) =
  \abs{\vec{r}[\Bm]}
  + \frac{1}{2} \frac{\beta_2}{\lambda_1[\Bm] + \beta_1} \abs{ \vec{r}[\Bm] - \vec{A}[\Bm] + \frac{ \vec{\lambda}_2[\Bm] }{ \beta_2 } - \frac{ \lambda_1[\Bm] + \beta_1 }{ \beta_2 } \vec{y}[\Bm] }^2
  \right\}
  \\
  &= \Shrink \left( \vec{A}[\Bm] - \frac{ \vec{\lambda}_2[\Bm] }{ \beta_2 } + \frac{ \lambda_1[\Bm] + \beta_1 }{ \beta_2 } \vec{y}[\Bm]
  \,,~ \frac{\lambda_1[\Bm] + \beta_1}{\beta_2} \right) \,.
 \end{align*}
 We have 
 $$
 \vec{\Br} = \left[ \Br_l \right]^L_{l=0} \in X^{L+1} \,, \vec{A} = \left[ \left[ \partial^+_l \Bu \right]_{l=0}^{L-1} \,, 1 \right] \in X^{L+1}
 \,, \vec{\By} = \left[ \By_l \right]_{l=0}^L \in X^{L+1} \,, \boldsymbol{\lambda_1} \in X 
 \,, \vec{\boldsymbol{\lambda_2}} = \left[ \boldsymbol{\lambda}_{\boldsymbol 2 l} \right]_{l=0}^L \in X^{L+1}.
 $$
 The matrix form is
 \begin{align*}
  \vec{\Br}^* &= \Shrink \left( \vec{\BA} - \frac{ \vec{\boldsymbol{\lambda_2}} }{ \beta_2 } + \frac{ \boldsymbol{\lambda_1} + \beta_1 }{ \beta_2 } \cdot^\times \vec{\By}
  \,,~ \frac{ \boldsymbol{\lambda_1} + \beta_1}{\beta_2} \right)
  \\
  &=
  \begin{cases}
   \displaystyle \Shrink \left( \partial_l^+ \Bu - \frac{ \boldsymbol{\lambda}_{\boldsymbol 2 l} }{ \beta_2 } + \frac{ \boldsymbol{\lambda_1} + \beta_1 }{ \beta_2 } \cdot^\times \By_l
   \,,~ \frac{ \boldsymbol{\lambda_1} + \beta_1}{\beta_2} \right)
   \,,~ & l = 0, \ldots, L-1 \,,
   \\
   \displaystyle \Shrink \left( \mathbf{1} - \frac{ \boldsymbol{\lambda}_{\boldsymbol 2 L} }{ \beta_2 } + \frac{ \boldsymbol{\lambda_1} + \beta_1 }{ \beta_2 } \cdot^\times \By_L
   \,,~ \frac{ \boldsymbol{\lambda_1} + \beta_1}{\beta_2}
   \right)
   \,,~ & l = L \,.
  \end{cases}
 \end{align*}
 
\end{proof}


\begin{lem} \label{lem:L1-projection} 
 Given $\vec{\By'} \in X^{L+1}$ and $\mu > 0$, a solution of $\ell_1$-minimization (a primal problem) in a matrix form is
 \begin{align}  \label{eq:primal:vectorL1:y}
  \vec{\By}_\text{p}^* &= \argmin_{  \vec{\By}_\text{p} \in X^{L+1} } \left\{ \mathscr F(\vec{\By}_\text{p}) :=  
  \mu \norm{ \vec{\By}_\text{p} }_{\ell_1} + \frac{1}{2} \norm{ \vec{\By}_\text{p} - \vec{\By'} }^2_{\ell_2}
  \right\}
  \\ \label{eq:primal:vectorL1:y:solution}
  &= \frac{ \vec{\By'} }{ \abs{\vec{\By'}} } \cdot^\times \max \left( \abs{\vec{\By'}} - \mu \,,~ 0 \right)
  := \Shrink \left( \vec{\By'} \,, \mu \right) \,,    
 \end{align} 
 with $\abs{\vec{\By'}} = \sqrt{ \sum_{l=0}^L \left[ \By'_l \right]^{\cdot 2} }$. 
 Denote $\vec{\By}_\text{d}$ as a dual variable of $\vec{\By}_\text{p}$,
 the Legendre-Fenchel transform of $\mathscr J (\vec{\By}_\text{p}) = \mu \norm{ \vec{\By}_\text{p} }_{\ell_1}$ on a convex set is
 \begin{align} \label{eq:LegendreFenchel:vectorL1:y}
  \mathscr R^* \left( \frac{ \vec{\By}_\text{d} }{ \mu } \right) = \begin{cases} 0 \,, & \vec{\By}_\text{d} \in \mathscr R(\mu) \\ +\infty \,, & \text{else} \end{cases}
  \,,~~ 
  \mathscr R(\mu) &= \left\{ \vec{\By}_{\text{d}} \in X^{L+1} ~:~ 
  \norm{ \vec{\By}_{\text{d}} }_{\ell_\infty} := \max_{\Bk \in \Omega} \sqrt{ \sum_{l=0}^L y^2_{\text{d} l}[\Bk]} \leq \mu \right\} \,.
 \end{align}
 A dual problem of (\ref{eq:primal:vectorL1:y}) is 
 \begin{align} \label{eq:dual:vectorL1:y}
  \vec{\By}^*_\text{d} &= \argmin_{ \vec{\By}_\text{d} \in X^{L+1} } \left\{  
  \mathscr R^* \left( \frac{ \vec{\By}_\text{d} }{ \mu } \right) + \frac{1}{2} \norm{ \vec{\By}_\text{d} - \vec{\By'} }^2_{\ell_2}
  \right\}
  = \text{Proj}_{\mathscr R(\mu)} \left( \vec{\By'} \right)
  \\ \label{eq:dual:vectorL1:y:solutionL1Projection}
  &= \begin{cases} \vec{\By'} \,, & \abs{ \vec{\By'} } \leq \mu 
  \\ 
  \mu \frac{ \vec{\By'} }{ \abs{\vec{\By'}} } \,, & \text{else} \end{cases}
  = \frac{ \mu \vec{\By'} }{ \max \left( \mu \,,~ \abs{\vec{\By'}} \right) } 
  \qquad \text{(by $\ell_1$-projection)} 
  \\ \label{eq:dual:vectorL1:y:solutionPrimalDual}
  &= \vec{\By'} - \underbrace{ \Shrink \left( \vec{\By'} \,, \mu \right) }_{= \vec{\By}_\text{p}^*} 
  ~\quad \qquad \qquad \qquad \qquad \text{(by primal-dual relation)} \,.
 \end{align}
 
\end{lem}

\begin{proof}

 \noindent
 {\bfseries Proof for a solution of $\ell_1$-minimization (\ref{eq:primal:vectorL1:y}) with a vector-valued form:}

 \noindent
 To be self-contained, a proof of (\ref{eq:primal:vectorL1:y:solution}) is provided for vector-valued data from \cite{GoldsteinOsher2009} and \cite{WangYinZhang2007}. 
 An objective function in (\ref{eq:primal:vectorL1:y}) is rewritten as
 \begin{align*}
  \mathscr F(\vec{\By}_\text{p}) = \sum _{\Bk \in \Omega} \left[ \mu \sqrt{ \sum_{l=0}^L y^2_{\text{p} l}[\Bk] } 
  + \frac{1}{2} \sum_{l=0}^L \left( y_{\text{p} l}[\Bk] - y'_{l}[\Bk] \right)^2
  \right] \,.
 \end{align*} 
 The optimal condition of $\mathscr F(\vec{\By}_\text{p})$ at $\Bk' \in \Omega$ and $l' \in \left\{ 0, \ldots, L \right\}$ is
 \begin{align*}
  0 = \frac{ \partial \mathscr F(\vec{\By}_\text{p}) }{ \partial y_{\text{p} l'}[\Bk'] }
  = \mu \frac{ y_{ \text{p} l' }[\Bk'] }{ \abs{\vec{y}_{ \text{p}}[\Bk']} } + y_{\text{p} l'}[\Bk'] - y'_{l'}[\Bk']
  \,,~~ 
  \abs{\vec{y}_{ \text{p}}[\Bk']} = \sqrt{ \sum_{l=0}^L y^2_{\text{p} l}[\Bk'] } \in \mathbb R
 \end{align*}
 If $\vec{\By}_{ \text{p}} = 0$, then by sub-differential we choose $\vec{c}[\Bk'] = \frac{ \vec{y}_{ \text{p}}[\Bk'] }{ \abs{\vec{y}_{ \text{p}}[\Bk']} } \in \mathbb R^{L+1}$ such that 
 $
  \abs{\vec{c}[\Bk']} = \sqrt{ \sum_{l=0}^{L} c^2_l[\Bk'] } \leq 1 
 $
 and 
 $\vec{y'}[\Bk'] = \mu \vec{c}[\Bk']$  
 with
 $\abs{\vec{y'}[\Bk']} = \mu \abs{\vec{c}[\Bk']} \leq \mu$.

 \noindent
 If $\vec{\By}_{ \text{p}} \neq 0$, then we have
 \begin{align*}
  \vec{y'}[\Bk'] = \left[ \frac{ \mu }{ \abs{\vec{y}_{ \text{p}}[\Bk']} } + 1 \right] \vec{y}_{ \text{p}}[\Bk']
  ~\Leftrightarrow~
  \abs{\vec{y'}[\Bk']} = \left[ \frac{ \mu }{ \abs{\vec{y}_{ \text{p}}[\Bk']} } + 1 \right] \abs{\vec{y}_{ \text{p}}[\Bk']}
  = \mu + \abs{\vec{y}_{ \text{p}}[\Bk']} \,.
 \end{align*} 
 Thus, we have 
 \begin{align*}
  \vec{y}_{ \text{p}}[\Bk'] = \frac{ \vec{y'}[\Bk'] }{ \abs{\vec{y'}[\Bk']} } \left[ \abs{\vec{y'}[\Bk']} - \mu \right] \,.
 \end{align*}
 Finally, we have a solution of (\ref{eq:primal:vectorL1:y}) with $\Bk' \in \Omega$ as
 \begin{align*}
  \vec{y}^*_{ \text{p}}[\Bk'] = \begin{cases}
   0 \,, & \abs{\vec{y'}[\Bk']} \leq \mu   
   \\
   \frac{ \vec{y'}[\Bk'] }{ \abs{\vec{y'}[\Bk']} } \left[ \abs{\vec{y'}[\Bk']} - \mu \right] \,,& \text{else}
  \end{cases}
  = \frac{ \vec{y'}[\Bk'] }{ \abs{\vec{y'}[\Bk']} } \max \left( \abs{\vec{y'}[\Bk']} - \mu \,, 0 \right) \,,
 \end{align*}
 or its matrix form is
 \begin{align*}
  \vec{\By}^*_{ \text{p}} = \frac{ \vec{\By'} }{ \abs{\vec{\By'}} } \cdot^\times \max \left( \abs{\vec{\By'}} - \mu \,, 0 \right) 
  \text{  with  }
  \abs{\vec{\By'}} = \sqrt{ \sum_{l=0}^L \left[ \By'_l \right]^{\cdot 2} } \,.
 \end{align*}
 
 \noindent
 {\bfseries Proof for a Lengendre-Fenchel transform of $\mathscr J (\vec{\By}_\text{p}) = \mu \norm{ \vec{\By}_\text{p} }_{\ell_1}$ in (\ref{eq:LegendreFenchel:vectorL1:y}):}
 
 \noindent
 The epigraph of a convex function $\mathscr J (\vec{\By}_\text{p})$ is a convex set as
 \begin{align*}
  \text{epi}(\mathscr J) = \left\{ \left( \vec{\By}_\text{p} \,, \nu \right) \in X^{L+1} \times \mathbb R ~:~ \mathscr J(\vec{\By}_\text{p}) \leq \nu \,, \forall \, \vec{\By}_\text{p} \right\} \,.
 \end{align*}
 A hyperplane (defined by $(\vec{\By}_\text{d} \,, \nu) \in X^{L+1} \times \mathbb R$) lying below $\mathscr J(\cdot)$, i.e. 
 $\left\langle \vec{\By}_\text{d} \,, \vec{\By}_\text{p} \right\rangle_{\ell_2} - \nu \leq \mathscr J(\vec{\By}_\text{p}) ~(\forall \, \vec{\By}_\text{p} \in X^{L+1})$,
 results in the Legendre-Fenchel transform of $\mathscr J (\vec{\By}_\text{p})$ as a convex function
 \begin{align} \label{eq:LegendreFenchel:proof}
  \mathscr J^*(\vec{\By}_\text{d}) &= \sup_{\vec{\By}_\text{p} \in X^{L+1}}
  \left\{
  \left\langle \vec{\By}_\text{d} \,, \vec{\By}_\text{p} \right\rangle_{\ell_2} - \mathscr J(\vec{\By}_\text{p})
  \right\}  \notag
  \\
  &= \mu \sup_{\vec{\By}_\text{p} \in X^{L+1}} \left\{
  \mathscr H(\vec{\By}_\text{p}) := \left\langle \frac{ \vec{\By}_\text{d} }{\mu} \,, \vec{\By}_\text{p} \right\rangle_{\ell_2}
  - \norm{ \vec{\By}_\text{p} }_{\ell_1}
  \right\}  \notag
  \\
  &:= \mu \mathscr R^* \left( \frac{ \vec{\By}_\text{d} }{\mu} \right) ~\leq~ \nu \,.
 \end{align}
 The optimal condition of its objective function
 \begin{align*}
  \mathscr H(\vec{\By}_\text{p}) = 
  \sum_{\Bk \in \Omega} \left[
  \frac{1}{\mu}\sum_{l=0}^L y_{\text{d} l}[\Bk] y_{\text{p} l}[\Bk]
  - \sqrt{ \sum_{l=0}^L y_{\text{p} l}^2[\Bk] }
  \right]
 \end{align*}
 at $\Bk' \in \Omega$ and $l' \in \{ 0, \ldots, L \}$ is
 \begin{align*}
  0 &= \frac{ \partial \mathscr H(\vec{\By}_\text{p}) }{ \partial y_{\text{p} l'}[\Bk'] }
  = \frac{1}{\mu} y_{\text{d} l'}[\Bk'] - \frac{ y_{\text{p} l'}[\Bk'] }{ \abs{ \vec{y}_\text{p} [\Bk']} }
  \\ \Leftrightarrow~
  y_{\text{d} l'}[\Bk'] &= \mu \frac{ y_{\text{p} l'}[\Bk'] }{ \abs{ \vec{y}_\text{p} [\Bk']} }
  \,,~ \text{  with  } \abs{ \vec{y}_\text{p} [\Bk']} = \sqrt{\sum_{l=0}^L y^2_{\text{p} l} [\Bk']} \,.
 \end{align*}
 If $\vec{\By}_{\text{p}} = 0$, then by sub-differential we choose $\vec{c}[\Bk'] = \frac{ \vec{y}_{\text{p}}[\Bk'] }{ \abs{ \vec{y}_\text{p} [\Bk']} } \in \mathbb R^{L+1}$
 such that $\abs{\vec{c}[\Bk']} \leq 1$ 
 , so $\mathscr H(\vec{\By}_\text{p}) = 0$ and $\vec{y}_{\text{d}}[\Bk'] = \mu \vec{c}[\Bk']$ with
 $\abs{\vec{y}_{\text{d}}[\Bk']} = \mu \abs{\vec{c}[\Bk']} \leq \mu \,,~ \forall \Bk' \in \Omega$
 which is equivalent to
 \begin{align*}
 \norm{ \vec{\By}_{\text{d}} }_{\ell_\infty} &:= \max_{\Bk' \in \Omega} \abs{\vec{y}_{\text{d}}[\Bk']} \leq \mu
 \text{   and   }
 \abs{\vec{y}_{\text{d}}[\Bk']} = \sqrt{ \sum_{l=0}^L y^2_{\text{d} l}[\Bk']} \,.
 \end{align*}
 Thus, 
 \begin{align*}
  \mathscr R^* \left( \frac{ \vec{\By}_\text{d} }{\mu} \right) = 0 
  \text{   with   } \norm{ \vec{y}_{\text{d}} }_{\ell_\infty} \leq \mu \,.
 \end{align*}
 If $\vec{\By}_{\text{p}} \neq 0$, we have
  \begin{align*}
  \mathscr H(\vec{\By}_\text{p}) &= 
  \sum_{\Bk \in \Omega} \left[
  \frac{1}{\mu}\sum_{l=0}^L y_{\text{d} l}[\Bk] y_{\text{p} l}[\Bk]
  - \sqrt{ \sum_{l=0}^L y_{\text{p} l}^2[\Bk] }
  \right]
  \geq 
  \sum_{\Bk \in \Omega} \sum_{l=0}^L \left[
  \frac{1}{\mu} y_{\text{d} l}[\Bk] y_{\text{p} l}[\Bk]
  - \abs{ y_{\text{p} l}[\Bk] }
  \right]  
  \\
  &= \begin{cases} \displaystyle
   \sum_{\Bk \in \Omega} \sum_{l=0}^L \left[ \frac{1}{\mu} y_{\text{d} l}[\Bk] - 1 \right] y_{\text{p} l}[\Bk] 
   \,, & y_{\text{p} l}[\Bk] \geq 0 \,, \forall l \,, \Bk \,,
   \\ \displaystyle
   \sum_{\Bk \in \Omega} \sum_{l=0}^L \left[ \frac{1}{\mu} y_{\text{d} l}[\Bk] + 1 \right] y_{\text{p} l}[\Bk]
   \,, & y_{\text{p} l}[\Bk] < 0 \,, \forall l \,, \Bk \,.
  \end{cases}
 \end{align*}
 We observe that given $\abs{\vec{y}_{\text{d}}[\Bk]} > \mu$, 
 $\sup_{\vec{\By}_\text{p} \in X^{L+1}} \mathscr H(\vec{\By}_\text{p}) = +\infty$
 when $\vec{\By}_\text{p} \rightarrow \pm \infty$.
 Thus,
 \begin{align*}
  \mathscr R^* \left( \frac{ \vec{\By}_\text{d} }{\mu} \right) = +\infty 
  \text{   with   } \abs{\vec{y}_{\text{d}}[\Bk]} > \mu \,, \forall \Bk \in \Omega \,.
 \end{align*}
 So, we have the Legendre-Fenchel transform of $\mathscr J (\vec{\By}_\text{p}) = \mu \norm{ \vec{\By}_\text{p} }_{\ell_1}$ on a convex set is
 \begin{align*}
  \mathscr J^* (\vec{\By}_\text{d}) &= \mathscr R^* \left( \frac{ \vec{\By}_\text{d} }{\mu} \right) =
  \begin{cases}
   0 \,,~ & \vec{\By}_{\text{d}} \in \mathscr R(\mu) 
   \\ 
   +\infty \,,~ & \text{else} 
  \end{cases} \,,
  \\ 
  \mathscr R(\mu) &= \left\{ \vec{\By}_{\text{d}} \in X^{L+1} ~:~ 
  \norm{ \vec{\By}_{\text{d}} }_{\ell_\infty} := \max_{\Bk \in \Omega} \sqrt{ \sum_{l=0}^L y^2_{\text{d} l}[\Bk]} \leq \mu \right\} \,.
 \end{align*}
 
 \noindent
 {\bfseries Proof for a solution of dual problem of (\ref{eq:primal:vectorL1:y}), Eq. (\ref{eq:dual:vectorL1:y})-(\ref{eq:dual:vectorL1:y:solutionPrimalDual}):}
 
 \noindent
 We define epigraph of the Legendre-Fenchel transform $\mathscr J^* (\vec{\By}_\text{d})$ in (\ref{eq:LegendreFenchel:proof}) as a set of hyperplane lying below a convex function $\mathscr J(\vec{\By}_\text{p})$ as
 \begin{align*}
  \text{epi}(\mathscr J^*) = \left\{  
  \left( \vec{\By}_\text{d} \,, \nu \right) \in X^{L+1} \times \mathbb R ~:~ \mathscr J^*(\vec{\By}_\text{d}) \leq \nu \,, \forall \, \vec{\By}_\text{d}
  \right\} \,.
 \end{align*}
 By bi-conjugate of a convex function, we have
 \begin{align*}
  &\mathscr J^{**}(\vec{\By}_\text{p}) = \sup_{\vec{\By}_\text{d} \in X^{L+1}} \left\{
  \left\langle \vec{\By}_\text{d} \,, \vec{\By}_\text{p} \right\rangle_{\ell_2} - \mathscr J^*(\vec{\By}_\text{d})
  \right\} = \mathscr J(\vec{\By}_\text{p})
  \\
  \Leftrightarrow~
  & \mathscr J^*(\vec{\By}_\text{d}) + \mathscr J(\vec{\By}_\text{p}) \geq \left\langle \vec{\By}_\text{d} \,, \vec{\By}_\text{p} \right\rangle_{\ell_2}
  \,,~ \forall \left( \vec{\By}_\text{d} \,, \vec{\By}_\text{p} \right) \in X^{L+1} \times X^{L+1} \,.
 \end{align*}
 The equality happens, i.e. $\mathscr J^*(\vec{\By}_\text{d}) + \mathscr J(\vec{\By}_\text{p}) = \left\langle \vec{\By}_\text{d} \,, \vec{\By}_\text{p} \right\rangle_{\ell_2}$, if $\vec{\By}_\text{p}$ is the sub-differential of $\mathscr J^*(\vec{\By}_\text{d})$ or
 $\vec{\By}_\text{d}$ is the sub-differential of $\mathscr J(\vec{\By}_\text{p})$ as
 \begin{align} \label{eq:subdifferential}
  \frac{\partial \mathscr J(\vec{\By}_\text{p})}{\partial \vec{\By}_\text{p}} = \vec{\By}_\text{d}
  ~\Leftrightarrow~ 
  \frac{\partial \mathscr J^*(\vec{\By}_\text{d})}{\partial \vec{\By}_\text{d}} = \vec{\By}_\text{p} \,.
 \end{align}
 Equivalently, the objective function in (\ref{eq:primal:vectorL1:y}) is rewritten as
 \begin{align*}
  \mathscr F(\vec{\By}_\text{p}) =  
  \mathscr J (\vec{\By}_\text{p}) + \frac{1}{2} \norm{ \vec{\By}_\text{p} - \vec{\By'} }^2_{\ell_2}
  = \left\langle \vec{\By}_\text{d} \,, \vec{\By}_\text{p} \right\rangle_{\ell_2} - \mathscr J^*(\vec{\By}_\text{d})
  + \frac{1}{2} \norm{ \vec{\By}_\text{p} - \vec{\By'} }^2_{\ell_2} \,.
 \end{align*}

 \noindent
 Since Legendre-Fenchel transform of $\mathscr J (\vec{\By}_\text{p})$ is $\mathscr J^* (\vec{\By}_\text{d}) = \mathscr R^* \left( \frac{\vec{\By}_\text{d}}{\mu} \right)$, 
 we build a dual relation of a non-smooth minimization (\ref{eq:primal:vectorL1:y}) by considering the optimal condition of an objective function $\mathscr F(\cdot)$
  at $(\vec{\By}^*_\text{p} \,, \vec{\By}^*_\text{d})$ with the sub-differential \cite{AujolChambolle2005, EkelandTeman1999} as 
 \begin{align*}
  \vec{\By'} = \vec{\By}^*_\text{d}   + \vec{\By}^*_\text{p} 
  \text{  and  }
  0 \in \vec{\By}^*_\text{p} - \frac{ \partial \mathscr J^*(\vec{\By}_\text{d}) }{ \partial \vec{\By}_\text{d} } \bigg |_{\vec{\By}_\text{d} = \vec{\By}^*_\text{d} } 
 \end{align*}
 Thus, we have 
 \begin{align*}
  0 \in \frac{ \partial \mathscr J^*(\vec{\By}_\text{d}) }{ \partial \vec{\By}_\text{d} } \bigg |_{\vec{\By}_\text{d} = \vec{\By}^*_\text{d} } 
  + \vec{\By}^*_\text{d} - \vec{\By'} \,,
 \end{align*}
 where $\vec{\By}^*_\text{d}$ is a solution of a dual problem of (\ref{eq:primal:vectorL1:y}) by a primal/dual relation as
 \begin{align} \label{eq:dual:vectorL1:y:proof}
  \vec{\By}^*_\text{d} &= \argmin_{\vec{\By}_\text{d} \in X^{L+1}} \left\{
  \mathscr R^* \left( \frac{\vec{\By}_\text{d}}{\mu} \right)
  + \frac{1}{2} \norm{ \vec{\By}_\text{d} - \vec{\By'} }^2_{\ell_2} 
  \right\}
  \\
  &= \vec{\By'} - \Shrink \left( \vec{\By'} \,, \mu \right) \,.   \notag
 \end{align}
 Note that a convex set $\mathscr R(\mu)$ is equivalent to $\ell_1$-ball as
 \begin{align*}
  \mathscr R(\mu) &= \left\{ \vec{\By}_{\text{d}} \in X^{L+1} ~:~ 
  \norm{ \vec{\By}_{\text{d}} }_{\ell_\infty} := \max_{\Bk \in \Omega} \abs{ \vec{y}_\text{d}[\Bk] } \leq \mu 
  ~\Leftrightarrow~ 
  \abs{ \vec{y}_\text{d}[\Bk] } \leq \mu \,,~ \forall \Bk \in \Omega
  \right\} 
  \,.
 \end{align*} 
 Thus, a solution of a dual problem (\ref{eq:dual:vectorL1:y:proof}) is a projection of $\vec{\By'}$ onto 
 an $\ell_1$-ball $\mathscr R(\mu)$ as
 \begin{align*} 
  \vec{\By}^*_\text{d} &= \argmin_{\vec{\By}_\text{d} \in \mathscr R(\mu)} 
  \norm{ \vec{\By}_\text{d} - \vec{\By'} }^2_{\ell_2} 
  ~=~ \text{Proj}_{\mathscr R(\mu)} \left( \vec{\By'} \right)
  ~=~ \begin{cases} \vec{\By'} \,, & \abs{ \vec{\By'} } \leq \mu 
  \\ 
  \mu \frac{ \vec{\By'} }{ \abs{\vec{\By'}} } \,, & \text{else} \end{cases}
  ~=~ \frac{ \mu \vec{\By'} }{ \max \left( \mu \,,~ \abs{\vec{\By'}} \right) } \,.
 \end{align*} 
\end{proof}


\begin{prop}
 The solution of the ``$\vec{\By} = \left[ \By_l \right]_{l=0}^{L-1}$-problem'' is
 \begin{align*}
  \vec{\By}^* &= \argmin_{\vec{\By} \in X^{L+1}} \left\{ \mathscr F(\vec{\By}) = 
  \mathscr R^* (\vec{\By}) +
  \big\langle \boldsymbol{\lambda_1} + \beta_1 \,, \abs{\vec{\Br}} - \langle \vec{\By} \,, \vec{\Br} \rangle_X \big\rangle_{\ell_2}
  + \frac{\beta_4}{2} \sum_{l=0}^L \norm{ \Bt_l - \By_l + \frac{\boldsymbol{\lambda}_{\boldsymbol 4l}}{\beta_4} }^2_{\ell_2}
  \right\}
  \\
  &= 
  \begin{cases}
   \By'_l \,, & \abs{\vec{\By'}} \leq 1  \\
   \frac{ \By'_l }{ \abs{\vec{\By'}} } \,, & \abs{\vec{\By'}} > 1
  \end{cases}
  \,,~~ l = 0, \ldots, L
  \quad \text{and} \quad
  \begin{cases}
   \By'_l &= \Bt_l + \frac{ \boldsymbol{\lambda}_{\boldsymbol 4 l} }{\beta_4} + \Br_l \cdot^\times \frac{\boldsymbol{\lambda_1} + \beta_1}{\beta_4} \,,
   \\
   \\
   \abs{\vec{\By'}} &= \sqrt{ \sum_{l=0}^L \left[ \Bt_l + \frac{ \boldsymbol{\lambda}_{\boldsymbol 4 l} }{\beta_4} + \Br_l \cdot^\times \frac{\boldsymbol{\lambda_1} + \beta_1}{\beta_4} \right]^{\cdot 2} } \,.
  \end{cases}  
 \end{align*}
 
\end{prop}

\begin{proof}
Under a minimization, the objective function is rewritten as
\begin{align*}
 \mathscr F(\vec{\By}) &= \mathscr R^*(\vec{\By}) - \sum_{l=0}^L \left\langle \boldsymbol{\lambda_1} + \beta_1 \,, \By_l \cdot^\times \Br_l \right\rangle_{\ell_2}
 + \frac{\beta_4}{2} \sum_{l=0}^L \norm{ \By_l - \left( \Bt_l + \frac{\boldsymbol \lambda_{\boldsymbol 4l}}{\beta_4} \right) }^2_{\ell_2}
 \\
 &= \mathscr R^*(\vec{\By}) + \frac{\beta_4}{2} \sum_{l=0}^L \left[  
 \norm{ \By_l }^2_{\ell_2} - 2 \left\langle \By_l \,, \Bt_l + \frac{\boldsymbol{\lambda_{4l}}}{\beta_4} + \Br_l \cdot^\times \frac{\boldsymbol{\lambda_1} + \beta_1}{\beta_4}
 \right\rangle_{\ell_2} + \norm{ \Bt_l + \frac{\boldsymbol{\lambda_{4l}}}{\beta_4} + \Br_l \cdot^\times \frac{\boldsymbol \lambda_1 + \beta_1}{\beta_4} }^2_{\ell_2} \right]
 \\&
 - \frac{\beta_4}{2} \sum_{l=0}^L \norm{ \Bt_l + \frac{\boldsymbol{\lambda_{4l}}}{\beta_4} + \Br_l \cdot^\times \frac{\boldsymbol{\lambda_1} + \beta_1}{\beta_4} }^2_{\ell_2}
 + \frac{\beta_4}{2} \sum_{l=0}^L \norm{ \Bt_l + \frac{\boldsymbol{\lambda_{4l}}}{\beta_4} }^2_{\ell_2}
 \\
 &= \mathscr R^*(\vec{\By}) + \frac{\beta_4}{2} \norm{ \vec{\By} - \left[ \vec{\Bt} + \frac{ \vec{\boldsymbol{\lambda}}_{\boldsymbol 4} }{\beta_4} + \vec{\Br} \cdot^\times \frac{\boldsymbol{\lambda_1} + \beta_1}{\beta_4} \right] }^2_{\ell_2}
 + \frac{\beta_4}{2} \left[ \norm{\vec{\Bt} + \frac{\vec{\boldsymbol{\lambda}}_{\boldsymbol 4}}{\beta_4} }^2_{\ell_2}
 - \norm{ \vec{\Bt} + \frac{\vec{\boldsymbol \lambda}_{\boldsymbol 4}}{\beta_4} + \vec{\Br} \cdot^\times \frac{\boldsymbol{\lambda_1} + \beta_1}{\beta_4}
 }^2_{\ell_2}
 \right] \,.
\end{align*}
According to Lemma \ref{lem:L1-projection}, a solution of a minimization of the "$\vec{\By}$-problem" is
\begin{align*}
 \vec{\By}^* = \argmin_{\vec{\By} \in X^{L+1}} \left\{
 \mathscr R^*(\vec{\By}) 
 + \frac{\beta_4}{2} \bigg|\bigg| \vec{\By} 
 - \bigg[ 
 \underbrace{ 
 \vec{\Bt} + \frac{ \vec{\boldsymbol{\lambda}}_{\boldsymbol 4} }{\beta_4} + \vec{\Br} \cdot^\times \frac{\boldsymbol{\lambda_1} + \beta_1}{\beta_4} 
 }_{= \vec{\By'}}
 \bigg] 
 \bigg|\bigg|^2_{\ell_2}
 \right\}
 = \begin{cases}
     \vec{\By'} \,, & \abs{\vec{\By'}} \leq 1  \\
     \frac{ \vec{\By'} }{ \abs{\vec{\By'}} } \,, & \text{else}
    \end{cases}
\end{align*}
with 
\begin{align*}
 \vec{\By'} &= \left[ \By'_l \right]_{l=0}^L \in X^{L+1} 
 \,,~~
 \By'_l = \Bt_l + \frac{ \boldsymbol{\lambda}_{\boldsymbol 4 l} }{\beta_4} + \Br_l \cdot^\times \frac{\boldsymbol{\lambda_1} + \beta_1}{\beta_4}
 ~~ \text{and}~~
 \abs{\vec{\By'}} = \sqrt{ \sum_{l=0}^L \left[ \Bt_l + \frac{ \boldsymbol{\lambda}_{\boldsymbol 4 l} }{\beta_4} + \Br_l \cdot^\times \frac{\boldsymbol{\lambda_1} + \beta_1}{\beta_4} \right]^{\cdot 2} } \,.
\end{align*}

\end{proof}


\begin{prop}  \label{prop:problem:u}

The solution of the ``$\Bu$-problem'' is as follows 
\begin{align*}  
 \Bu^* &= \min_{\Bu \in X} \left\{ \mathscr F(\Bu) :=  
 \frac{\beta_2}{2} \sum_{l=0}^{L-1} \norm{ \Br_l - \partial^+_l \Bu + \frac{ \boldsymbol{\lambda}_{\boldsymbol 2 l} }{\beta_2} }^2_{\ell_2}
 + \frac{\beta_5}{2} \norm{ \Bf - \Bh \ast \Bu - \Bh \ast \Bv - \Bh \ast \Brho - \Beps + \frac{\boldsymbol{\lambda}_{\boldsymbol 5}}{\beta_5} }^2_{\ell_2}
 \right\}
 \\
 &= \RE \left[ \cF^{-1} \left\{ \frac{\cY(\Bz)}{\cX(\Bz)} \right\} \right] \,.
\end{align*}
with
\begin{align*}
 \cY(\Bz) &= \beta_2 \sum_{l=0}^{L-1} \left[ \cos\left(\frac{\pi l}{L}\right)(z_2^{-1} - 1) + \sin\left(\frac{\pi l}{L}\right)(z_1^{-1} - 1) \right]
 \left[ R_l(\Bz) + \frac{ \Lambda_{2 l}(\Bz) }{\beta_2} \right]
 \\
 &+ \beta_5 H(\Bz^{-1}) \left[ F(\Bz) - H(\Bz) V(\Bz) - H(\Bz) P(\Bz) - \cE(\Bz) + \frac{\Lambda_5(\Bz)}{\beta_5} \right] \,,
 \\
 \cX(\Bz) &= \beta_2 \sum_{l=0}^{L-1} \abs{ \cos\left(\frac{\pi l}{L}\right)(z_2 - 1) + \sin\left(\frac{\pi l}{L}\right)(z_1 - 1) }^2 + \beta_5 \abs{H(\Bz)}^2 \,.
\end{align*}

\end{prop}

\begin{proof}
 The Euler-Lagrange equation is
 \begin{align*}
  &0 = \frac{\partial \mathscr F(\Bu)}{\partial \Bu} = 
  -\beta_2 \sum_{l=0}^{L-1} \underbrace{ \frac{\partial \left\{ \partial_l^+ \Bu \right\} }{\partial \Bu} }_{ = (\partial_l^+)^* \delta[\Bk] = -\partial_l^- \delta[\Bk] }
  \left[ \Br_l - \partial^+_l \Bu + \frac{ \boldsymbol{\lambda}_{\boldsymbol 2 l} }{\beta_2} \right]
  - \beta_5 \check{\Bh} \ast \left[ \Bf - \Bh \ast \Bu - \Bh \ast \Bv - \Bh \ast \Brho - \Beps + \frac{\boldsymbol{\lambda}_{\boldsymbol 5}}{\beta_5} \right]
  \\ \Leftrightarrow~ &
  \left[ -\beta_2 \sum_{l=0}^{L-1} \partial_l^- \partial^+_l \delta + \beta_5 \check{\Bh} \ast \Bh \right] \ast \Bu
  =
  - \beta_2 \sum_{l=0}^{L-1} \partial_l^- \left[ \Br_l + \frac{ \boldsymbol{\lambda}_{\boldsymbol 2 l} }{\beta_2} \right]
  + \beta_5 \check{\Bh} \ast \left[ \Bf - \Bh \ast \Bv - \Bh \ast \Brho - \Beps + \frac{\boldsymbol{\lambda}_{\boldsymbol 5}}{\beta_5} \right] 
  \\
  \stackrel{\cF}{\longleftrightarrow}~ 
  &\left[ \beta_2 \sum_{l=0}^{L-1} \abs{ \cos(\frac{\pi l}{L})(z_2 - 1) + \sin(\frac{\pi l}{L})(z_1 - 1) }^2 
  + \beta_5 \abs{H(\Bz)}^2 \right] U(\Bz)
  \\&
  = \beta_2 \sum_{l=0}^{L-1} \left[ \cos\left(\frac{\pi l}{L}\right)(z_2^{-1} - 1) + \sin\left(\frac{\pi l}{L}\right)(z_1^{-1} - 1) \right]
  \left[ R_l(\Bz) + \frac{ \Lambda_{2 l}(\Bz) }{\beta_2} \right]
  \\&
  + \beta_5 H(\Bz^{-1}) \left[ F(\Bz) - H(\Bz) V(\Bz) - H(\Bz) P(\Bz) - \cE(\Bz) + \frac{\Lambda_5(\Bz)}{\beta_5} \right] 
 \end{align*}
 By applying the Fourier transform, a solution of the ``$\Bu$-problem'' is
 \begin{align*}
  U(\Bz) := \frac{\cY(\Bz)}{\cX(\Bz)} = 
  &\left[ \beta_2 \sum_{l=0}^{L-1} \abs{ \cos(\frac{\pi l}{L})(z_2 - 1) + \sin(\frac{\pi l}{L})(z_1 - 1) }^2 + \beta_5 \abs{H(\Bz)}^2 \right]^{-1}
  \\&
  \bigg[ \beta_2 \sum_{l=0}^{L-1} \left[ \cos(\frac{\pi l}{L})(z_2^{-1} - 1) + \sin(\frac{\pi l}{L})(z_1^{-1} - 1) \right]
  \left[ R_l(\Bz) + \frac{ \Lambda_{2 l}(\Bz) }{\beta_2} \right]
  \\&
  + \beta_5 H(\Bz^{-1}) \left[ F(\Bz) - H(\Bz) V(\Bz) - H(\Bz) P(\Bz) - \cE(\Bz) + \frac{\Lambda_5(\Bz)}{\beta_5} \right] \bigg] \,.
 \end{align*}
  
\end{proof}


\begin{prop} 
 The sum of two $\ell_2$ functions is rewritten as
 \begin{align*}
  \mathscr F(\Bv) & := \beta_1 \norm{\Bv - \mathbf{f_1}}^2_{\ell_2} ~+~ \beta_2 \norm{\Bv - \mathbf{f_2}}^2_{\ell_2}
  \\
  &= (\beta_1 + \beta_2) \norm{ \Bv - \frac{\beta_1}{\beta_1 + \beta_2} \mathbf{f_1} - \frac{\beta_2}{\beta_1 + \beta_2} \mathbf{f_2} }^2_{\ell_2}
  ~+~ c_{f_1 f_2}
 \end{align*}
\end{prop}
 
\begin{proof}
 \begin{align*}
  \mathscr F(\Bv) &= (\beta_1 + \beta_2) \norm{\Bv}^2_{\ell_2} ~-~ 2\langle \Bv \,, \beta_1 \mathbf{f_1} ~+~ \beta_2 \mathbf{f_2} \rangle_{\ell_2} 
  ~+~ \beta_1 \norm{\mathbf{f_1}}^2_{\ell_2} ~+~ \beta_2 \norm{\mathbf{f_2}}^2_{\ell_2}
  \\
  &= (\beta_1 + \beta_2) \norm{ \Bv - \frac{\beta_1}{\beta_1 + \beta_2} \mathbf{f_1} - \frac{\beta_2}{\beta_1 + \beta_2} \mathbf{f_2} }^2_{\ell_2}
  \\
  &+ \underbrace{
  (\beta_1 + \beta_2) \left[ - \norm{\frac{\beta_1}{\beta_1 + \beta_2}\mathbf{f_1} + \frac{\beta_2}{\beta_1 + \beta_2}\mathbf{f_2} }^2_{\ell_2} 
  ~+~ \frac{\beta_1}{\beta_1 + \beta_2} \norm{\mathbf{f_1}}^2_{\ell_2} ~+~ \frac{\beta_2}{\beta_1 + \beta_2} \norm{\mathbf{f_2}}^2_{\ell_2} \right]
  }_{:=~ c_{f_1 f_2}}
 \end{align*}
\end{proof} 
 
\begin{lem} \label{lem:linearizeConvexFunc}
 Given a non-smooth convex function $f_1$, e.g. $f_1(\cdot) = \norm{\cdot}_{\ell_1}$ and
 a smooth convex function $f_2$, e.g. $f_2(\cdot) = \norm{\cdot}_{\ell_2}^2$, we linearize a convex minimization as
 \begin{align*}  
  \Bv^* &= \argmin_{\Bv \in X} \left\{ \mu f_1(\Bv) + f_2(\Bv) \right\}
  \\
  \Leftrightarrow~ 
  \Bv^{(\tau)} &= \argmin_{\Bv \in X} \left\{ 
  \mu f_1(\Bv) + \frac{1}{2 \alpha^{(\tau)}} \norm{ \Bv - \left( \Bv^{(\tau-1)} - \alpha^{(\tau)} \nabla_{\Bv} f_2( \Bv^{(\tau-1)} ) \right) }^2_{\ell_2}
  \right\} \,,~~ \tau = 1, \ldots
 \end{align*}
\end{lem}

\begin{proof}
 
\begin{rem}
 1-dimensional Taylor expansion of a function $f(x)$ at $a$ is
 \begin{align*}
  f(x) \mid_{x = a} ~\approx~ \frac{f^{(1)}(a)}{1!} (x-a) + \frac{f^{(2)}(a)}{2!} (x-a)^2 + \frac{f^{(3)}(a)}{3!} (x-a)^3 + \cdots
 \end{align*}
 with $f^{(i)}(a)$ is a $i$-th order derivative of $f(x)$ at $x = a$, i.e. 
 $f^{(i)}(a) = \frac{\text{d}^i f(x)}{\text{d} x^i} \Big |_{x = a} $.
 
 The 1st order of the 2-dimensional Taylor expansion of a function $f(\Bx)$ at $\Bx^{(\tau-1)}$ with $\Bx = (x \,, y)$ is
 \begin{align}  \label{eq:Taylorexpansion}
  f(\Bx) \mid_{\Bx = \Bx^{(\tau-1)}} \approx 
  f(\Bx^{(\tau-1)}) + \left\langle \Bx - \Bx^{(\tau-1)} \,, \nabla_{\Bx} f(\Bx^{(\tau-1)}) \right\rangle_{\ell_2}  
  + \frac{1}{2 \alpha^{(\tau)}} \norm{ \Bx - \Bx^{(\tau-1)} }^2_{\ell_2}
 \end{align}
 
\end{rem}
As in ISTA \cite{DaubechiesDefriseMol2004} or FISTA \cite{BeckTeboulle2009}, given a convex minimization
\begin{align}  \label{eq:cvxapproximation:ISTA}
 \min_{\Bv \in X} \left\{ \mu f_1(\Bv) + f_2(\Bv) \right\}
\end{align}
$f_1$ is a non-smooth (non-differentiable) convex function, e.g. $f_1(\cdot) = \norm{\cdot}_{\ell_1}$ and
$f_2$ is a smooth (differentiable) convex function, e.g. $f_2(\cdot) = \norm{\cdot}_{\ell_2}^2$.
By the 1-st order Taylor expansion of a smooth function $f_2(\Bv)$ at $\Bv^{(\tau-1)}$ in (\ref{eq:Taylorexpansion}),
minimization (\ref{eq:cvxapproximation:ISTA}) is rewritten as
\begin{align}  \label{eq:cvxapproximation:ISTA:Taylor}
 \Bv^{(\tau)} &= \argmin_{\Bv \in X} 
 \Big\{ \mu f_1(\Bv) + 
 \overbrace{ f_2(\Bv^{(\tau-1)}) + 
             \underbrace{ \Big\langle \Bv - \Bv^{(\tau-1)} \,, \nabla_{\Bv} f_2(\Bv^{(\tau-1)}) \Big\rangle_{\ell_2}  
                          + \frac{1}{2 \alpha^{(\tau)}} \norm{ \Bv - \Bv^{(\tau-1)} }^2_{\ell_2} 
                        }_{\displaystyle = \frac{1}{2 \alpha^{(\tau)}} \norm{ \Bv - \Big( \Bv^{(\tau-1)} - \alpha^{(\tau)} \nabla_\Bv f_2( \Bv^{(\tau-1)} ) \Big) }^2_{\ell_2} + \text{const} }
           }
          ^{\displaystyle \approx f_2(\Bv) \mid_{\Bv = \Bv^{(\tau-1)}} 
 } \Big\}   \notag
 \\
 &= \argmin_{\Bv \in X} \left\{ 
 \mu f_1(\Bv) + \frac{1}{2 \alpha^{(\tau)}} \norm{ \Bv - \left( \Bv^{(\tau-1)} - \alpha^{(\tau)} \nabla_{\Bv} f_2( \Bv^{(\tau-1)} ) \right) }^2_{\ell_2}
 \right\} \,.
\end{align}

\end{proof}


\begin{prop}  \label{prop:problem:v}

The solution of the ``$\Bv$-problem'' is as follows 
\begin{align*}
 \Bv^* &= \argmin_{\Bv \in X} \left\{  
 \mu_2 \norm{\Bv}_{\ell_1} 
 + \frac{\beta_5}{2} \norm{ \Bh \ast \Bv - \Big( \Bf - \Bh \ast \Bu - \Bh \ast \Brho - \Beps + \frac{\boldsymbol{\lambda}_{\boldsymbol 5}}{\beta_5} \Big) }^2_{\ell_2}
 + \frac{\beta_7}{2} \norm{ \Bv - \left( \text{div}^-_S \vec{\Bg} - \frac{\boldsymbol{\lambda}_{\boldsymbol 7}}{\beta_7} \right) }^2_{\ell_2}
 \right\}
 \\ \Leftrightarrow~
 \Bv^{(\tau)} &= \Shrink \left( \Bt_\Bv^{(\tau)} \,,~ \frac{\mu_2 \alpha^{(\tau)}}{\beta_5 + \alpha^{(\tau)} \beta_7} \right) \,,~~ \tau = 1, \ldots 
\end{align*}
with
\begin{align*}
 \Bt_\Bv^{(\tau)} &= \frac{ \beta_5 }{ \beta_5 + \alpha^{(\tau)} \beta_7 }
 \left( \left[ \delta - \alpha^{(\tau)} \check{\Bh} \ast \Bh \right] \ast \Bv^{(\tau-1)}     
 + \alpha^{(\tau)} \check{\Bh} \ast \left( \Bf - \Bh \ast \Bu - \Bh \ast \Brho - \Beps + \frac{\boldsymbol{\lambda}_{\boldsymbol 5}}{\beta_5} \right) \right) 
 \\
 &+ \frac{ \beta_7 \alpha^{(\tau)} }{ \beta_5 + \alpha^{(\tau)} \beta_7 }
 \Big( \underbrace{ -\sum_{s=0}^{S-1} \big[ \cos(\frac{\pi s}{S}) \Bg_s \BDt + \sin(\frac{\pi s}{S}) \BDoT \Bg_s \big] }_{= \text{div}^-_S \vec{\Bg} }
     - \frac{\boldsymbol{\lambda}_{\boldsymbol 7}}{\beta_7} \Big) \,.
\end{align*}

\end{prop}

\begin{proof}
 
According to Lemma \ref{lem:linearizeConvexFunc}, a linearized version of the ``$\Bv$-problem'' is
\begin{align*}
 \Bv^* &= \argmin_{\Bv \in X} \bigg\{  
 \frac{\mu_2}{\beta_5} \norm{\Bv}_{\ell_1} 
 + \underbrace{ \frac{1}{2} \norm{ \Bh \ast \Bv - \Big( \Bf - \Bh \ast \Bu - \Bh \ast \Brho - \Beps + \frac{\boldsymbol{\lambda}_{\boldsymbol 5}}{\beta_5} \Big) }^2_{\ell_2} }_{\displaystyle = q(\Bv)}
 + \frac{\beta_7}{2\beta_5} \norm{ \Bv - \Big( \text{div}^-_S \vec{\Bg} - \frac{\boldsymbol{\lambda}_{\boldsymbol 7}}{\beta_7} \Big) }^2_{\ell_2}
 \bigg\} 
 \\
 \Leftrightarrow~ 
 \Bv^{(\tau)} &= \argmin_{\Bv \in X} \Bigg\{ 
 \frac{\mu_2}{\beta_5} \norm{\Bv}_{\ell_1}  
 + \frac{1}{2 \alpha^{(\tau)}} \bigg|\bigg| \Bv - \Big( \Bv^{(\tau-1)} 
 - \alpha^{(\tau)} \overbrace{ \check{\Bh} \ast \Big[ \Bh \ast \Bv^{(\tau-1)} - \big( \Bf - \Bh \ast \Bu - \Bh \ast \Brho - \Beps + \frac{\boldsymbol{\lambda}_{\boldsymbol 5}}{\beta_5} \big) \Big] }
                            ^{= \frac{\partial q(\Bv)}{\partial \Bv} \mid_{\Bv = \Bv^{(\tau-1)}} }
 \Big) 
 \bigg |\bigg | ^2_{\ell_2}
 \\& \qquad \qquad \quad
 + \frac{\beta_7}{2\beta_5} \norm{ \Bv - \Big( \text{div}^-_S \vec{\Bg} - \frac{\boldsymbol{\lambda}_{\boldsymbol 7}}{\beta_7} \Big) }^2_{\ell_2}
 \Bigg\} \,,~~ \tau = 1, \ldots 
 \\&
 = \argmin_{\Bv \in X} \Bigg\{ 
 \norm{\Bv}_{\ell_1}  
 + \frac{\beta_5 + \alpha^{(\tau)} \beta_7}{2 \mu_2 \alpha^{(\tau)}} \norm{ \Bv - \Bt_\Bv }^2_{\ell_2}
 \Bigg\} 
\end{align*}
with
\begin{align*}
 \Bt_\Bv = \frac{ \beta_5 }{ \beta_5 + \alpha^{(\tau)} \beta_7 }
 \Big( \Bv^{(\tau-1)} - \alpha^{(\tau)} \check{\Bh} \ast \Big[ \Bh \ast \Bv^{(\tau-1)} - \big( \Bf - \Bh \ast \Bu - \Bh \ast \Brho - \Beps + \frac{\boldsymbol{\lambda}_{\boldsymbol 5}}{\beta_5} \big) \Big] \Big) 
 + \frac{ \beta_7 \alpha^{(\tau)} }{ \beta_5 + \alpha^{(\tau)} \beta_7 }
 \Big( \text{div}^-_S \vec{\Bg} - \frac{\boldsymbol{\lambda}_{\boldsymbol 7}}{\beta_7} \Big) \,.
\end{align*}
Thus, a solution of the $\Bv$-problem is defined at iteration $\tau$ as
\begin{align*}
 \Bv^{(\tau)} = \Shrink \Big( \Bt_\Bv^{(\tau)} \,,~ \frac{\mu_2 \alpha^{(\tau)}}{\beta_5 + \alpha^{(\tau)} \beta_7} \Big) \,,~~ \tau = 1, \ldots 
\end{align*}

\end{proof}

\begin{prop}  \label{prop:problem:g}

The ``$\vec \Bg$-problem'' 
\begin{align*}
 \min_{\vec{\Bg} \in X^S} \bigg\{ \mathscr F(\vec{\Bg}) := 
 \frac{\beta_6}{2} \norm{ \vec{\Bw} - \vec{\Bg} + \frac{\vec{\boldsymbol{\lambda}}_{\boldsymbol 6}}{\beta_6}}^2_{\ell_2}
 + \frac{\beta_7}{2} \norm{ \Bv - \text{div}^-_S \vec{\Bg} + \frac{\boldsymbol{\lambda}_{\boldsymbol 7}}{\beta_7} }^2_{\ell_2} 
 \bigg\}
\end{align*}
has a minimizer as
\begin{align*}
G_s(\Bz) = \frac{ \cB_s(\Bz) }{ \cA_s(\Bz) } \,,~~ s = 0, \ldots, S-1 \,.
\end{align*}
with
\begin{align*}
 &\cA_s(\Bz) = \beta_6 + \beta_7 \abs{ \cos\left(\frac{\pi s}{S}\right)(z_2 - 1) + \sin\left(\frac{\pi s}{S}\right)(z_1 - 1) }^2 \,,
 \\
 &\cB_s(\Bz) = \beta_6 \left[ W_s(\Bz) + \frac{\Lambda_{6 s}(\Bz)}{\beta_6} \right]
 -\beta_7 \left[ \cos\left(\frac{\pi s}{S}\right)(z_2 - 1) + \sin\left(\frac{\pi s}{S}\right) (z_1 - 1) \right] \times
 \\&
 \left[ V(\Bz) + \sum_{s'=[0, S-1] \backslash \{s\}} \left[ \cos\left(\frac{\pi s'}{S}\right)(z_2^{-1} - 1) + \sin\left(\frac{\pi s'}{S}\right) (z_1^{-1} - 1) \right] G_{s'}(\Bz) + \frac{\Lambda_7(\Bz)}{\beta_7} \right] \,.
\end{align*}

\end{prop}

\begin{proof}

Given
$\vec{\Bg} = \big[ \Bg_s \big]_{s=0}^{S-1} \,,
\vec{\Bw} = \big[ \Bw_s \big]_{s=0}^{S-1} \,, \vec{\boldsymbol{\lambda}}_{\boldsymbol 6} = \big[ \boldsymbol{\lambda}_{\boldsymbol 6s} \big]_{s=0}^{S-1}$
and $\nabla_S^+ = \big[ \partial_s^+ \big]_{s=0}^{S-1}$, 
the Euler-Lagrange equation is
\begin{align*}
 &0 = \frac{\partial {\mathscr F}(\vec{\Bg})}{\partial \vec{\Bg}} = 
 - \beta_6 \Big[ \vec{\Bw} - \vec{\Bg} + \frac{ \vec{\boldsymbol{\lambda}}_{\boldsymbol 6} }{ \beta_6 } \Big]
 - \beta_7 \underbrace{ \frac{ \partial \Big\{ \text{div}^-_S \vec{\Bg} \Big\} }{ \partial \vec{\Bg} } }
                     _{ = \big( \text{div}^-_S \big)^* \delta = - \nabla_S^+ \delta }
 \Big[ \Bv - \text{div}^-_S \vec{\Bg} +\frac{\boldsymbol{\lambda_7}}{\beta_7} \Big]
\end{align*}
Given $s = 1, \ldots, S-1$, we have
\begin{align*}
 0 &= -\beta_6 \Big[ \Bw_s - \Bg_s + \frac{ \boldsymbol{\lambda}_{\boldsymbol 6 s} }{ \beta_6 } \Big]
 + \beta_7 \partial_s^+ \Big[ \Bv - \underbrace{ \sum_{s'=0}^{S-1} \partial^-_{s'} \Bg_{s'} }_{= \text{div}^-_S \vec{\Bg} } +\frac{\boldsymbol{\lambda_7}}{\beta_7} \Big] 
 \\
 \Leftrightarrow 
 \Bg_s &= \Big[ \beta_6 - \beta_7 \partial^+_s \partial^-_s \text{Id} \Big]^{-1}
         \bigg[ -\beta_7 \partial_s^+ \Big[ \Bv - \sum_{s'=[0, S-1] \backslash \{s\}} \partial^-_s \Bg_s + \frac{\boldsymbol{\lambda_7}}{\beta_7} \Big]
                + \beta_6 \Big[ \Bw_s + \frac{\boldsymbol{\lambda}_{\boldsymbol 6 s}}{\beta_6} \Big]
         \bigg]
\end{align*}
Thus, a solution of the "$\vec{\Bg}$-problem" is obtained in the Fourier domain.

\end{proof}

\begin{prop}  \label{prop:VCPyramid:problem:u:spatial}
 A solution of the "$\Bu$-problem" (\ref{eq:problem:u}) can be rewritten in a form of sampling theory as in equation (\ref{eq:VCPyramid:problem:u:spatial:1}).
\end{prop}

\begin{proof}

we analyze filter banks generated by the DMCD model at iteration $\tau$ in the listed order of Algorithm 1-4 in Appendix C.

{\bfseries a. The ``$\vec{\Bt}$-problem'':}

From (\ref{eq:problem:d:solution}) and a solution of the ``$\vec{\Bt}$-problem'' (\ref{eq:problem:t}), we have
{\normalsize 
\begin{align*} 
 \cN_l(\Bz) &\stackrel{\cF^{-1}}{\longleftrightarrow}
 \beta_4 \delta[\Bk] - \beta_3 \partial_l^- \partial_l^+ \delta[\Bk]
 \\
 \cM_l^{(\tau)}(\Bz) &\stackrel{\cF^{-1}}{\longleftrightarrow}
 \beta_4 \Big[ y_l^{(\tau-1)} - \frac{ \lambda_{4l}^{(\tau-1)} }{ \beta_4 } \Big] [\Bk]
  \\&
 -\beta_3 \partial_l^+ 
 \Big[ \underbrace{ \text{Shrink} \Big( \sum_{l'=0}^{L-1} \partial_{l'}^- t_{l'}^{(\tau-1)} - \frac{\lambda_3^{(\tau-1)}}{\beta_3}, \frac{1}{\beta_3} \Big) }
                 _{ = d^{(\tau)}[\Bk] }
       + \frac{ \lambda_3^{(\tau-1)} }{ \beta_3 } 
       - \sum_{l'=[0, L-1] \backslash \{l\}} \partial_{l'}^- t_{l'}^{(\tau-1)} \Big] [\Bk] 
\end{align*}
}
Due to
$
 0 < \beta_4 \leq \cN_l(e^{j\Bome}) < +\infty \,,~ \forall \Bome \in [-\pi, \pi]^2 \,,
$
and choose $\beta_3 = c_{34} \beta_4$,
a solution of the ``$\vec{\Bt}$-problem'' at a direction $l = 0, \ldots, L-1$ is rewritten as
\begin{align} \label{eq:VCPyramid:problem:t:1}
 T_l^{(\tau)}(\Bz) &= \Xi^{L, c_{34}}_l(\Bz) \Big[ Y_l^{(\tau-1)}(\Bz) - \frac{ \Lambda_{4l}^{(\tau-1)}(\Bz) }{ \beta_4 } \Big]   \notag
 \\&
 + \Theta^{L, c_{34}}_l(\Bz)
 \Big[ D^{(\tau)}(\Bz) + \sum_{l'=[0, L-1] \backslash \{l\}} \big[ \cos\left(\frac{\pi l'}{L}\right) (z_2^{-1} - 1) + \sin\left(\frac{\pi l'}{L}\right) (z_1^{-1} - 1) \big] T_{l'}^{(\tau-1)}(\Bz) + \frac{ \Lambda_3^{(\tau-1)}(\Bz) }{ \beta_3 } \Big] \,.
\end{align}
Given $\Bk \in \Omega$, the inverse Fourier transform of (\ref{eq:VCPyramid:problem:t:1}) at a direction $l = 0, \ldots, L-1$ is
\begin{align} \label{eq:VCPyramid:problem:t:2}
 t_l^{(\tau)}[\Bk] &= \Big( \xi^{L, c_{34}}_l \ast \big[ y_l^{(\tau-1)} - \frac{ \lambda_{4l}^{(\tau-1)} }{ \beta_4 } \big] \Big) [\Bk]   \notag
 \\&
 + \Big( \theta^{L, c_{34}}_l \ast \big[ \underbrace{ \text{Shrink} \Big( \sum_{l'=0}^{L-1} \check{\tilde \theta}^L_{l'} \ast t_{l'}^{(\tau-1)} - \frac{\lambda_3^{(\tau-1)}}{\beta_3} \,,~ \frac{1}{\beta_3} \Big) + \frac{ \lambda_3^{(\tau-1)} }{ \beta_3 } }
                                       _{\displaystyle = \ST \big( \sum_{l'=0}^{L-1} \check{\tilde \theta}^L_{l'} \ast t_{l'}^{(\tau-1)} \,, \frac{\lambda_3^{(\tau-1)}}{\beta_3} \,, \frac{ 1 }{ \beta_3 } \big) }
         - \sum_{l'=[0, L-1] \backslash \{l\}} \check{\tilde \theta}^L_{l'} \ast  t_{l'}^{(\tau-1)} \big]
   \Big) [\Bk] \,.
\end{align}
Frames at a direction $l = 0, \ldots, L-1$ are 
\begin{align} \label{eq:FilterBanks:XiThetaTheta_tilde:1}
 \Xi^{L, c_{34}}_l(\Bz) &= \frac{ 1 }{ 1 + c_{34} \abs{ \cos\left(\frac{\pi l}{L}\right) (z_2 - 1) + \sin\left(\frac{\pi l}{L}\right) (z_1 - 1) }^2 }  
 ~\stackrel{\cF^{-1}}{\longleftrightarrow}~  
 \xi^{L, c_{34}}_l[\Bk] = \left[ 1 - c_{34} \partial_l^- \partial_l^+ \right]^{-1} \delta[\Bk] \,,
 \\  \label{eq:FilterBanks:XiThetaTheta_tilde:2}
 \Theta^{L, c_{34}}_l(\Bz) &= 
 \frac{ -c_{34} \left[ \cos\left(\frac{\pi l}{L}\right) (z_2 - 1) + \sin\left(\frac{\pi l}{L}\right) (z_1 - 1) \right] }
      { 1 + c_{34} \abs{ \cos\left(\frac{\pi l}{L}\right) (z_2 - 1) + \sin\left(\frac{\pi l}{L}\right) (z_1 - 1) }^2 }
 ~\stackrel{\cF^{-1}}{\longleftrightarrow}~  
 \theta^{L, c_{34}}_l[\Bk] = -c_{34} \partial_l^+ \xi^{L, c_{34}}_l[\Bk] \,,
 \\  \label{eq:FilterBanks:XiThetaTheta_tilde:3}
 \tilde \Theta^L_l(\Bz^{-1}) &= - \left[ \cos\left(\frac{\pi l}{L}\right) (z_2^{-1} - 1) + \sin\left(\frac{\pi l}{L}\right) (z_1^{-1} - 1) \right]
 ~\stackrel{\cF^{-1}}{\longleftrightarrow}~  
 \check{\tilde \theta}^L_l[\Bk] = \partial_l^- \delta[\Bk] \,.
\end{align}
Note that these frames satisfy the unity condition as
\begin{align} \label{eq:FilterBanks:XiThetaTheta_tilde:unitycondition}
 &\Xi^{L, c_{34}}_l(\Bz) + \Theta^{L, c_{34}}_l(\Bz) \tilde \Theta^L_l(\Bz^{-1}) = 1 \,,
 \\
 & \Xi^{L, c_{34}}_l(e^{j\mathbf 0}) = 1 
 \text{ and }
 \Theta^{L, c_{34}}_l(e^{j\mathbf 0}) = \tilde \Theta^L_l(e^{j\mathbf 0}) = 0 \,.   \notag
\end{align}
Figure \ref{fig:MDCD_let:FilterBanks:c0_1} and \ref{fig:MDCD_let:FilterBanks:c10} illustrate the spectrum of these filter banks.


{\bfseries b. The ``$\vec{\By}$-problem'':}

Given $\Bk \in \Omega$, a solution of the ``$\vec{\By}$-problem'' (\ref{eq:problem:y}) (at iteration $\tau$) is rewritten as
\begin{align*}
 \By_l^{(\tau)} &= 
 \begin{cases}
  \By_l^{\prime (\tau)} \,, & \abs{\vec{\By}^{\prime(\tau)}} \leq 1  \\
  \frac{ \By_l^{\prime(\tau)} }{ \abs{\vec{\By}^{\prime(\tau)}} } \,, & \abs{\vec{\By}^{\prime(\tau)}} > 1
 \end{cases}
 \,,~~ l = 0, \ldots, L
 \,, \quad \text{and} \quad
 \begin{cases} 
  \By_l^{\prime(\tau)} &= \Bt_l^{(\tau)} + \frac{ \boldsymbol{\lambda}_{\boldsymbol 4 l}^{(\tau-1)} }{\beta_4} + \frac{\beta_1 + \boldsymbol{\lambda}_{\boldsymbol 1}^{(\tau-1)}}{\beta_4} \cdot^\times \Br_l^{(\tau)} 
  \\
  \abs{\vec{\By'}} &= \sqrt{ \sum_{l=0}^L \Big[ \Bt_l^{(\tau)} + \frac{ \boldsymbol{\lambda}_{\boldsymbol 4 l}^{(\tau-1)} }{\beta_4} + \frac{\beta_1 + \boldsymbol{\lambda}_{\boldsymbol 1}^{(\tau-1)}}{\beta_4} \cdot^\times \Br_l^{(\tau)} \Big]^{\cdot 2} } 
 \end{cases} 
 \\
 \Leftrightarrow~ 
 \vec{y}^{(\tau)}[\Bk] &= \text{Proj}_{[-1,1]} 
 \bigg[ \underbrace{
         \vec{t}^{(\tau)}[\Bk] + \frac{ \vec{\lambda}_4^{(\tau-1)}[\Bk] }{\beta_4} 
         + \frac{\beta_1 + \lambda_1^{(\tau-1)}[\Bk]}{\beta_4} \vec{r}^{(\tau)}[\Bk] 
        }_{ \displaystyle = \vec{y}^{\prime(\tau)}[\Bk] = \Big[ y_l^{\prime(\tau)}[\Bk] \Big]_{l=0}^L }
 \bigg] \,.
\end{align*}
By denoting 
$\vec{y}^{(\tau)} [\Bk] = \Big[ y_l^{(\tau)} [\Bk] \Big]_{l=0}^{L} \,,~
\vec{t}^{(\tau)} [\Bk] = \Big[ t_l^{(\tau)} [\Bk] \Big]_{l=0}^{L} \,,~
\vec{r}^{(\tau)} [\Bk] = \Big[ r_l^{(\tau)} [\Bk] \Big]_{l=0}^{L} \,,
$ 
and from (\ref{eq:VCPyramid:problem:t:2}) and a solution of the ``$\vec{\Br}$-problem'' at direction $l=0, \ldots, L-1$, 
we rewrite $\vec{y}^{(\tau)}[\Bk] = \text{Proj}_{[-1,1]} \Big[ \vec{y}^{\prime(\tau)}[\Bk] \Big]$ with
\begin{align} \label{eq:VCPyramid:problem:y}
 y_l^{\prime(\tau)}[\Bk] &= 
 t_l^{(\tau)}[\Bk] + \frac{ \lambda_{4l}^{(\tau-1)}[\Bk] }{\beta_4} 
 + \frac{\beta_1 + \lambda_1^{(\tau-1)}[\Bk]}{\beta_4} r_l^{(\tau)}[\Bk]     \notag
 \\&
 = \left( \xi^{L, c_{34}}_l \ast \left[ y_l^{(\tau-1)} - \frac{ \lambda_{4l}^{(\tau-1)} }{ \beta_4 } \right] \right) [\Bk]
 + \frac{ \lambda_{4l}^{(\tau-1)}[\Bk] }{\beta_4}                            \notag
 \\&
 + \left( \theta^{L, c_{34}}_l \ast \left[  \text{Shrink} \left( \sum_{l'=0}^{L-1} \check{\tilde \theta}^L_{l'} \ast t_{l'}^{(\tau-1)} - \frac{\lambda_3^{(\tau-1)}}{\beta_3} \,,~ \frac{1}{\beta_3} \right) + \frac{ \lambda_3^{(\tau-1)} }{ \beta_3 }
         - \sum_{l'=[0, L-1] \backslash \{l\}} \check{\tilde \theta}^L_{l'} \ast  t_{l'}^{(\tau-1)} \right]
   \right) [\Bk]                                                               \notag
 \\& 
 + \frac{\lambda_1^{(\tau-1)}[\Bk] + \beta_1}{\beta_4} 
   \Shrink \left( \partial_l^+ u^{(\tau-1)}[\Bk] - \frac{ \lambda_{2 l}^{(\tau-1)}[\Bk] }{ \beta_2 } 
                  + \frac{ \lambda_1^{(\tau-1)}[\Bk] + \beta_1 }{ \beta_2 } y_l^{(\tau-1)}[\Bk]
                  \,, \frac{ \lambda_1^{(\tau-1)}[\Bk] + \beta_1}{\beta_2} \right) \,,
\end{align}
where the definition of $\xi^{L, c_{34}}_l(\cdot) \,, \theta^{L, c_{34}}_l(\cdot)$ and $\tilde\theta^{L, c_{34}}_l(\cdot)$
are well defined in (\ref{eq:FilterBanks:XiThetaTheta_tilde:1})-(\ref{eq:FilterBanks:XiThetaTheta_tilde:unitycondition}).


{\bfseries c. The ``$\Bu$-problem'':}

From a solution of the ``$\Bu$-problem (\ref{eq:problem:u}), we have
\begin{align*} 
 \cX(\Bz) &~\stackrel{\cF^{-1}}{\longleftrightarrow}~ 
 \beta_5 \big( h \ast \check{h} \big)[\Bk] -\beta_2 \sum_{l=0}^{L-1} \partial_l^- \partial_l^+ \delta[\Bk]
 \\
 \cY^{(\tau)}(\Bz) &~\stackrel{\cF^{-1}}{\longleftrightarrow}~  
 \beta_5 \Big( \check{h} \ast \Big[ f - h \ast v^{(\tau-1)} - h \ast \rho^{(\tau-1)} - \epsilon^{(\tau-1)} + \frac{\lambda_5^{(\tau-1)}}{\beta_5} \Big] \Big) [\Bk]
 \\&
 - \beta_2 \sum_{l=0}^{L-1} \partial_l^- 
 \Big[ \underbrace{ \Shrink \Big( \partial_l^+ u^{(\tau-1)}[\Bk] - \frac{ \lambda_{2 l}^{(\tau-1)} [\Bk] }{ \beta_2 } + \frac{ \lambda_1^{(\tau-1)}[\Bk] + \beta_1 }{ \beta_2 } y_l^{(\tau-1)} [\Bk]
                     \,,~ \frac{ \lambda_1^{(\tau-1)}[\Bk] + \beta_1}{\beta_2} \Big) 
                  }_{ = r_l^{(\tau)}[\Bk] }
       + \frac{ \lambda_{2 l}^{(\tau-1)}[\Bk] }{\beta_2} \Big]
\end{align*}
To avoid singularity, we check $\cX(\Bz)$ at $\Bome = \mathbf 0$ as
\begin{align*}
 \cX(e^{j\mathbf 0}) &= \beta_2 \sum_{l=0}^{L-1} \abs{ \cos\left(\frac{\pi l}{L}\right)(e^{j0} - 1) + \sin\left(\frac{\pi l}{L}\right)(e^{j0} - 1) }^2 + \beta_5 \abs{H(e^{j \mathbf 0})}^2
 = \beta_5 \abs{H(e^{j \mathbf 0})}^2 \,.
\end{align*}
If a blur operator has $\abs{H(e^{j \mathbf 0})}^2 > 0$ (usually $H(e^{j \mathbf 0}) = 1$ because $H(e^{j\Bome})$ often plays as a lowpass kernel in the Fourier domain),
then the function $\cX(\Bz)$ satisfies 
\begin{align*}
 0 < \beta_5 \abs{H(e^{j \mathbf 0})}^2 \leq \cX(\Bz) < +\infty \,,~ \forall \Bome \in [-\pi \,, \pi]^2 \,.
\end{align*}
Note that this is similar to a condition of Riesz basis.

\noindent
Thus, a solution of the ``$\Bu$-problem'' is rewritten as
\begin{align} \label{eq:VCPyramid:problem:u}
 &U^{(\tau)}(\Bz) = \cX^{-1}(\Bz) \cY^{(\tau)}(\Bz)    \notag
 \\&
 = \Phi^{L, c_{25}}(\Bz) H(\Bz^{-1}) \left[ F(\Bz) - H(\Bz) V^{(\tau-1)}(\Bz) - H(\Bz) P^{(\tau-1)}(\Bz) - \cE^{(\tau-1)}(\Bz) + \frac{\Lambda_5^{(\tau-1)}(\Bz)}{\beta_5} \right]   \notag
 \\& \quad
 + \sum_{l=0}^{L-1} \tilde\Psi_l^{L, c_{25}}(\Bz^{-1}) \left[ R_l^{(\tau)}(\Bz) + \frac{ \Lambda_{2 l}^{(\tau-1)}(\Bz) }{\beta_2} \right] \,.
\end{align}
Given $\Bk \in \Omega$, the inverse Fourier transform of (\ref{eq:VCPyramid:problem:u}) is
\begin{align} \label{eq:VCPyramid:problem:u:spatial}
 &u^{(\tau)}[\Bk] = \Big( \phi^{L, c_{25}} \ast \check{h} \ast \big( f - h \ast v^{(\tau-1)} - h \ast \rho^{(\tau-1)} - \epsilon^{(\tau-1)} + \frac{\lambda_5^{(\tau-1)}}{\beta_5} \big) \Big)[\Bk]   \notag
 \\&
 + \sum_{l=0}^{L-1} \check{\tilde{\psi}}^{L, c_{25}}_l [\Bk] \ast 
 \bigg[ \Shrink \Big( \big(\psi^L_l \ast u^{(\tau-1)}\big)[\Bk] - \frac{ \lambda_{2 l}^{(\tau-1)} [\Bk] }{ \beta_2 } + \frac{ \lambda_1^{(\tau-1)}[\Bk] + \beta_1 }{ \beta_2 } y_l^{(\tau-1)} [\Bk]
                     \,,~ \frac{ \lambda_1^{(\tau-1)}[\Bk] + \beta_1}{\beta_2} \Big)       
       + \frac{ \lambda_{2 l}^{(\tau-1)}[\Bk] }{\beta_2} \bigg] .
\end{align}
At a direction $l = 0, \ldots, L-1$, we have 
$y_l^{(\tau-1)} [\Bk] = \mathbb P_{[-1,1]} \Big[ \vec{y}\prime^{(\tau-1)}[\Bk] \Big]$
with $\vec{y}\prime^{(\tau-1)}[\Bk] = \Big[ y_l^{\prime(\tau-1)}[\Bk] \Big]_{l=0}^{L-1}$ 
as defined in (\ref{eq:VCPyramid:problem:y}) at iteration $(\tau -1)$.

\noindent
Note that $u^{(\tau)}$ in (\ref{eq:VCPyramid:problem:u:spatial}) is updated from 
$u^{(\tau-1)}$ and $u^{(\tau-2)}$ at every iteration $\tau$.
Frames (at direction $l = 0, \ldots, L-1$) are well defined in the Fourier domain
\begin{align}  \label{eq:VCPyramid:problem:frameElements:u:scalingfunc}
 \Phi^{L, c_{25}}(\Bz) &= \frac{1}{\displaystyle c_{25} \sum_{l=0}^{L-1} \abs{ \cos\left(\frac{\pi l}{L}\right)(z_2 - 1) + \sin\left(\frac{\pi l}{L}\right)(z_1 - 1) }^2 + \abs{H(\Bz)}^2}  
 ~\stackrel{\cF^{-1}}{\longleftrightarrow}~
 \phi^{L, c_{25}}[\Bk] \,,
 \\  \label{eq:VCPyramid:problem:frameElements:u:dualwavelet}
 \tilde{\Psi}^{L, c_{25}}_l(\Bz^{-1}) 
 &= \frac{\displaystyle c_{25} \left[ \cos\left(\frac{\pi l}{L}\right)(z_2^{-1} - 1) + \sin\left(\frac{\pi l}{L}\right)(z_1^{-1} - 1) \right] }
         {\displaystyle c_{25} \sum_{l'=0}^{L-1} \abs{ \cos\left(\frac{\pi l'}{L}\right)(z_2 - 1) + \sin\left(\frac{\pi l'}{L}\right)(z_1 - 1) }^2 + \abs{H(\Bz)}^2 }
 ~\stackrel{\cF^{-1}}{\longleftrightarrow}~ 
 \check{\tilde \psi}^{L, c_{25}}_l[\Bk] = - c_{25} \partial_l^- \phi^{L, c_{25}}[\Bk]  \,,
 \\  \label{eq:VCPyramid:problem:frameElements:u:wavelet}
 \Psi^L_l(\Bz) &= \cos\left(\frac{\pi l}{L}\right)(z_2 - 1) + \sin\left(\frac{\pi l}{L}\right)(z_1 - 1)
 ~\stackrel{\cF^{-1}}{\longleftrightarrow}~
 \psi^L_l[\Bk] = \partial_l^+ \delta[\Bk] \,.
\end{align}
Given an impulse response of a blur operator $H(\Bz)$,
these frames satisfy the unity conditions as
\begin{align*}
 & \abs{H(\Bz)}^2 \Phi^{L, c_{25}}(\Bz) + \sum_{l=0}^{L-1} \tilde \Psi^{L, c_{25}}_l(\Bz^{-1}) \Psi^L_l(\Bz) = 1 \,,
 \\
 &\Phi^{L, c_{25}}(e^{j \mathbf 0}) = \frac{1}{ \abs{H(e^{j \mathbf 0})}^2 }
 \quad \text{and} \quad
 \tilde{\Psi}^{L, c_{25}}_l(e^{j \mathbf 0}) = \Psi^L_l(e^{j \mathbf 0}) = 0 \,.
\end{align*}
The definition of $\Xi^{L, c_{34}}_l(\Bz) \,, \Theta^{L, c_{34}}_l(\Bz)$ and $\tilde\Theta^{L, c_{34}}_l(\Bz)$ (for variable $y_l^{(\tau - 1)}[\Bk]$ in (\ref{eq:VCPyramid:problem:u:spatial}))
are in (\ref{eq:FilterBanks:XiThetaTheta_tilde:1})-(\ref{eq:FilterBanks:XiThetaTheta_tilde:3}) which satisfy 
the unity condition (\ref{eq:FilterBanks:XiThetaTheta_tilde:unitycondition}) in the Fourier domain.
Figure \ref{fig:MDCD_let:FilterBanks:c0_1} and \ref{fig:MDCD_let:FilterBanks:c10} 
depict the spectra of these frames in the Fourier domain for different values of parameter
$c_{25} = c_{34} = 0.1$ and $10$, respectively.

\end{proof}

\begin{prop}  \label{prop:VCPyramid:problem:v:spatial:1}
 A solution of the "$\Bv$-problem" (\ref{eq:problem:v}) can be rewritten as in equation (\ref{eq:VCPyramid:problem:v:spatial:1}) with two shrinkage operators due to $\norm{\Bv}_{\ell_1}$ and $\norm{ \Bg_s }_{\ell_1}$ in a minimization problem (\ref{eq:minimization:SDMCDD:2}). 
\end{prop}

\begin{proof}

From a solution of the ``$\vec{\Bg}$-problem'' (\ref{eq:problem:g}), we have
\begin{align*}
 \cA_s(\Bz) &~\stackrel{\cF^{-1}}{\longleftrightarrow}~
 \beta_6 \delta[\Bk] - \beta_7 \partial_s^- \partial_s^+ \delta[\Bk]
 \\
 \cB_s^{(\tau)}(\Bz) &~\stackrel{\cF^{-1}}{\longleftrightarrow}~
 \beta_6 \Big[ \underbrace{\Shrink \Big( g_s^{(\tau-1)}[\Bk] - \frac{\lambda_{6 s}^{(\tau-1)}[\Bk]}{\beta_6} \,,~ \frac{\mu_1}{\beta_6} \Big)}_{= w_s^{(\tau)}[\Bk]} 
 + \frac{\lambda_{6 s}^{(\tau-1)}[\Bk]}{\beta_6} \Big]
 \\& \qquad
 -\beta_7 \partial_s^+ \Big[ \underbrace{ v^{(\tau-1)}[\Bk] - \sum_{s'=[0, S-1] \backslash \{s\}} \partial_{s'}^- g_{s'}^{(\tau-1)}[\Bk] }_{= \partial_s^- g_s^{(\tau-1)}[\Bk]}
 + \frac{\lambda_7^{(\tau-1)}[\Bk]}{\beta_7} \Big]          
\end{align*}
Due to $\cA_s(e^{j\mathbf 0}) = \beta_6 > 0$ and by choosing $\beta_7 = c_{67} \beta_6$, a solution of the $\vec{\Bg}$-problem 
(at a direction $s = 0, \ldots, S-1$) is well defined in the Fourier domain as
\begin{align}  \label{eq:FilterBanks:problem:u:Fourier}
 G_s^{(\tau)}(\Bz) &= \cA_s^{-1}(\Bz) \cB_s^{(\tau)}(\Bz)
 = \Xi^{S, c_{67}}_s(\Bz) \Big[ W_s^{(\tau)}(\Bz) + \frac{\Lambda_{6 s}^{(\tau-1)}(\Bz)}{\beta_6} \Big]
 + \Theta^{S, c_{67}}_s(\Bz) \Big[ \tilde \Theta^{S}_s(\Bz^{-1}) G_s^{(\tau-1)}(\Bz) + \frac{\Lambda_7^{(\tau-1)}(\Bz)}{\beta_7} \Big] \,.         
\end{align}
Its inverse Fourier transform is
\begin{align}  \label{eq:FilterBanks:problem:g:spatial}
 g_s^{(\tau)}[\Bk] &= \xi^{S, c_{67}}_s[\Bk] \ast \Big[ \Shrink \Big( g_s^{(\tau-1)}[\Bk] - \frac{\lambda_{6 s}^{(\tau-1)}[\Bk]}{\beta_6} \,,~ \frac{\mu_1}{\beta_6} \Big) + \frac{\lambda_{6 s}^{(\tau-1)}[\Bk]}{\beta_6} \Big]
 + \theta^{S, c_{67}}_s[\Bk] \ast \Big[ \check{\tilde \theta}^{S}_s[\Bk] \ast g_s^{(\tau-1)}[\Bk] + \frac{\lambda_7^{(\tau-1)}[\Bk]}{\beta_7} \Big] \,.         
\end{align}
Definition of frames  
$\xi^{S, c_{67}}_s(\cdot) \,, \theta^{S, c_{67}}_s(\cdot)$ and $\tilde \theta^S_s(\cdot)$ are well defined in
(\ref{eq:FilterBanks:XiThetaTheta_tilde:1:1})-(\ref{eq:FilterBanks:XiThetaTheta_tilde:unitycondition:1}).
And we finalize this proof by applying $\vec{\Bg}^{(\tau)}$ in equation (\ref{eq:FilterBanks:problem:g:spatial}) to a solution of the $\Bv$-problem (\ref{eq:problem:v}).

\end{proof}


\begin{prop} \label{prop:DMCDDMultiscaleSamplingTheory:u-problem}
 The discrete and continuous versions of the multiscale sampling theory for the "$\Bu$-problem" are defined in (\ref{label:multiScaleSamp:discreteprojection:prop}) and (\ref{label:multiScaleSamp:continuousprojection:prop}), respectively. 
\end{prop}

\begin{proof}
 
{\bfseries 1. Proof for the discrete multiscale sampling theory (\ref{label:multiScaleSamp:discreteprojection:prop}):}
 
In order to build a form of multiscale sampling version generated by a solution of the "$\Bu$-problem", for easy calculation we consider (\ref{eq:VCPyramid:problem:u:spatial:1}) without blur operator, i.e. $h(\cdot) = \delta(\cdot)$, and simplify its notation by
 \begin{itemize}
 \item replacing $\Bh \ast \Bu = \Bf - \Bh \ast \Bv - \Bh \ast \Brho - \Beps$ (due to a condition of a decomposition in (\ref{eq:minimization:SDMCDD:1}))
 \item removing shrinkage operator, the Lagrange multipliers $(\boldsymbol{\lambda_1} \,, \vec{\boldsymbol \lambda}_{\boldsymbol 2} \,, \boldsymbol{\lambda_5})$ and $\vec{\By}^{(\tau-1)}$ \,,
 \item denoting $c = c_{25} \,, \tilde \psi_l = \tilde \psi_l^{L,c} \,, \psi_l = \psi_l^L \text{ and interpolant } \phi_\text{int} = \phi^{L, c}$.
\end{itemize}
Thus, we rewrite (\ref{eq:VCPyramid:problem:u:spatial:1}) as
\begin{align} \label{eq:multiScaleSamp:noscale_blur}
 &u^{(\tau)}[\Bk] = \big( u^{(\tau-1)} \ast \phi_\text{int} \big)[\Bk]   
 + \sum_{l=0}^{L-1} \big( u^{(\tau-1)} \ast \check{\tilde \psi}_l \ast \psi_l \big)[\Bk] \,.
\end{align}
At a convergence, i.e. when iteration $\tau$ goes to infinity, we have $\Bu^{(\tau)} = \Bu^{(\tau-1)} = \Bu$ and 
(\ref{eq:multiScaleSamp:noscale_blur}) is rewritten as
\begin{align} \label{eq:multiScaleSamp:noscale}
 &u[\Bk] = \big( u \ast \phi_\text{int} \big)[\Bk] + \sum_{l=0}^{L-1} \big( u \ast \check{\tilde \psi}_l \ast \psi_l \big)[\Bk] 
 ~\stackrel{\cF}{\longleftrightarrow}~ 
 U(\Bz) = U(\Bz) \Phi_\text{int}(\Bz) + \sum_{l=0}^{L-1} U(\Bz) \tilde{\Psi}_l(\Bz^{-1}) \Psi_l(\Bz) \,. 
\end{align}
From (\ref{eq:VCPyramid:problem:frameElements:u:scalingfunc:1})-(\ref{eq:VCPyramid:problem:frameElements:u:wavelet:1}) without blur operator, i.e. $H(\Bz) = 1$,
and definition of the impulse response of a discrete directional Laplacian operator, 
the interpolant $\phi_\text{int}(\cdot)$ and directional primal/dual wavelet $\psi_l(\cdot) \,, \tilde \psi_l(\cdot)$ 
with $l = 0 \,, \ldots \,, L-1$ are
\begin{align}
 \label{eq:FilterBanks:XiThetaTheta_tilde:multiScaleSamp:1}
 \phi_\text{int}[\Bk] &= \left[ c (-\Delta_{\text{d}L}) + 1 \right]^{-1} \delta[\Bk]
 ~\stackrel{\cF}{\longleftrightarrow}~
 \Phi_\text{int}(\Bz) = \frac{1}{\displaystyle 1 + c \sum_{l=0}^{L-1} \abs{ \sin\left(\frac{\pi l}{L}\right)(z_1 - 1) + \cos\left(\frac{\pi l}{L}\right)(z_2 - 1) }^2} \,,
 \\
 \label{eq:FilterBanks:XiThetaTheta_tilde:multiScaleSamp:2}
 \check{\tilde \psi}_l[\Bk] &= \underbrace{ - c \Big[ c (-\Delta_{\text{d}L}) + 1 \Big]^{-1} \partial_l^- \delta[\Bk] }
                                 _{\displaystyle = -c \partial_l^- \phi_\text{int}[\Bk] }
 ~\stackrel{\cF}{\longleftrightarrow}~
 \tilde \Psi_l(\Bz^{-1}) = \frac{\displaystyle c \left[ \sin\left(\frac{\pi l}{L}\right) (z_1^{-1} - 1) + \cos\left(\frac{\pi l}{L}\right) (z_2^{-1} - 1) \right] }
                           {\displaystyle 1 + c \sum_{l=0}^{L-1} \abs{ \sin\left(\frac{\pi l}{L}\right)(z_1 - 1) + \cos\left(\frac{\pi l}{L}\right)(z_2 - 1) }^2 } \,,
 \\
 \label{eq:FilterBanks:XiThetaTheta_tilde:multiScaleSamp:3}
 \psi_l[\Bk] &= \partial^+_l \delta[\Bk] 
 ~\stackrel{\cF}{\longleftrightarrow}~
 \Psi_l(\Bz) = \sin\left(\frac{\pi l}{L}\right) (z_1 - 1) + \cos\left(\frac{\pi l}{L}\right) (z_2 - 1) \,.
\end{align}
Note that the directional mother dual/primal wavelet plays as an operator-like wavelet of interpolant as
\begin{align*}
 &\psi[\Bk] = c (-\Delta_{\text{d}L}) \phi_\text{int}[\Bk] 
 ~\stackrel{\cF}{\longleftrightarrow}~
 \Psi(\Bz) = \sum_{l=0}^{L-1} \Psi_l(\Bz) \tilde \Psi_l(\Bz^{-1})
\end{align*}
and these frame elements satisfy a perfect reconstruction scheme, i.e. 
\begin{align}  \label{eq:FilterBanks:XiThetaTheta_tilde:multiScaleSamp:unitycondition}
 \Phi_\text{int}(\Bz) + \sum_{l=0}^{L-1} \Psi_l(\Bz) \tilde \Psi_l(\Bz^{-1}) = 1 \,.
\end{align}
From (\ref{eq:multiScaleSamp:noscale}), we defined a multiscale projection of a function $\Bu \in X$ to scaling space and its orthogonal wavelet space, i.e.
$\mathcal V \left\{ \phi_\text{int} \right\} \perp \mathcal W_{il} \big\{ \psi_{il} \,, \tilde \psi_{il} \big\}$ with scale $i=0, \ldots I-1$ and direction $l = 0, \ldots, L-1$ 
as
\begin{align} \label{label:multiScaleSamp:projection}
 u[\Bk] &= \underbrace{ \sum_{\Bn \in \mathbb Z^2} u[\Bn] \phi[\Bk - \Bn] }_{= (u \ast \phi)[\Bk]}
 + \sum_{i=0}^{I-1} \sum_{l=0}^{L-1} 
 \underbrace{ \sum_{\Bn \in \mathbb Z^2} \big\langle u, \tilde \psi_{il}[\cdot - \Bn] \big\rangle_{\ell_2} \psi_{il}[\Bk - \Bn] }
           _{ = ( u \ast \check{ \tilde{\psi} }_{il} \ast \psi_{il} )[\Bk] }
 \\
 \stackrel{\cF}{\longleftrightarrow}~           
 U(\Bz) &= U(\Bz) \Phi(\Bz)
 + \sum_{i=0}^{I-1} \sum_{l=0}^{L-1} U(\Bz) \tilde \Psi_{il}(\Bz^{-1}) \Psi_{il}(\Bz)    \notag
\end{align}
where definition of frames on spaces $\mathcal V$ and $\mathcal W_{il}$ are defined in the Fourier domain
from (\ref{eq:FilterBanks:XiThetaTheta_tilde:multiScaleSamp:1})-(\ref{eq:FilterBanks:XiThetaTheta_tilde:multiScaleSamp:unitycondition}), 
 as (see Figure \ref{fig:MDCD_let:u-g-problems}(b) for their spectra)
\begin{align*}
 \phi[\Bk] 
 &~\stackrel{\cF}{\longleftrightarrow}~ 
 \Phi(\Bz) = I^{-1} \sum_{i=0}^{I-1} \Phi_\text{int}(\Bz^{a^i}) \,,
 \\
 \psi_{il}[\Bk] 
 &~\stackrel{\cF}{\longleftrightarrow}~ 
 \Psi_{il}(\Bz) = I^{-\frac{1}{2}} \Psi_l(\Bz^{a^i}) ~~ \text{and}
 \\
 \check{\tilde \psi}_{il}[\Bk] 
 &~\stackrel{\cF}{\longleftrightarrow}~ 
 \tilde \Psi_{il}(\Bz^{-1}) = I^{-\frac{1}{2}} \tilde \Psi_l(\Bz^{-a^i}) \,,
\end{align*}
which satisfy the unity condition as
\begin{align*}
 \Phi(\Bz) + \sum_{i=0}^{I-1} \sum_{l=0}^{L-1} \tilde \Psi_{il}(\Bz^{-1}) \Psi_{il}(\Bz) = 1 \,.  
\end{align*}
Given $\Bn \in \Omega$, wavelet coefficients in direction $l$ and scale $i$ is
\begin{align*}
 c_{il}[\Bn] = \big\langle u \,, \tilde \psi_{il}[\cdot - \Bn] \big\rangle_{\ell_2} 
 ~\stackrel{\cF}{\longleftrightarrow}~
 C_{il}(\Bz) = U(\Bz) \tilde \Psi_{il}(\Bz^{-1}) \,.
\end{align*}


{\bfseries 2. Proof for the continuous multiscale sampling theory (\ref{label:multiScaleSamp:continuousprojection:prop}):}

\noindent
 The impulse response of a continuous directional Laplacian operator is
 \begin{align*} 
  \Delta_L \delta(\Bx) = \sum_{l=0}^{L-1} \partial_l^2 \delta(\Bx)
  ~\stackrel{\cF}{\longleftrightarrow}~
  -\sum_{l=0}^{L-1} \left[ \cos\left( \frac{\pi l}{2} \right) \omega_2 + \sin \left( \frac{\pi l}{L} \right) \omega_1 \right]^2
 \end{align*}
 Note that the 1st order Maclaurin approximation is a link between a continuous and discrete version of the operator.
 Similar to a discrete domain (\ref{eq:FilterBanks:XiThetaTheta_tilde:multiScaleSamp:1})-(\ref{eq:FilterBanks:XiThetaTheta_tilde:multiScaleSamp:unitycondition}),  
 the continuous version of the interpolant $\phi_\text{int}(\cdot)$ and the directional dual/primal wavelet
 $\tilde \psi_l(\cdot)$ and $\psi_l(\cdot)$ (with $l = 0, \ldots, L-1$) are
 \begin{align*}
  \phi_\text{int}(\Bx) &= \left[ c (-\Delta_{L}) + 1 \right]^{-1} \delta(\Bx)
  ~\stackrel{\cF}{\longleftrightarrow}~
  \widehat{\phi}_\text{int}(\Bome) = \frac{1}{\displaystyle 1 + c \sum_{l=0}^{L-1} \left[ \cos\left(\frac{\pi l}{L}\right) \omega_2 + \sin\left(\frac{\pi l}{L}\right) \omega_1 \right]^2} \,,
  \\
  \check{\tilde \psi}_l(\Bx) &= -c \partial_l \phi_\text{int}(\Bx)
  ~\stackrel{\cF}{\longleftrightarrow}~
  \widehat{\tilde \psi^*_l}(\Bome) = \frac{\displaystyle -c \left[ \cos\left(\frac{\pi l}{L}\right) j \omega_2 + \sin\left(\frac{\pi l}{L}\right) j \omega_1 \right] }
                            {\displaystyle 1 + c \sum_{l=0}^{L-1} \left[ \cos\left(\frac{\pi l}{L}\right) \omega_2 + \sin\left(\frac{\pi l}{L}\right) \omega_1 \right]^2} \,,
  \\
  \psi_l(\Bx) &= \partial_l \delta(\Bx)
  ~\stackrel{\cF}{\longleftrightarrow}~
  \widehat{\psi}_l(\Bome) = \cos\left(\frac{\pi l}{L}\right) j\omega_2 + \sin\left(\frac{\pi l}{L}\right) j\omega_1 \,.
 \end{align*}
 Similarly, this directional wavelet also plays as an operator-like wavelet of the interpolant 
 \cite{UnserVandeville2010, UnserSageVandeville2009, KhalidovUnser2006} as
 \begin{align*}
  &\psi(\Bx) = c (-\Delta_{L}) \phi_\text{int}(\Bx)
  ~\stackrel{\cF}{\longleftrightarrow}~
  \widehat{\psi}(\Bome) = \sum_{l=0}^{L-1} \widehat{\psi}_l(\Bome) \widehat{\tilde \psi^*_l}(\Bome) \,.
 \end{align*}
 Note that we have a pair of the continuous Fourier transform as
$\abs{c}^{-1} \phi(c^{-1} \Bx) \stackrel{\cF}{\longleftrightarrow} \widehat \phi(c \, \Bome)$ with a constant $c$ and $\Bx \in \mathbb R^2$.
 Thus, definition of frames is (see Figure \ref{fig:MDCD_let:u-g-problems}(a) for their spectrum)
 \begin{align*}
  \phi(\Bx) &= I^{-1} \sum_{i=0}^{I-1} a^{-i} \phi_\text{int} (a^{-1} \Bx)
  ~\stackrel{\cF}{\longleftrightarrow}~ 
  \widehat{\phi}(\Bome) = I^{-1} \sum_{i=0}^{I-1} \widehat{\phi}_\text{int}(a^i \Bome) \,,
  \\
  \psi_{il}(\Bx) &= I^{-\frac{1}{2}} a^{-i} \psi_l(a^{-i} \Bx)
  ~\stackrel{\cF}{\longleftrightarrow}~ 
  \widehat{\psi}_{il}(\Bome) = I^{-\frac{1}{2}} \widehat{\psi}_l(a^i\Bome) ~~ \text{and}
  \\
  \check{\tilde \psi}_{il}(\Bx) &= I^{-\frac{1}{2}} a^{-i} \check{\tilde\psi}_l(a^{-i} \Bx) 
  ~\stackrel{\cF}{\longleftrightarrow}~ 
  \widehat{\tilde \psi^*_{il}}(\Bome) = I^{-\frac{1}{2}} \widehat{\tilde \psi^*_l}(a^i\Bome) \,,
 \end{align*}
 and they satisfy the unity condition
 \begin{align*}
  \widehat{\phi}(\Bome) + \sum_{i=0}^{I-1} \sum_{l=0}^{L-1} \widehat{\tilde \psi^*_{il}}(\Bome) \widehat{\psi}_{il}(\Bome) = 1 \,, \widehat{\phi}(0) = 1 \,, \widehat{\psi}_{il}(0) = \widehat{\tilde \psi}_{il}(0) = 0 \,.
 \end{align*}
 A multiscale decomposition in a continuous setting for $\Bu \in X$ is
 \begin{align*}
  u[\Bk] &= \underbrace{ \sum_{\Bn \in \mathbb Z^2} u[\Bn] \phi(\Bk - \Bn) }_{= \left( u \ast \phi \right)[\Bk]}
  + \sum_{i=0}^{I-1} \sum_{l=0}^{L-1} 
  \underbrace{ \sum_{\Bn \in \mathbb Z^2} \left\langle u \,, \tilde \psi_{il} (\cdot - \Bn) \right\rangle_{\ell_2} \psi_{il}(\Bk - \Bn) }_{= \left( u \ast \check{\tilde \psi}_{il} \ast \psi_{il} \right)[\Bk]}
  \\ \stackrel{\cF}{\longleftrightarrow}~
  U(e^{j\Bome}) &= U(e^{j\Bome}) \widehat{\phi}(\Bome) + \sum_{i=0}^{I-1} \sum_{l=0}^{L-1} U(e^{j\Bome}) \widehat{\tilde \psi}_{il}^* (\Bome) \widehat{\psi}_{il}(\Bome) \,.
 \end{align*}
 
\end{proof}


\begin{prop} \label{prop:DMCDDMultiscaleSamplingTheory:g-problem}
 The continuous and discrete versions of the multiscale sampling theory for the "$\vec{\Bg}$-problem" are defined in (\ref{label:multiScaleSamp:discreteprojection:gproblem:prop}) and 
(\ref{label:multiScaleSamp:continuousprojection:gproblem:prop}). 
\end{prop}

\begin{proof}

{\bfseries 1. Proof for the discrete multiscale sampling theory (\ref{label:multiScaleSamp:discreteprojection:gproblem:prop}):}

\noindent
To ease the calculation, we denote $c = c_{67} \,, \xi_s = \xi^{S, c_{67}}_s \,, \theta_s = \theta^{S, c_{67}}_s \,, \tilde \theta_s = \tilde \theta^{S}_s$.
Given $s = 0 \,, \ldots \,, S-1$, a simplified version of (\ref{eq:FilterBanks:problem:g:spatial:1}) in the $\vec{\Bg}$-problem is obtained by removing a shrinkage operator and Lagrange multipliers $(\vec{\boldsymbol \lambda}_{\boldsymbol 6} \,, \boldsymbol{\lambda_7})$ as
\begin{align*}  
 g_s^{(\tau)}[\Bk] &= \left( \xi_s \ast g_s^{(\tau-1)} \right) [\Bk]
 + \left( \theta_s \ast \check{\tilde \theta}_s \ast g_s^{(\tau-1)} \right) [\Bk] \,.
\end{align*}
Denote $\Xi_s(\Bz) \,, \Theta_s(\Bz)$ and $\tilde \Theta_s(\Bz)$ as discrete Fourier transform of $\xi_s[\Bk] \,, \theta_s[\Bk]$ and $\tilde \theta_s[\Bk]$, respectively.
At convergence when the iteration $\tau$ goes to infinity, i.e. $\Bg_s^{(\tau)} = \Bg_s^{(\tau-1)} = \Bg_s$, we have
\begin{align*}  
 g_s[\Bk] &= \left( \xi_s \ast g_s \right) [\Bk]
 + \left( \theta_s \ast \check{\tilde \theta}_s \ast g_s \right) [\Bk]
 ~\stackrel{\cF}{\longleftrightarrow}~
 G_s(\Bz) = \Xi_s(\Bz) G_s(\Bz) + \Theta_s(\Bz) \tilde \Theta_s(\Bz^{-1}) G_s(\Bz) 
\end{align*}
and the unity condition is
\begin{align*}
 \Xi_s(\Bz) + \Theta_s(\Bz) \tilde \Theta_s(\Bz^{-1}) = 1 \,.
\end{align*}
Given $\Bf \in X$ whose discrete Fourier transform is $F(\Bz)$,
a multiscale sampling theory in the Fourier domain is described as
\begin{align*}
 F(\Bz) &= 
 \frac{1}{SI} \sum_{i=0}^{I-1} \sum_{s=0}^{S-1} F(\Bz) \Xi_s(\Bz^{a^i}) +  
 \frac{1}{SI} \sum_{i=0}^{I-1} \sum_{s=0}^{S-1} F(\Bz) \tilde\Theta_s(\Bz^{-a^i}) \Theta_s(\Bz^{a^i})  
 \\
 &= 
 F(\Bz) \Xi(\Bz) +  
 \sum_{i=0}^{I-1} \sum_{s=0}^{S-1} F(\Bz) \tilde\Theta_{si}(\Bz^{-1}) \Theta_{si}(\Bz)
 \\
 \stackrel{\cF}{\longleftrightarrow}~
 f[\Bk] &= \underbrace{ \sum_{\boldsymbol n \in \mathbb Z^2} f[\boldsymbol n] \xi[\Bk - \boldsymbol n] }_{ = (\xi \ast f)[\Bk] }
 + \sum_{i=0}^{I-1} \sum_{s=0}^{S-1} 
 \underbrace{
 \sum_{\boldsymbol n \in \mathbb Z^2} \left\langle \Bf \,, \tilde{\theta}_{si}(\cdot - \boldsymbol n) \right\rangle_{\ell_2} 
 \theta_{si}[\Bk - \boldsymbol n]
 }_{ = ( \check{\tilde \theta}_{si} \ast \theta_{si} \ast f )[\Bk] } \,,
\end{align*}
where the frames (see Figure \ref{fig:MDCD_let:u-g-problems}(d) for their spectra)
\begin{align*}
 \xi[\Bk] \stackrel{\cF}{\longleftrightarrow}
 \Xi(\Bz) = \frac{1}{SI} \sum_{i=0}^{I-1} \sum_{s=0}^{S-1} \Xi_s(\Bz^{a^i}) \,,~
 \theta_{si}[\Bk] \stackrel{\cF}{\longleftrightarrow}
 \Theta_{si}(\Bz) = \frac{1}{\sqrt{SI}} \Theta_s(\Bz^{a^i})  \text{ and }
 \tilde\theta_{si}[\Bk] \stackrel{\cF}{\longleftrightarrow}
 \tilde\Theta_{si}(\Bz) = \frac{1}{\sqrt{SI}} \tilde\Theta_s(\Bz^{a^i}) \,.
\end{align*}
satisfy the unity condition
\begin{align*}
 \Xi(\Bz) + \sum_{i=0}^{I-1} \sum_{s=0}^{S-1} \tilde\Theta_{si}(\Bz^{-1}) \Theta_{si}(\Bz) = 1 \,,~
 \Xi(e^{j\mathbf 0}) = 1 
 ~~\text{and}~~
 \Theta_{si}(e^{j\mathbf 0}) = \tilde\Theta_{si}(e^{j\mathbf 0}) = 0 \,.
\end{align*}
and $\Xi_s(\Bz) = \Xi_s^{S, c_{67}}(\Bz) \,, \Theta_s(\Bz) = \Theta_s^{S, c_{67}}(\Bz)$ and $\tilde\Theta_s(\Bz) = \tilde \Theta_s^S(\Bz)$ are defined in (\ref{eq:FilterBanks:XiThetaTheta_tilde:1:1})-(\ref{eq:FilterBanks:XiThetaTheta_tilde:3:1}).

{\bfseries 2. Proof for the continuous multiscale sampling theory (\ref{label:multiScaleSamp:continuousprojection:gproblem:prop}):}

\noindent
Similar to the $\Bu$-problem and by some simplified notations as in the previous explanation, we derive a continuous version for multiscale sampling theory in the $\vec{\Bg}$-problem by considering continuous operators in (\ref{eq:FilterBanks:XiThetaTheta_tilde:1:1})-(\ref{eq:FilterBanks:XiThetaTheta_tilde:3:1}) as
\begin{align*} 
 \xi_s(\Bx) = \left[ 1 - c \partial_s^2 \right]^{-1} \delta(\Bx) 
 &~\stackrel{\cF}{\longleftrightarrow}~  
 \widehat{\xi_s}(\Bome) = \frac{ 1 }{ 1 + c \left[ \cos \left( \frac{\pi s}{S} \right) \omega_2 + \sin \left( \frac{\pi s}{S} \right) \omega_1 \right]^2 }  \,,
 \\ 
 \theta_s(\Bx) = -c \partial_s \xi_s(\Bx) 
 &~\stackrel{\cF}{\longleftrightarrow}~  
 \widehat{\theta_s}(\Bome) = 
 \frac{ -c \left[ \cos\left(\frac{\pi s}{S}\right) j \omega_2 + \sin\left(\frac{\pi s}{S}\right) j \omega_1 \right] }
        { 1 + c \left[ \cos \left( \frac{\pi s}{S} \right) \omega_2 + \sin \left( \frac{\pi s}{S} \right) \omega_1 \right]^2 } \,,
 \\ 
 \check{\tilde \theta}_s(\Bx) = \partial_s \delta(\Bx)
 &~\stackrel{\cF}{\longleftrightarrow}~
 \widehat{\tilde \theta_s^*}(\Bome) = \cos \left( \frac{\pi s}{S} \right) j \omega_2 + \sin \left( \frac{\pi s}{S} \right) j \omega_1 \,.
\end{align*}
These bounded frames satisfy the unity condition in the Fourier domain as ($s = 0 , \ldots, S-1$)
\begin{align*}
 \widehat{\xi_s}(\Bome) + \widehat{\theta_s}(\Bome) \widehat{\tilde \theta_s^*}(\Bome) = 1 \,,  
 \widehat{\xi_s}(0) = 1 \,, \widehat{\theta_s}(0) = \widehat{\tilde \theta_s}(0) = 0 \,.
\end{align*}
Given a constant $a > 0$ and a discrete function $f \in \ell_2 (\mathbb R^2)$ whose discrete Fourier transform is $F(e^{j\Bome})$, a multiscale sampling theory in the continuous setting is defined as 
\begin{align*}
 F(e^{j\Bome}) &= \frac{1}{SI} \sum_{i=0}^{I-1} \sum_{s=0}^{S-1} F(e^{j\Bome}) \widehat \xi_s(a^i \Bome) 
 + \frac{1}{SI} \sum_{i=0}^{I-1} \sum_{s=0}^{S-1} F(e^{j\Bome}) \widehat{\tilde \theta_s^*}(a^i \Bome) \widehat{\theta_s}(a^i \Bome) 
\end{align*}
Thus, we reformulate the above formulae as
\begin{align*}
 f[\Bk] &= \underbrace{ \sum_{\Bn \in \mathbb Z^2} f[\Bn] \xi(\Bk - \Bn) }_{ = (f \ast \xi)[\Bk] }
 + \sum_{i=0}^{I-1} \sum_{s=0}^{S-1} 
 \underbrace{
 \sum_{\Bn \in \mathbb Z^2}
 \left\langle f \,, \tilde \theta_{si}(\cdot - \Bn) \right\rangle_{\ell_2} \theta_{si}(\Bk - \Bn) 
 }_{= (f \ast \check{\tilde \theta}_{si} \ast \theta_{si})[\Bk]} \,.
 \\
 \stackrel{\cF}{\longleftrightarrow}~
 F(e^{j\Bome}) &= F(e^{j\Bome}) \widehat \xi(\Bome) + \sum_{i=0}^{I-1} \sum_{s=0}^{S-1} F(e^{j\Bome}) \widehat{\tilde \theta}_{si}^*(\Bome) \widehat{\theta}_{si}(\Bome) 
\end{align*}
with frames (see Figure \ref{fig:MDCD_let:u-g-problems}(c) for their spectra)
\begin{align*}
 \xi(\Bx) = \frac{1}{SI} \sum_{i=0}^{I-1} \sum_{s=0}^{S-1} a^{-i} \xi_s(a^{-i} \Bx)
 &~\stackrel{\cF}{\longleftrightarrow}~
 \widehat \xi(\Bome) = \frac{1}{SI} \sum_{i=0}^{I-1} \sum_{s=0}^{S-1} \widehat \xi_s(a^i \Bome)
 \\
 \theta_{si}(\Bx) = \frac{1}{\sqrt{SI}} a^{-i} \theta_s(a^{-i} \Bx)
 &~\stackrel{\cF}{\longleftrightarrow}~
 \widehat \theta_{si}(\Bome) = \frac{1}{\sqrt{SI}} \widehat \theta_s(a^i \Bome)
 \\
 \tilde \theta_{si}(\Bx) = \frac{1}{\sqrt{SI}} a^{-i} \tilde \theta_s(a^{-i} \Bx)
 &~\stackrel{\cF}{\longleftrightarrow}~
 \widehat{\tilde \theta}_{si}(\Bome) = \frac{1}{\sqrt{SI}} \widehat{\tilde \theta}_s(a^i \Bome) \,.
\end{align*}
Note that these multiscale frames are also bounded and satisfy the unity condition in the Fourier domain as
\begin{align*}
 \widehat \xi(\Bome) + \sum_{i=0}^{I-1} \sum_{s=0}^{S-1} \widehat{\tilde \theta_{si}^*} (\Bome) \widehat{\theta_{si}} (\Bome) = 1 \,, 
 \widehat \xi(0) = 1 
 ~~\text{and}~~
 \widehat{\tilde \theta}_{si}(0) = \widehat{\tilde \theta}_{si}(0) = 0 \,.
\end{align*}

\end{proof}


\section*{Appendix B. Figures} \label{sec:appendixB:Figures}

\setcounter{subfigure}{0}
\begin{figure*}
\begin{center}
 
 \includegraphics[width=1\textwidth]{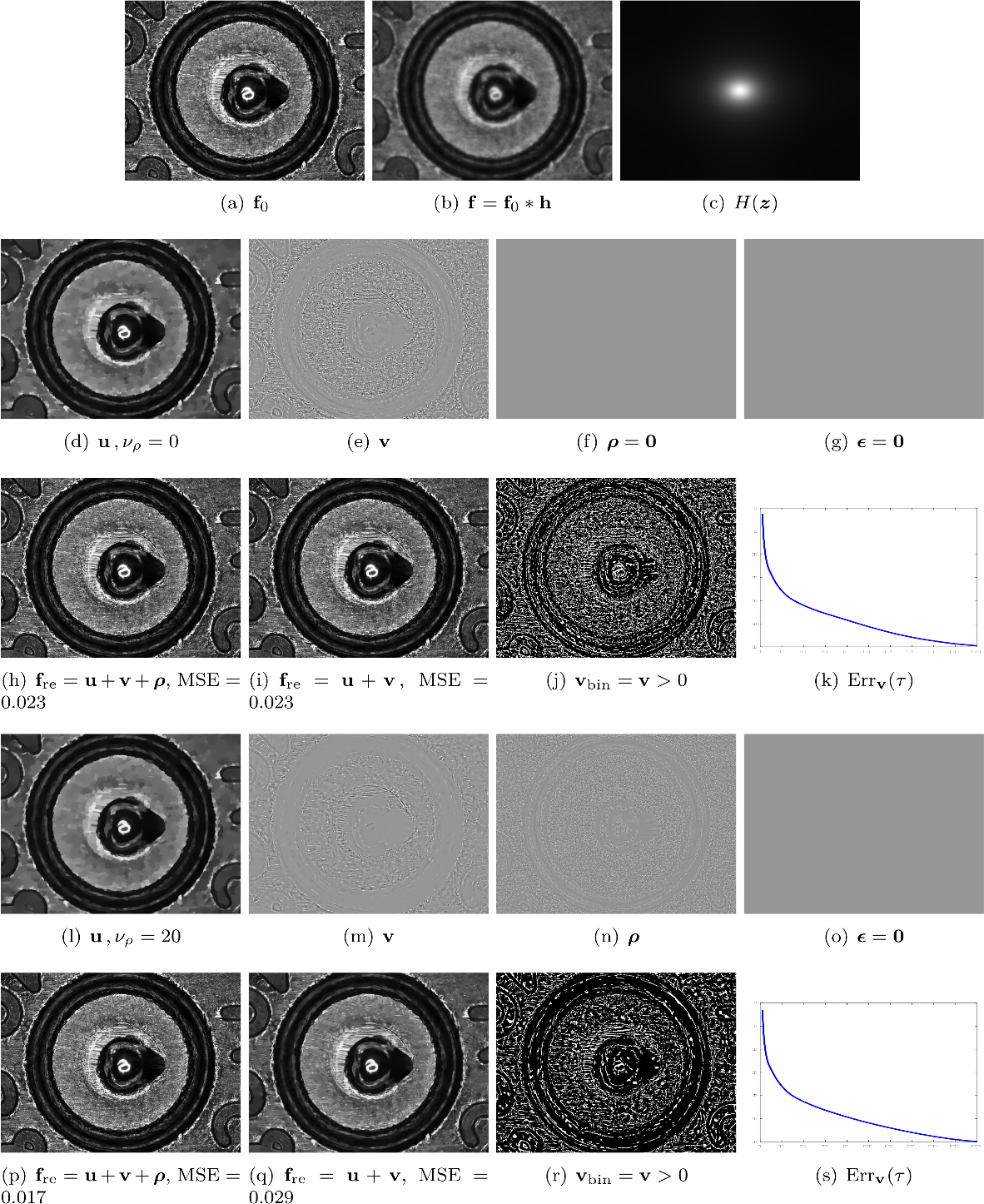}
 
 \caption{The (2nd-3rd) and (4th-5th) rows are reconstructed image of a ballistic image by the DMCD model
          with $\nu_\rho = 0$ and $\nu_\rho = 20$, respectively.
          The other parameters are the same as in Figure \ref{fig:deblurdecomp:barbara:nonoise}.
          The convergences of DMCD (k) and (s) are in log scale.
          A parameter for blur operator $h$ is $L_\text{blur} = 20$.
          \label{fig:deblurdecomp:ballistic}
         }
\end{center}
\end{figure*}

\setcounter{subfigure}{0}
\begin{figure*}
\begin{center}
 
 \includegraphics[width=1\textwidth]{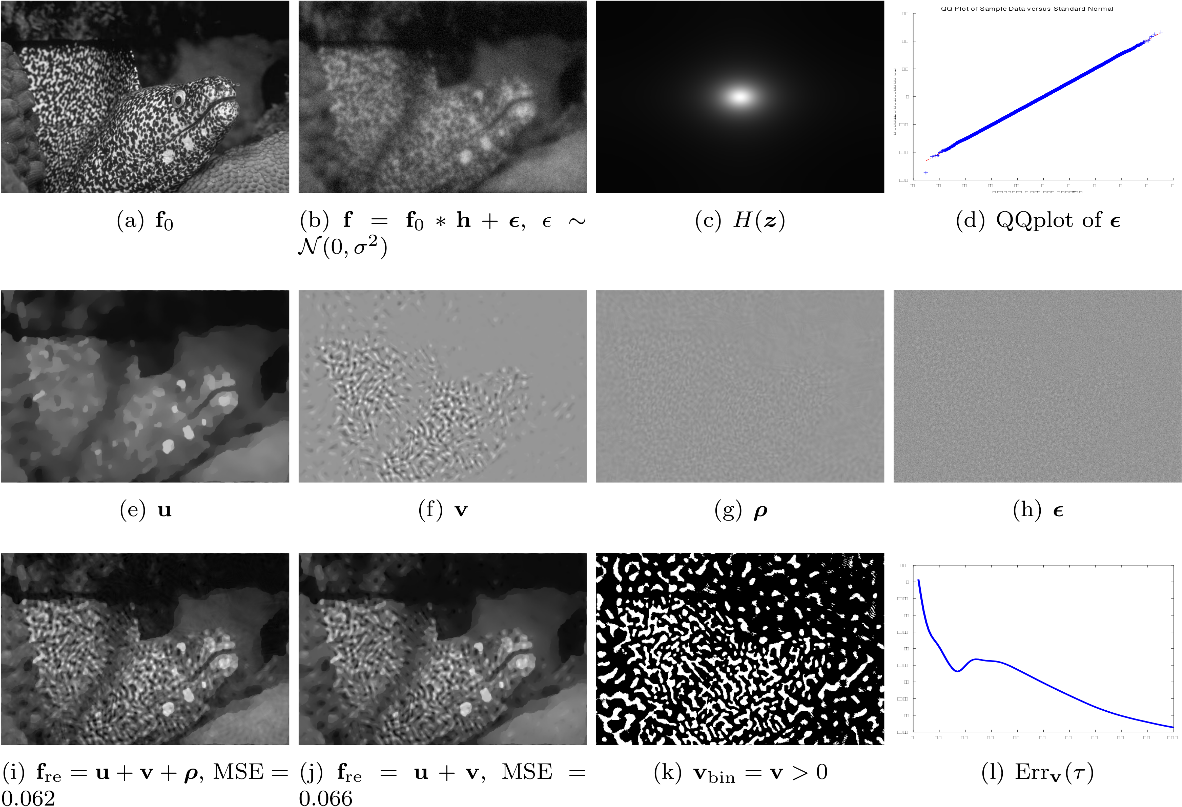} 
 
 \caption{Visualization of a reconstructed image by DMCD with 
          $\sigma = 10 \,, L_\text{blur} = 20 \,, \nu_\epsilon = 10 \,, \nu_\rho = 6.8 \,, L = S = 11$ and $\mu_2 = 2.5 \times 10^{10}$.
          The other parameters are the same as as in Figure \ref{fig:deblurdecomp:barbara:nonoise}.
          The convergence of DMCD is in log scale.
          \label{fig:deblurdecomp:fish}
         }
\end{center}
\end{figure*}

\setcounter{subfigure}{0}
\begin{figure*}
\begin{center}

 \includegraphics[width=1\textwidth]{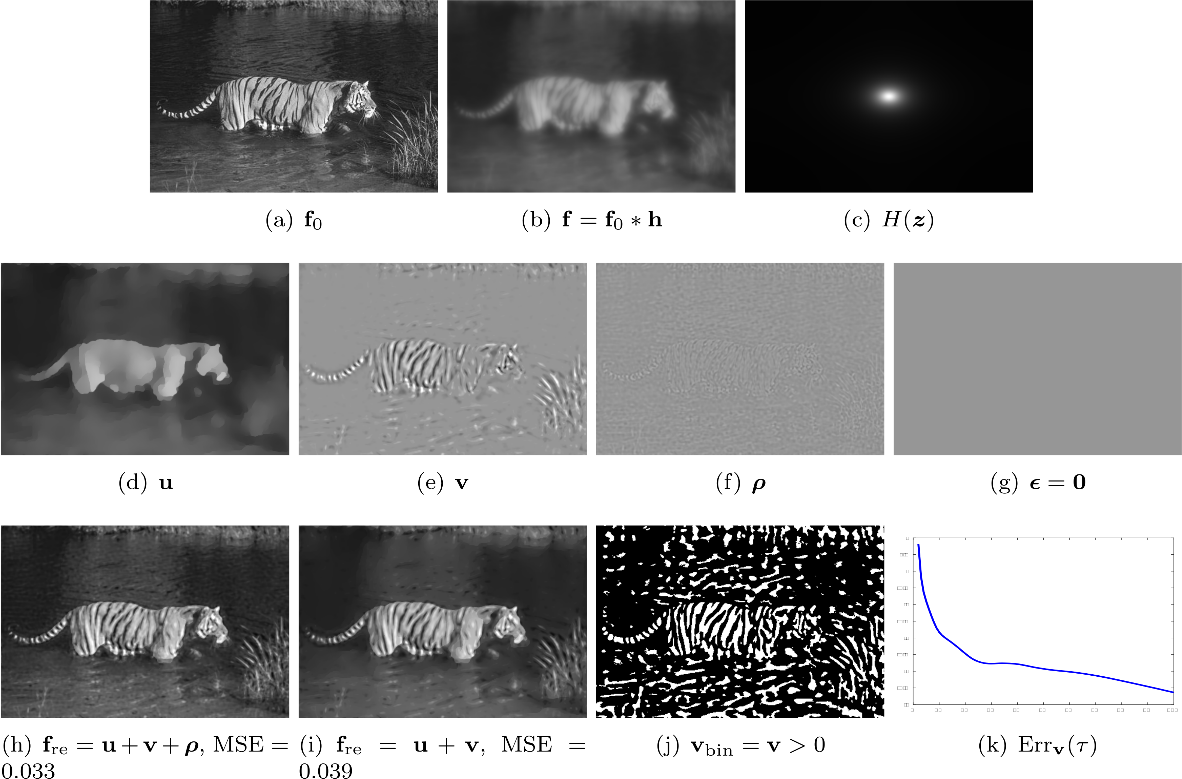}
 
 \caption{Visualization of a reconstructed image by DMCD with 
          $L_\text{blur} = 50 \,, \nu_\epsilon = 10 \,, \nu_\rho = 0 \,, L = S = 16$ and $\mu_2 = 4 \times 10^{10}$.
          The other parameters are the same as as in Figure \ref{fig:deblurdecomp:barbara:nonoise}.
          The convergence of DMCD is in log scale.
          \label{fig:deblurdecomp:tiger}
         }
\end{center}
\end{figure*}


\begin{figure*}

 \includegraphics[width=1\textwidth]{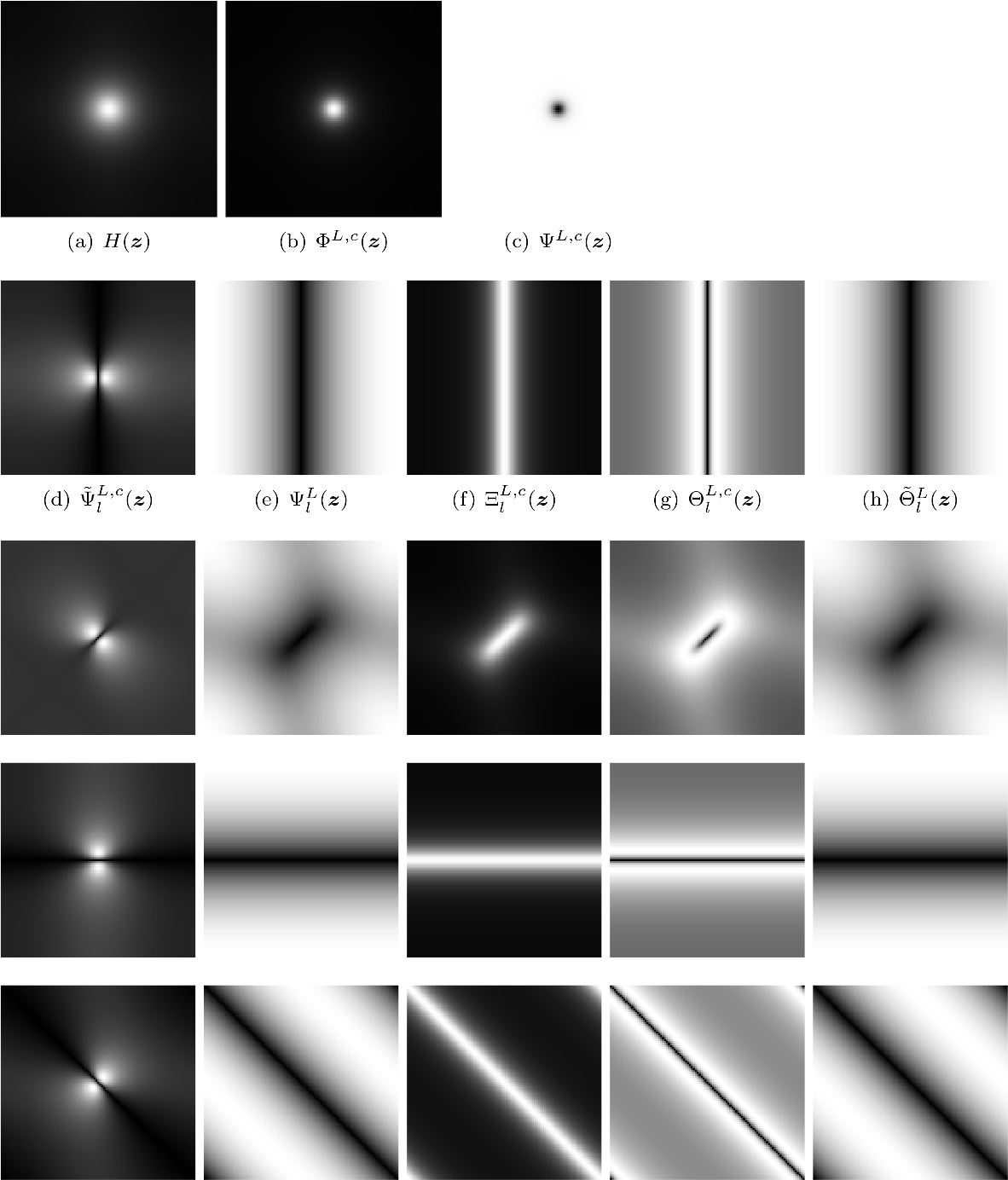}
  
 \caption{\label{fig:MDCD_let:FilterBanks:c10}
 This figure visualizes filter banks produced by the DMCD model (\ref{eq:VCPyramid:problem:frameElements:u:scalingfunc:1})-(\ref{eq:FilterBanks:XiThetaTheta_tilde:3:1}) with parameters $L_\text{blur} = 10 \,, L = S = 4 \,, c_{25} = c_{34} = c = 10$.
 A total wavelet function (c) is defined as $\displaystyle \Psi^{L,c}(\Bz) = \sum_{l=0}^{L-1} \tilde \Psi_l^{L,c}(\Bz^{-1}) \Psi_l^{L}(\Bz)$.
  }
 
\end{figure*}

\begin{figure*}    
 
 \includegraphics[width=1\textwidth]{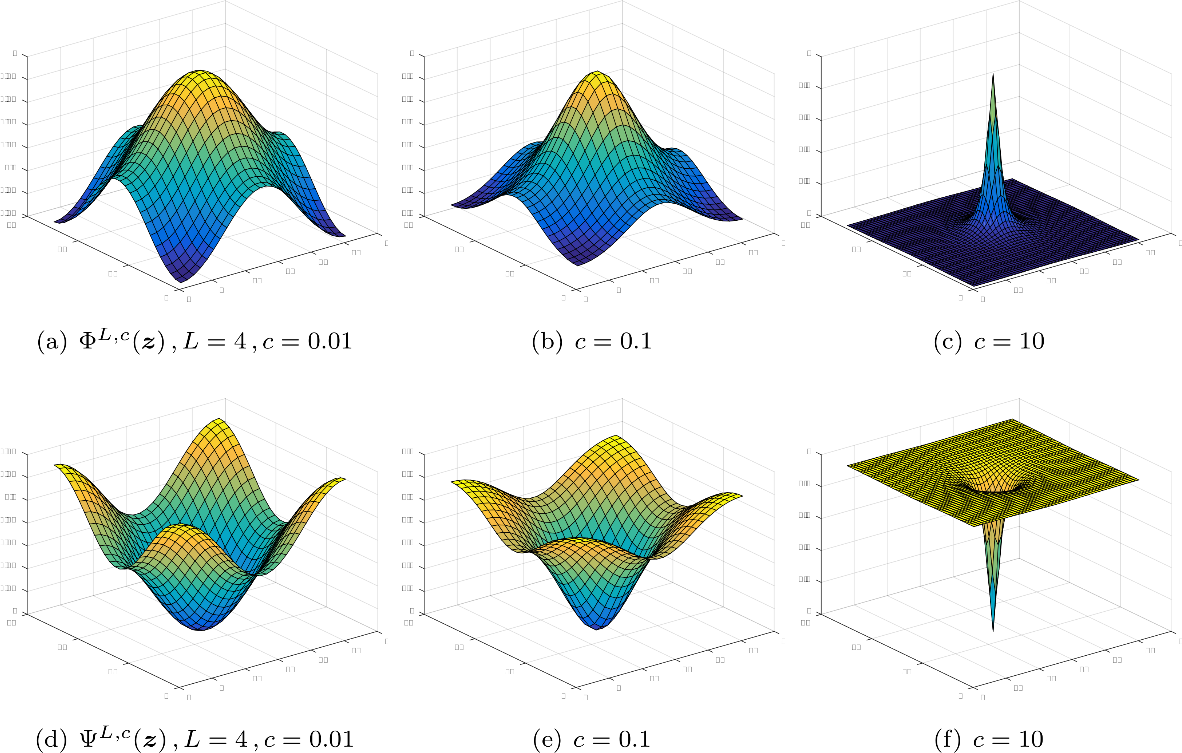}
 
 \caption{\label{fig:MDCD_let:FilterBanks:3D}
          This figure visualizes 3-dimensional version of scaling and wavelet functions produced by the DMCD model 
          without blurring effect, i.e. $H(\Bz) = 1$.
          Notice that bandwidths of these frames are controlled by parameter $c$ 
          with $\displaystyle \Psi^{L,c}(\Bz) = \sum_{l=0}^{L-1} \tilde \Psi_l^{L,c}(\Bz^{-1}) \Psi_l^{L}(\Bz)$.
          In particular, a larger value of $c$ results in smaller bandwidth of these functions.
         }
\end{figure*}  



\section*{Appendix C. Algorithms}  \label{sec:appendixC:Algorithms}

\begin{algorithm*} 
\label{alg:DMCDD:1}
\caption{Directional Mean Curvature for Textured Image Demixing (DMCD)}
\begin{algorithmic}
\small
 \STATE{
  {\bfseries Parameters:}
  $\gamma = \Big[ L, S, \mu_1, \mu_2, \nu_\epsilon, \nu_\rho, \alpha, \big[ \beta_i \big]_{i=1}^7 \Big]
  $ 
 } 
 \STATE{
  {\bfseries Initialization:}
  $\Bu^{(0)} = \Bf \,,  
   \Bv^{(0)} = \Beps^{(0)} = \Bd^{(0)} = \vec{\Br}^{(0)} = \vec{\Bt}^{(0)} = \vec{\By}^{(0)} = \vec{\Bw}^{(0)} = \vec{\Bg}^{(0)}
   = \boldsymbol{\lambda}_{\boldsymbol 1}^{(0)} = \vec{\boldsymbol{\lambda}}_{\boldsymbol 2}^{(0)}
   = \boldsymbol{\lambda}_{\boldsymbol 3}^{(0)} = \vec{\boldsymbol{\lambda}}_{\boldsymbol 4}^{(0)} 
   = \boldsymbol{\lambda}_{\boldsymbol 5}^{(0)} = \vec{\boldsymbol{\lambda}}_{\boldsymbol 6}^{(0)} 
   = \boldsymbol{\lambda}_{\boldsymbol 7}^{(0)} = \boldsymbol{\lambda}_{\boldsymbol 8}^{(0)} = \mathbf 0$ 
 } 
 \STATE{ } 
 \FOR{$\tau = 1 \,, \ldots \,,  $}
 \STATE
 {
  \% Compute
  $\Gamma = \Big[ \Bu, \Bv, \Beps, \Bd, \vec{\Br}, \vec{\Bt}, \vec{\By}, \vec{\Bw}, \vec{\Bg} \Big] 
   \in X^{4+3(L+1)+2S}$
  \begin{align*}
   \Gamma^{(\tau)} = \text{DMCD} \Big(\Gamma^{(\tau-1)} ~;~ \Bf \,, \gamma \Big) \,;
  \end{align*}

 }
\ENDFOR
\end{algorithmic}
\end{algorithm*}


\begin{algorithm*}   
\label{alg:DMCDD:2}
\caption{DMCD (I)} 
\begin{algorithmic}   
\small
 \STATE
 { 
   \STATE{ 
   \begin{align*}
    \Gamma^\text{(new)} = \text{DMCD} \Big(\Gamma^\text{(old)} ~;~ \Bf \,, \gamma \Big) \,,
    \quad \text{with }
    \Gamma = \Big[ \Bu, \Bv, \Beps, \Bd, \vec{\Br}, \vec{\Bt}, \vec{\By}, \vec{\Bw}, \vec{\Bg} \Big] 
   \end{align*}
   }
   \STATE{\bfseries I. Compute sub-problems:}
   \STATE{}
   \begin{align*}
    &\text{\bfseries 1.}~~ 
    \Bd^{(\tau)} = \Shrink \Big( \underbrace{ - \sum_{l=0}^L \big[ \cos\left(\frac{\pi l}{L}\right) \Bt_l^{(\tau-1)} \BDt + \sin\left(\frac{\pi l}{L}\right) \BDoT \Bt_l^{(\tau-1)} \big] }_{= \text{div}^-_L \vec{\Bt}^{(\tau-1)} }
    - \frac{\boldsymbol{\lambda}_{\boldsymbol 3}^{(\tau-1)}}{\beta_3} \,, \frac{1}{\beta_3} \Big)
    \\   
    &\text{\bfseries 2.}~~ 
    \Bt_l^{(\tau)} =
    \begin{cases}
     \displaystyle
     \RE \left[ \cF^{-1} \left\{ \frac{\cM_l^{(\tau)}(\Bz)}{\cN_l(\Bz)} \right\} \right]
     \,,~ & l = 0, \ldots, L-1 
     \\
     \displaystyle
     \By_L - \frac{ \boldsymbol{\lambda}_{\boldsymbol 4L} }{ \beta_4 } \,,~ & l = L
    \end{cases}    
    \\& \qquad
    \cN_l(\Bz) = \beta_4 + \beta_3 \abs{ \cos\left(\frac{\pi l}{L}\right) (z_2 - 1) + \sin\left(\frac{\pi l}{L}\right) (z_1 - 1) }^2 \,,
    \\& \qquad
    \cM_l^{(\tau)}(\Bz) = \beta_4 \left[ Y_l^{(\tau-1)}(\Bz) - \frac{ \Lambda_{4l}^{(\tau-1)}(\Bz) }{ \beta_4 } \right]
    -\beta_3 \left[ \cos\left(\frac{\pi l}{L}\right) (z_2 - 1) + \sin\left(\frac{\pi l}{L}\right) (z_1 - 1) \right] \times
    \\& \qquad \qquad \quad
    \left[ D^{(\tau)}(\Bz) + \sum_{l'=[0, L-1] \backslash \{l\}} \left[ \cos\left(\frac{\pi l'}{L}\right) (z_2^{-1} - 1) + \sin\left(\frac{\pi l'}{L}\right) (z_1^{-1} - 1) \right] T_{l'}^{(\tau-1)}(\Bz) + \frac{ \Lambda_3^{(\tau-1)}(\Bz) }{ \beta_3 } \right] \,.    
    \\   
    &\text{\bfseries 3.}~~ 
    \Br_l^{(\tau)} =
    \begin{cases}
     \displaystyle \Shrink \bigg( \big[ \underbrace{ \cos(\frac{\pi l}{L}) \Bu^{(\tau-1)} \BDtT + \sin(\frac{\pi l}{L}) \BDo \Bu^{(\tau-1)} }_{= \partial_l^+ \Bu } \big] 
                 - \frac{ \boldsymbol{\lambda}_{\boldsymbol 2 l}^{(\tau-1)} }{ \beta_2 } + \frac{ \boldsymbol{\lambda}_{\boldsymbol 1}^{(\tau-1)} + \beta_1 }{ \beta_2 } \cdot^\times \By_l^{(\tau-1)}
     \,, \frac{ \boldsymbol{\lambda}_{\boldsymbol 1}^{(\tau-1)} + \beta_1}{\beta_2} \bigg)
     , l = 0, \ldots, L-1
     \\
     \displaystyle \Shrink \bigg( \mathbf{1} - \frac{ \boldsymbol{\lambda}_{\boldsymbol 2 L}^{(\tau-1)} }{ \beta_2 } + \frac{ \boldsymbol{\lambda}_{\boldsymbol 1}^{(\tau-1)} + \beta_1 }{ \beta_2 } \cdot^\times \By_L^{(\tau-1)}
     \,,~ \frac{ \boldsymbol{\lambda}_{\boldsymbol 1}^{(\tau-1)} + \beta_1}{\beta_2}
     \bigg)
     , l = L 
    \end{cases}    
    \\
    &\text{\bfseries 4.}~~     
    \By_l^{(\tau)} = \begin{cases}
                \By_l^{\prime (\tau)} \,, & \abs{\vec{\By}^{\prime (\tau)}} \leq 1  \\
                \frac{ \By_l^{\prime(\tau)} }{ \abs{\vec{\By}^{\prime (\tau)}} } \,, & \abs{\vec{\By} ^{\prime (\tau)}} > 1
               \end{cases}
               \,,~~ l = 0, \ldots, L
    \,, \quad \text{and} \quad
    \begin{cases} 
     \By_l^{\prime (\tau)} &= \Bt_l^{(\tau)} + \frac{ \boldsymbol{\lambda}_{\boldsymbol 4 l}^{(\tau-1)} }{\beta_4} + \frac{\beta_1 + \boldsymbol{\lambda}_{\boldsymbol 1}^{(\tau-1)}}{\beta_4} \cdot^\times \Br_l^{(\tau)} \,,
     \\
     \abs{\vec{\By'}} &= \sqrt{ \sum_{l=0}^L \Big[ \Bt_l^{(\tau)} + \frac{ \boldsymbol{\lambda}_{\boldsymbol 4 l}^{(\tau-1)} }{\beta_4} + \frac{\beta_1 + \boldsymbol{\lambda}_{\boldsymbol 1}^{(\tau-1)}}{\beta_4} \cdot^\times \Br_l^{(\tau)} \Big]^{\cdot 2} } \,.
    \end{cases}
    \\   
    &\text{\bfseries 5.}~~
    \Bw_s^{(\tau)} = \Shrink \Big( \Bg_s^{(\tau-1)} - \frac{\boldsymbol{\lambda}_{\boldsymbol 6 s}^{(\tau-1)}}{\beta_6} \,,~ \frac{\mu_1}{\beta_6} \Big) \,,~
    s = 0, \ldots, S-1
   \end{align*} 
 }
\end{algorithmic}
\end{algorithm*}


\begin{algorithm*}
\label{alg:DMCDD:3}
\caption{DMCD (II)}
\begin{algorithmic}
\small
 \STATE
 {
   \begin{align*}
    &\text{\bfseries 6.}~~
    \Bg_s^{(\tau)} = \RE \left[ \cF^{-1} \left\{ \frac{ \cB_s^{(\tau)}(\Bz) }{ \cA_s(\Bz) } \right\} \right] \,,~ s = 0, \ldots, S-1
    \\& \qquad
    \cA_s(\Bz) = \beta_6 + \beta_7 \abs{ \cos\left(\frac{\pi s}{S}\right)(z_2 - 1) + \sin\left(\frac{\pi s}{S}\right)(z_1 - 1) }^2 
    \\& \qquad
    \cB_s^{(\tau)}(\Bz) = \beta_6 \left[ W_s^{(\tau)}(\Bz) + \frac{\Lambda_{6 s}^{(\tau-1)}(\Bz)}{\beta_6} \right]
    -\beta_7 \left[ \cos\left(\frac{\pi s}{S}\right)(z_2 - 1) + \sin\left(\frac{\pi s}{S}\right) (z_1 - 1) \right] \times
    \\& \qquad
    \left[ V^{(\tau-1)}(\Bz) + \sum_{s'=[0, S-1] \backslash \{s\}} \left[ \cos\left(\frac{\pi s'}{S}\right)(z_2^{-1} - 1) + \sin\left(\frac{\pi s'}{S}\right) (z_1^{-1} - 1) \right] G_{s'}^{(\tau-1)}(\Bz) + \frac{\Lambda_7^{(\tau-1)}(\Bz)}{\beta_7} \right]        
    \\   
    &\text{\bfseries 7.}~~ 
    \Bu^{(\tau)} = \RE \left[ \cF^{-1} \left\{ \frac{\cY^{(\tau)}(\Bz)}{\cX(\Bz)} \right\} \right]
    \\& \qquad
    \cY^{(\tau)}(\Bz) = \beta_2 \sum_{l=0}^{L-1} \left[ \cos\left(\frac{\pi l}{L}\right)(z_2^{-1} - 1) + \sin\left(\frac{\pi l}{L}\right)(z_1^{-1} - 1) \right]
    \left[ R_l^{(\tau)}(\Bz) + \frac{ \Lambda_{2 l}^{(\tau-1)}(\Bz) }{\beta_2} \right]
    \\& \qquad \qquad \qquad
    + \beta_5 H(\Bz^{-1}) \left[ F(\Bz) - H(\Bz) V^{(\tau-1)}(\Bz) - H(\Bz) \cP^{(\tau-1)}(\Bz) - \cE^{(\tau-1)}(\Bz) + \frac{\Lambda_5^{(\tau-1)}(\Bz)}{\beta_5} \right] 
    \\& \qquad
    \cX(\Bz) = \beta_2 \sum_{l=0}^{L-1} \abs{ \cos\left(\frac{\pi l}{L}\right)(z_2 - 1) + \sin\left(\frac{\pi l}{L}\right)(z_1 - 1) }^2 + \beta_5 \abs{H(\Bz)}^2    
    \\
    &\text{\bfseries 8.}~~
    \Bv^{(\tau)} = \Shrink \left( \Bt_\Bv^{(\tau)} \,,~ \frac{\mu_2 \alpha^{(\tau)}}{\beta_5 + \alpha^{(\tau)} \beta_7} \right) 
    \\& \qquad
    \Bt_\Bv^{(\tau)} = \frac{ \beta_5 }{ \beta_5 + \alpha^{(\tau)} \beta_7 }
    \left( \left[ \delta - \alpha^{(\tau)} \check{\Bh} \ast \Bh \right] \ast \Bv^{(\tau-1)}    
    + \alpha^{(\tau)} \check{\Bh} \ast \left[ \Bf - \Bh \ast \Bu^{(\tau)} - \Bh \ast \Brho^{(\tau-1)} - \Beps^{(\tau-1)} + \frac{\boldsymbol{\lambda}_{\boldsymbol 5}^{(\tau-1)}}{\beta_5} \right] \right) 
    \\& \qquad \qquad
    + \frac{ \beta_7 \alpha^{(\tau)} }{ \beta_5 + \alpha^{(\tau)} \beta_7 }
    \bigg( \underbrace{ -\sum_{s=0}^{S-1} \big[ \cos(\frac{\pi s}{S}) \Bg_s^{(\tau)} \BDt + \sin(\frac{\pi s}{S}) \BDoT \Bg_s^{(\tau)} \big] }_{= \text{div}^-_S \vec{\Bg} }
    - \frac{\boldsymbol{\lambda}_{\boldsymbol 7}^{(\tau-1)}}{\beta_7} \bigg)
    \\    
    &\text{\bfseries 9.}~~ 
    \Brho^{(\tau)} = \tilde{\Brho}^{(\tau)} - \CST \left[ \tilde{\Brho}^{(\tau)} \,,~ \nu_\rho^{(\tau)} \right] \,,
    \\& \qquad     
    \tilde{\Brho}^{(\tau)} = \left( \delta - \alpha^{(\tau)} \check \Bh \ast \Bh \right) \ast \Brho^{(\tau-1)} 
    + \alpha^{(\tau)} \check \Bh \ast \left( \Bf - \Bh \ast \Bu^{(\tau)} - \Bh \ast \Bv^{(\tau)} - \Beps^{(\tau-1)} + \frac{\boldsymbol{\lambda}_{\boldsymbol 5}^{(\tau-1)}}{\beta_5} \right) 
   \end{align*} 
 }
\end{algorithmic}
\end{algorithm*}

\begin{algorithm*}
\label{alg:DMCDD:4}
\caption{DMCD (III)}
\begin{algorithmic}
\small

\STATE{
\begin{align*}
  &\text{\bfseries 10.}~~ 
  \Beps^{(\tau)} = \tilde{\Beps}^{(\tau)} - \CST \left[ \tilde{\Beps}^{(\tau)} \,,~ \nu_\epsilon^{(\tau)} \right] 
  \,,~~~
  \tilde{\Beps}^{(\tau)} = \Bf - \Bh \ast \Bu^{(\tau)} - \Bh \ast \Bv^{(\tau)} - \Bh \ast \Brho^{(\tau)} + \frac{\boldsymbol{\lambda}_{\boldsymbol 5}^{(\tau-1)}}{\beta_5} 
\end{align*}
}

 \STATE{\bfseries II. Update Lagrange multipliers:}
 \STATE{
  \begin{align*}
  \boldsymbol{\lambda}_{\boldsymbol 1}^{(\tau)} &= \boldsymbol{\lambda}_{\boldsymbol 1}^{(\tau-1)} 
  + \beta_1 \left[ \abs{\vec{\Br}^{(\tau)}} - \langle \vec{\By}^{(\tau)} \,, \vec{\Br}^{(\tau)} \rangle_X \right]  
  \,,~~ \abs{\vec{\Br}} = \sqrt{ \sum_{l=0}^L \Br_l^{\cdot 2} } \,,~~ 
  \langle \vec{\By} \,, \vec{\Br} \rangle_X = \sum_{l=0}^L \By_l \cdot^\times \Br_l 
  \\
  \boldsymbol{\lambda}_{\boldsymbol 2 l}^{(\tau)} &= 
  \begin{cases}  
   \boldsymbol{\lambda}_{\boldsymbol 2 l}^{(\tau-1)} + \beta_2 \left[ \Br_l^{(\tau)} - \cos\left(\frac{\pi l}{L}\right) \Bu^{(\tau)} \BDtT - \sin\left(\frac{\pi l}{L}\right) \BDo \Bu^{(\tau)} \right] \,, & l = 0, \ldots, L-1 \\
   \boldsymbol{\lambda}_{\boldsymbol 2 l}^{(\tau-1)} + \beta_2 \left[ \Br_l^{(\tau)} - 1 \right] \,, & l = L
  \end{cases}
  \\
  \boldsymbol{\lambda}_{\boldsymbol 3}^{(\tau)} &= \boldsymbol{\lambda}_{\boldsymbol 3}^{(\tau-1)}
  + \beta_3 \left[ \Bd^{(\tau)} + \sum_{l=0}^{L-1} \left[ \cos\left(\frac{\pi l}{L}\right) \Bt_l^{(\tau)} \BDt + \sin\left(\frac{\pi l}{L}\right) \BDoT \Bt_l^{(\tau)} \right] \right]
  \\
  \boldsymbol{\lambda}_{\boldsymbol 4 l}^{(\tau)} &= \boldsymbol{\lambda}_{\boldsymbol 4 l}^{(\tau-1)}
  + \beta_4 \left[ \Bt_l^{(\tau)} - \By_l^{(\tau)} \right] \,,~~ l = 0, \ldots, L
  \\
  \boldsymbol{\lambda}_{\boldsymbol 5}^{(\tau)} &= \boldsymbol{\lambda}_{\boldsymbol 5}^{(\tau-1)}
  + \beta_5 \left[ \Bf - \Bh \ast \Bu^{(\tau)} - \Bh \ast \Bv^{(\tau)} - \Bh \ast \Brho^{(\tau)} - \Beps^{(\tau)} \right]  
  \\
  \boldsymbol{\lambda}_{\boldsymbol 6 s}^{(\tau)} &= \boldsymbol{\lambda}_{\boldsymbol 6 s}^{(\tau-1)}
  + \beta_6 \left[ \Bw_s^{(\tau)} - \Bg_s^{(\tau)} \right]     
  \,,~ s = 0, \ldots, S-1 
  \\
  \boldsymbol{\lambda}_{\boldsymbol 7}^{(\tau)} &= \boldsymbol{\lambda}_{\boldsymbol 7}^{(\tau-1)}
  + \beta_7 \bigg[ \Bv^{(\tau)} + \underbrace{ \sum_{s=0}^{S-1} \Big[ \cos\left(\frac{\pi s}{S}\right) \Bg_s^{(\tau)} \BDt + \sin\left(\frac{\pi s}{S}\right) \BDoT \Bg_s^{(\tau)} \Big] }_{= -\text{div}^-_S \vec{\Bg}^{(\tau)}} \bigg] 
 \end{align*}
 }
\end{algorithmic}
\end{algorithm*}


\begin{algorithm*} 
\label{alg:DMCDD:5}
\caption{Discrete multiscale projection from the DMCD model with $H(\Bz) = 1$ at $L$ directions, $I$ scale at constant $a > 0$}
\begin{algorithmic} 
\small

\STATE{
{\bfseries Two versions of a multiscale projection from the $\Bu$- and $\Bv$-problem:} 
\begin{align*} 
 f[\Bk] &= (f \ast \phi)[\Bk]
 + \sum_{i=0}^{I-1} \sum_{l=0}^{L-1} ( f \ast \check{ \tilde{\psi} }_{il} \ast \psi_{il} )[\Bk]
 ~\stackrel{\cF}{\longleftrightarrow}~           
 F(\Bz) = F(\Bz) \Phi(\Bz)
 + \sum_{i=0}^{I-1} \sum_{l=0}^{L-1} F(\Bz) \tilde \Psi_{il}(\Bz^{-1}) \Psi_{il}(\Bz) 
 \\
 f[\Bk] &= (f \ast \xi)[\Bk]
 + \sum_{i=0}^{I-1} \sum_{l=0}^{L-1} ( f \ast \check{\tilde \theta}_{il} \ast \theta_{il} )[\Bk]
 ~\stackrel{\cF}{\longleftrightarrow}~
 F(\Bz) = 
 F(\Bz) \Xi(\Bz) +  
 \sum_{i=0}^{I-1} \sum_{l=0}^{L-1} F(\Bz) \tilde\Theta_{il}(\Bz^{-1}) \Theta_{il}(\Bz) 
\end{align*}

{\bfseries 1. Frames:}
\begin{align*}
 \Phi(\Bz) &= I^{-1} \sum_{i=0}^{I-1} \Phi_\text{int}(\Bz^{a^i}) \,,
 \Psi_{il}(\Bz) = I^{-\frac{1}{2}} \Psi_l(\Bz^{a^i}) \,,
 \tilde \Psi_{il}(\Bz^{-1}) = I^{-\frac{1}{2}} \tilde \Psi_l(\Bz^{a^i}) \,,
 \\
 \Xi(\Bz) &= \frac{1}{IL} \sum_{i=0}^{I-1} \sum_{l=0}^{L-1} \Xi_l(\Bz^{a^i}) \,,
 \Theta_{il}(\Bz) = \frac{1}{\sqrt{IL}} \Theta_s(\Bz^{a^i}) \,, 
 \tilde\Theta_{il}(\Bz) = \frac{1}{\sqrt{IL}} \tilde\Theta_l(\Bz^{a^i}) \,. 
\end{align*}

{\bfseries 2. The unity conditions:} 
\begin{align*}
 &\Phi(\Bz) + \sum_{i=0}^{I-1} \sum_{l=0}^{L-1} \tilde \Psi_{il}(\Bz^{-1}) \Psi_{il}(\Bz) = 1 \,,
 \Phi(e^{j 0}) = 1 \,, \Psi_{il}(e^{j 0}) = \tilde \Psi_{il}(e^{j 0}) = 0 
 \\
 &\Xi(\Bz) + \sum_{i=0}^{I-1} \sum_{l=0}^{L-1} \tilde\Theta_{il}(\Bz^{-1}) \Theta_{il}(\Bz) = 1 \,,~
 \Xi(e^{j\mathbf 0}) = 1 \,,
 \Theta_{il}(e^{j\mathbf 0}) = \tilde\Theta_{il}(e^{j\mathbf 0}) = 0 \,.
\end{align*}
{\bfseries 3. Mother frames:} $l = 0, \ldots, L-1$
\begin{align*}
 \phi_\text{int}[\Bk] = \Big[ c (-\Delta_{dL}) + 1 \Big]^{-1} \delta[\Bk]
 &~\stackrel{\cF}{\longleftrightarrow}~
 \Phi_\text{int}(\Bz) = \frac{1}{\displaystyle 1 + c \sum_{l'=0}^{L-1} \abs{ \sin\left(\frac{\pi l'}{L}\right)(z_1 - 1) + \cos\left(\frac{\pi l'}{L}\right)(z_2 - 1) }^2} \,,
 \\
 \tilde \psi_l[\Bk] = -c \partial_l^- \phi_\text{int}[\Bk]
 &~\stackrel{\cF}{\longleftrightarrow}~
 \tilde \Psi_l(\Bz) = \frac{\displaystyle c \Big[ \sin\left(\frac{\pi l}{L}\right) (z_1^{-1} - 1) + \cos\left(\frac{\pi l}{L}\right) (z_2^{-1} - 1) \Big] }
                           {\displaystyle 1 + c \sum_{l'=0}^{L-1} \abs{ \sin\left(\frac{\pi l'}{L}\right)(z_1 - 1) + \cos\left(\frac{\pi l'}{L}\right)(z_2 - 1) }^2 } \,,
 \\
 \psi_l[\Bk] = \partial^+_l \delta[\Bk] 
 &~\stackrel{\cF}{\longleftrightarrow}~
 \Psi_l(\Bz) = \sin\left(\frac{\pi l}{L}\right) (z_1 - 1) + \cos\left(\frac{\pi l}{L}\right) (z_2 - 1) 
 \\
 \xi_l[\Bk] = \left[ 1 - c \partial_l^- \partial_l^+ \right]^{-1} \delta[\Bk]
 &~\stackrel{\cF}{\longleftrightarrow}~
 \Xi_l(\Bz) = \frac{ 1 }{ 1 + c \abs{ \cos\left(\frac{\pi l}{L}\right) (z_2 - 1) + \sin\left(\frac{\pi l}{L}\right) (z_1 - 1) }^2 } 
 \\  
 \theta_l[\Bk] = -c \partial_l^+ \xi_l[\Bk]
 &~\stackrel{\cF}{\longleftrightarrow}~
 \Theta_l(\Bz) = 
 \frac{ -c \left[ \cos\left(\frac{\pi l}{L}\right) (z_2 - 1) + \sin\left(\frac{\pi l}{L}\right) (z_1 - 1) \right] }
      { 1 + c \abs{ \cos\left(\frac{\pi l}{L}\right) (z_2 - 1) + \sin\left(\frac{\pi l}{L}\right) (z_1 - 1) }^2 } 
 \\  
 \tilde \theta_l[\Bk] = \partial_l^- \delta[\Bk]
 &~\stackrel{\cF}{\longleftrightarrow}~
 \tilde \Theta_l(\Bz) = - \left[ \cos\left(\frac{\pi l}{L}\right) (z_2^{-1} - 1) + \sin\left(\frac{\pi l}{L}\right) (z_1^{-1} - 1) \right] 
\end{align*}
}

\end{algorithmic}
\end{algorithm*}



\begin{algorithm*} 
\label{alg:DMCDD:5}
\caption{Continuous multiscale projection from the DMCD model with $H(\Bz) = 1$ (I)}
\begin{algorithmic} 
\small

\STATE{
{\bfseries Two versions of a multiscale projection from the $\Bu$- and $\Bv$-problem:} 
\begin{align*} 
  f[\Bk] &= (f \ast \phi)[\Bk] + \sum_{i=0}^{I-1} \sum_{l=0}^{L-1} ( f \ast \check{ \tilde{\psi} }_{il} \ast \psi_{il} )[\Bk]  
  ~\stackrel{\cF}{\longleftrightarrow}~
  F(e^{j\Bome}) = F(e^{j\Bome}) \widehat \phi(\Bome) + \sum_{i=0}^{I-1} \sum_{l=0}^{L-1} F(e^{j\Bome}) \widehat{\tilde{\psi}^*_{il}}(\Bome) \widehat{\psi_{il}} (\Bome)
 \\
 f[\Bk] &= (f \ast \xi)[\Bk] + \sum_{i=0}^{I-1} \sum_{l=0}^{L-1} (f \ast \check{\tilde \theta}_{il} \ast \theta_{il})[\Bk] 
 ~\stackrel{\cF}{\longleftrightarrow}~
 F(e^{j\Bome}) = F(e^{j\Bome}) \widehat \xi(\Bome) + \sum_{i=0}^{I-1} \sum_{l=0}^{L-1} F(e^{j\Bome}) \widehat{\tilde \theta_{il}^*} (\Bome) \widehat{\theta_{il}} (\Bome) 
\end{align*}

{\bfseries 1. Frames:}
\begin{align*}
 \phi(\Bx) = I^{-1} \sum_{i=0}^{I-1} a^{-i} \phi_\text{int} (a^{-1} \Bx)
 &~\stackrel{\cF}{\longleftrightarrow}~ 
 \widehat{\phi}(\Bome) = I^{-1} \sum_{i=0}^{I-1} \widehat{\phi}_\text{int}(a^i \Bome) 
 \\
 \psi_{il}(\Bx) = I^{-\frac{1}{2}} a^{-i} \psi_l(a^{-i} \Bx)
 &~\stackrel{\cF}{\longleftrightarrow}~ 
 \widehat{\psi}_{il}(\Bome) = I^{-\frac{1}{2}} \widehat{\psi}_l(a^i\Bome)
 \\
 \tilde \psi_{il}(\Bx) = I^{-\frac{1}{2}} a^{-i} \tilde\psi_l(a^{-i} \Bx) 
 &~\stackrel{\cF}{\longleftrightarrow}~ 
 \widehat{\tilde \psi_{il}}(\Bome) = I^{-\frac{1}{2}} \widehat{\tilde \psi_l}(a^i\Bome) 
 \\
 \xi(\Bx) = \frac{1}{IL} \sum_{i=0}^{I-1} \sum_{l=0}^{L-1} a^{-i} \xi_l(a^{-i} \Bx)
 &~\stackrel{\cF}{\longleftrightarrow}~
 \widehat \xi(\Bome) = \frac{1}{IL} \sum_{i=0}^{I-1} \sum_{l=0}^{L-1} \widehat \xi_l(a^i \Bome)
 \\
 \theta_{il}(\Bx) = \frac{1}{\sqrt{IL}} a^{-i} \theta_l(a^{-i} \Bx)
 &~\stackrel{\cF}{\longleftrightarrow}~
 \widehat \theta_{il}(\Bome) = \frac{1}{\sqrt{IL}} \widehat \theta_l(a^i \Bome)
 \\
 \tilde \theta_{il}(\Bx) = \frac{1}{\sqrt{IL}} a^{-i} \tilde \theta_l(a^{-i} \Bx)
 &~\stackrel{\cF}{\longleftrightarrow}~
 \widehat{\tilde \theta}_{il}(\Bome) = \frac{1}{\sqrt{IL}} \widehat{\tilde \theta}_l(a^i \Bome) 
\end{align*}

{\bfseries 2. The unity conditions:} 
\begin{align*}
 &\widehat{\phi}(\Bome) + \sum_{i=0}^{I-1} \sum_{l=0}^{L-1} \widehat{\tilde \psi^*_{il}}(\Bome) \widehat{\psi}_{il}(\Bome) = 1 \,,
 \widehat \phi(0) = 1 \,, \widehat \psi_{il}(0) = \widehat{\tilde \psi_{il}}(0) = 0 
 \\
 &\widehat \xi(\Bome) + \sum_{i=0}^{I-1} \sum_{l=0}^{L-1} \widehat{\tilde \theta_{il}^*} (\Bome) \widehat{\theta_{il}} (\Bome) = 1 \,, 
 \widehat \xi(0) = 1 \,, \widehat{\tilde \theta}_{il}(0) = \widehat{\tilde \theta}_{il}(0) = 0 
\end{align*}

{\bfseries 3. Mother frames:} $l = 0, \ldots, L-1$
\begin{align*}
 \phi_\text{int}(\Bx) = \left[ c \left(-\Delta_{L}\right) + 1 \right]^{-1} \delta(\Bx)
 &~\stackrel{\cF}{\longleftrightarrow}~
 \widehat{\phi}_\text{int}(\Bome) = \frac{1}{\displaystyle 1 + c \sum_{l'=0}^{L-1} \left[ \cos\left(\frac{\pi l'}{L}\right) \omega_2 + \sin\left(\frac{\pi l'}{L}\right) \omega_1 \right]^2} \,,
 \\
 \check{\tilde \psi}_l(\Bx) = -c \partial_l \phi_\text{int}(\Bx)
 &~\stackrel{\cF}{\longleftrightarrow}~
 \widehat{\tilde \psi_l^*}(\Bome) = \frac{\displaystyle -c \Big[ \cos \left(\frac{\pi l}{L}\right) j \omega_2 + \sin\left(\frac{\pi l}{L}\right) j \omega_1 \Big] }
                           {\displaystyle 1 + c \sum_{l'=0}^{L-1} \Big[ \cos\left(\frac{\pi l'}{L}\right) \omega_2 + \sin\left(\frac{\pi l'}{L}\right) \omega_1 \Big]^2} \,,
 \\
 \psi_l(\Bx) = \partial_l \delta(\Bx)
 &~\stackrel{\cF}{\longleftrightarrow}~
 \widehat{\psi}_l(\Bome) = \cos \left(\frac{\pi l}{L} \right) j\omega_2 + \sin \left(\frac{\pi l}{L} \right) j\omega_1 \,.
\end{align*}
}

\end{algorithmic}
\end{algorithm*}


\begin{algorithm*} 
\label{alg:DMCDD:5}
\caption{Continuous multiscale projection from the DMCD model with $H(\Bz) = 1$ (II)}
\begin{algorithmic} 
\small

\STATE{
\begin{align*} 
 \xi_l(\Bx) = \left[ 1 - c \partial_l^2 \right]^{-1} \delta(\Bx) 
 &~\stackrel{\cF}{\longleftrightarrow}~  
 \widehat{\xi_l}(\Bome) = \frac{ 1 }{ 1 + c \left[ \cos \left( \frac{\pi l}{L} \right) \omega_2 + \sin \left( \frac{\pi l}{L} \right) \omega_1 \right]^2 }  \,,
 \\ 
 \theta_l(\Bx) = -c \partial_l \xi_l(\Bx) 
 &~\stackrel{\cF}{\longleftrightarrow}~  
 \widehat{\theta_l}(\Bome) = 
 \frac{ -c \Big[ \cos(\frac{\pi l}{L}) j \omega_2 + \sin(\frac{\pi l}{L}) j \omega_1 \Big] }
        { 1 + c \left[ \cos \left( \frac{\pi l}{L} \right) \omega_2 + \sin \left( \frac{\pi l}{L} \right) \omega_1 \right]^2 } \,,
 \\ 
 \check{\tilde \theta}_l(\Bx) = \partial_l \delta(\Bx)
 &~\stackrel{\cF}{\longleftrightarrow}~
 \widehat{\tilde \theta_l^*}(\Bome) = \cos \left( \frac{\pi l}{L} \right) j \omega_2 + \sin \left( \frac{\pi l}{L} \right) j \omega_1 \,.
\end{align*}
}

\end{algorithmic}
\end{algorithm*}


\end{document}